\documentclass{colt2015} % Include author names

\title{Convex Risk Minimization and Conditional Probability Estimation}
\usepackage{times}
\usepackage{caption}
\coltauthor{% %XXX without this comment line (or putting \Name here), latex inserts a space... lol
  \Name{Matus Telgarsky} \Email{mtelgars@cs.ucsd.edu}\\
  \addr University of Michigan, Ann Arbor
  \AND
  \Name{Miroslav Dud\'ik} \Email{mdudik@microsoft.com}\\
  \addr Microsoft Research
  \AND
  \Name{Robert Schapire} \Email{schapire@microsoft.com}\\
  \addr Microsoft Research and Princeton University
}

\usepackage{eucal}
\usepackage{nccmath}
\newcommand{\mint}{\medint\int}

\usepackage{thmtools}

\usepackage{hyperref}

\usepackage{cleveref}

\usepackage{mathrsfs}
\usepackage{wrapfig}
\usepackage{tikz}
\usepackage{mathtools}
\usepackage{enumitem}
\usepackage{nicefrac}

\usetikzlibrary{calc,intersections,through,positioning,matrix}
\usetikzlibrary{decorations.markings}

\def\ddefloop#1{\ifx\ddefloop#1\else\ddef{#1}\expandafter\ddefloop\fi}
\def\ddef#1{\expandafter\def\csname b#1\endcsname{\ensuremath{\mathbf{#1}}}}
\ddefloop ABCDEFGHIJKLMNOPQRSTUVWXYZ\ddefloop
\def\ddef#1{\expandafter\def\csname bb#1\endcsname{\ensuremath{\mathbb{#1}}}}
\ddefloop ABCDEFGHIJKLMNOPQRSTUVWXYZ\ddefloop
\def\ddef#1{\expandafter\def\csname c#1\endcsname{\ensuremath{\mathcal{#1}}}}
\ddefloop ABCDEFGHIJKLMNOPQRSTUVWXYZ\ddefloop
\def\ddef#1{\expandafter\def\csname v#1\endcsname{\ensuremath{\boldsymbol{#1}}}}
\ddefloop ABCDEFGHIJKLMNOPQRSTUVWXYZabcdefghijklmnopqrstuvwxyz\ddefloop
\def\ddef#1{\expandafter\def\csname
  v#1\endcsname{\ensuremath{\boldsymbol{\csname #1\endcsname}}}}
\ddefloop
    {alpha}{beta}{gamma}{delta}{epsilon}{varepsilon}{zeta}{eta}{theta}{vartheta}{iota}{kappa}{lambda}{mu}{nu}{xi}{pi}{varpi}{rho}{varrho}{sigma}{varsigma}{tau}{upsilon}{phi}{varphi}{chi}{psi}{omega}{Gamma}{Delta}{Theta}{Lambda}{Xi}{Pi}{Sigma}{varSigma}{Upsilon}{Phi}{Psi}{Omega}\ddefloop

\def\1{\mathds 1}
\def\R{\mathbb R}

\def\Z{\mathbb Z}
\def\cRz{\cR_{\textup{z}}}

\def\fR{\mathfrak R}

\def\SPAN{\textup{span}}

\def\sign{\textup{sign}}

\DeclareMathOperator{\dom}{dom}

\def\Ker{\textup{Ker}}

\def\Pr{\textup{Pr}}

\def\Bal{\textup{Bal}} %BALAKRISHNAN OR BALSUBRAMANI

\newcommand{\ip}[2]{\left\langle #1, #2 \right \rangle}

\newcommand\hmu{{\widehat \mu}}

\newcommand\hcR{\widehat \cR}
\newcommand\hcE{\hat \cE}

\newcommand\barq{\bar q}
\newcommand\barr{\bar r}
\newcommand\barf{\bar f}  %BARF
\newcommand\hatq{\hat q}
\newcommand\q{q}
\newcommand{\bars}{\bar s}
\newcommand{\barw}{\bar w}

\newcommand{\Parens}[1]{\left(#1\right)}
\newcommand{\bigParens}[1]{\bigl(#1\bigr)}
\newcommand{\BigParens}[1]{\Bigl(#1\Bigr)}
\newcommand{\BiggParens}[1]{\Biggl(#1\Biggr)}
\newcommand{\Bracks}[1]{\left[#1\right]}
\newcommand{\bigBracks}[1]{\bigl[#1\bigr]}
\newcommand{\BigBracks}[1]{\Bigl[#1\Bigr]}
\newcommand{\Braces}[1]{\left\{#1\right\}}
\newcommand{\bigBraces}[1]{\bigl\{#1\bigr\}}
\newcommand{\BigBraces}[1]{\Bigl\{#1\Bigr\}}

\newcommand{\Set}[1]{\left\{#1\right\}}
\newcommand{\norm}[1]{\lVert#1\rVert}
\newcommand{\bigNorm}[1]{\bigl\lVert#1\bigr\rVert}
\newcommand{\abs}[1]{\lvert#1\rvert}
\newcommand{\bigAbs}[1]{\bigl\lvert#1\bigr\rvert}
\newcommand{\BigAbs}[1]{\Bigl\lvert#1\Bigr\rvert}
\newcommand{\angles}[1]{\langle#1\rangle}

\newcommand{\one}{\boldsymbol{1}}
\newcommand{\eps}{\varepsilon}
\newcommand{\ind}{\bbI}

\newcommand{\Eq}[1]{Eq.~\eqref{eq:#1}}

\newcommand{\Lclass}{\bbL}
\newcommand{\Ldiff}{\Lclass^{\!\!2+}}
\newcommand{\Lbnd}{\Ldiff_\textup{\mdseries b}}

\newcommand{\Dcan}{\cD_\star}

\newenvironment{proofof}[1]{\begin{proof}\textbf{(of {#1})}}{\end{proof}}
\newlist{enumroman}{enumerate}{1}
\setlist[enumroman]{label=\textup{(\roman*)},noitemsep}
\newlist{enumromansep}{enumerate}{1}
\setlist[enumromansep]{label=\textup{(\roman*)}}
\newlist{enumromansquash}{enumerate}{1}
\setlist[enumromansquash]{label=\textup{(\roman*)},noitemsep,topsep=0pt}

\begin{document}

\maketitle

\begin{abstract}
  This paper proves, in very general settings, that
  convex risk minimization is a procedure to select a unique conditional probability model determined
  by the classification problem.
  Unlike most previous work, we give results that are general enough to include cases in which no minimum exists, as occurs typically, for instance, with standard boosting algorithms.
  Concretely, we first show that any sequence of predictors minimizing convex risk over the source
  distribution will converge to this unique model when the class of predictors is linear (but potentially of infinite dimension).
  %\MJT{and maybe insert something about what boosting converges to?}.
  Secondly, we show the same result holds for \emph{empirical} risk minimization whenever this class of predictors is finite dimensional,
  where the essential technical contribution is a norm-free generalization bound.
\end{abstract}


\section{Introduction}
\label{sec:intro}

The goal in (binary) classification is to learn to accurately predict
the label $y\in\{-1,+1\}$ associated with an input $x$.
Unfortunately, it is NP-hard even to approximate this problem in easy cases \citep{raghavendra_halfspace_hardness};
thus a computationally attractive surrogate is often utilized.
Foremost amongst these is \emph{convex risk minimization}
in which a sequence of predictors are produced which minimize in the limit
some convex upper bound on a predictor's classification error.
In this paper, we attempt to analyze the effectiveness of such methods
in as much generality as possible.
Specifically, we aim to address the following questions:

\begin{description}
  \item[(Q1)]
    Suppose a sequence of predictors minimizes the convex risk over
    the true distribution.
    Does this sequence converge to some concrete object?
    This question is murky because convex functions need not have a
    minimum; for instance, the function $e^x$ has no minimum, but
    rather is minimized in the limit $x\rightarrow -\infty$.
    For the high-dimensional problems considered in convex risk
    minimization, the minimum may also only occur ``at infinity'' but
    in a far less straightforward way.
    This is typically the case, for instance, for standard boosting
    algorithms like AdaBoost~\citep{schapire_freund_book_final}.
    In such cases, what can be said concretely about the convergence
    of a minimizing sequence?

  \item[(Q2)]
    Now suppose a given sequence of predictors minimizes the \emph{empirical} convex risk,
    meaning the convex risk over some finite random draw from the distribution.
    What can be said about convergence with respect to the true
    distribution?
    In other words, what can be said about generalization and learning?
    The resolution is unclear here as well, since the preceding question highlights the need for
    predictors to be arbitrarily large, thus dooming the boundedness
    on which most standard statistical procedures rely \citep[Section 4]{bbl_esaim}.
\end{description}
In this paper we resolve both these questions by showing that convex
risk minimization converges to a unique conditional probability model $\bar \eta$.

\paragraph{Main results.}

To state our main theorems, we first present our learning setting.
We consider \emph{linear} classes of functions.
    That is,
    given a base set $\cH$ of prediction functions $h:\cX\to [-1,+1]$,
    the corresponding linear class consists of weightings of these
    functions described by weight vectors $w$ with $\sum_{h\in\cH}
    \abs{w[h]}<\infty$ where $w[h]$ denotes the weight of the
    function $h$, and where it is understood that these weights are
    non-zero only on a countable subset of $\cH$.
    Formally, this class is
    \[
      \BigBraces{ x \mapsto \sum_{h\in \cH} w[h]\,h(x) : \sum_{h\in\cH} \bigAbs{w[h]} <\infty }.
    \]
This setting recovers, for instance, the classical regression setting by choosing $\cH$ to consist of
    covariates corresponding to the dimensions of $x$, as well as the classical boosting setting by leaving $\cH$ arbitrary.
    Let $L_1(\cH)$ denote all possible choices for $w$ as above;
    moreover, given $w\in L_1(\cH)$,
    let $Hw: \cX \to \R$ denote the corresponding element of the linear class, meaning, $(Hw)(x)=\sum_h w[h]\,h(x)$. Thus,
    $H$ is a linear operator, abstractly collecting the elements of
    $\cH$ as ``columns''.

The loss functions $\ell$ that we study come from a large class $\Lbnd$ of certain twice continuously differentiable losses,
    whose precise definition is deferred to \Cref{sec:proof_outlines}.
Both the well-studied logistic loss $\ell_{\log}(r) := \ln(1+\exp(r))$
and exponential loss
    $\ell_{\exp}(r) := \exp(r)$ belong to $\Lbnd$.  With respect to
loss $\ell$, we define the population and
    empirical convex risk to be
    \begin{align*}
      \cR(w)
      := \mint \ell\BigParens{-y(Hw)(x)}d\mu(x,y)
      \qquad\textup{and}\qquad
      \hat\cR_n(w)
      := \frac 1 n \sum_{i=1}^n \ell\BigParens{-y_i(Hw)(x_i)},
      %\mint \ell(-y(Hw)(x))d\hat\mu(x,y),
    \end{align*}
    where $\Parens{(x_i,y_i)}_{i=1}^n$ is an i.i.d.\@ draw of size
    $n$ from the true distribution $\mu$.  Lastly, we
    define the excess convex risk $\cE(w) := \cR(w) - \inf_{v\in L_1(\cH)} \cR(v)$, with $\hcE_n$ defined analogously.

There are well-established methods for converting the models
produced using convex risk minimization into
\emph{conditional probability models}.
Specifically,
given loss $\ell\in\Lbnd$, functions $\cH$, and weighting $w\in L_1(\cH)$,
we define
    \begin{equation}  \label{eq:phi-def}
      \phi(r)
      := \frac {\ell'(r)}{\ell'(r) + \ell'(-r)}
      %= \frac {1}{1 + \frac{\ell'(-r)}{ \ell'(r)}},
      \qquad\textup{and}\qquad
      \eta_w(x,y) := \phi\BigParens{y(Hw)(x)}.
    \end{equation}
This function $\eta_w(x,y)$, which is well-defined with range $[0,1]$
for all $\ell\in\Lbnd$,
can be regarded as an estimate of the
conditional probability
$\Pr[Y=y|x]$
% for justification, see, for instance,
\citep{friedman_hastie_tibshirani_statboost,zhang_convex_consistency,bartlett_jordan_mcauliffe}.
For example, logistic loss $\ell_{\log}$ yields the usual sigmoid
$\phi(r) = (1 + \exp(-r))^{-1}$.

Our convergence results do not apply to the weighting sequences $(w_i)_{i\geq 1}$
directly, since, as earlier mentioned, these will often have no limit.
Instead we prove convergence of their corresponding
{conditional probability models}.
Specifically,
our first main result, the resolution of \textbf{(Q1)},
states that minimizing $\cR$ implies convergence to a unique conditional probability model $\bar\eta$.

\begin{theorem}%[Simplification of \Cref{fact:convergence}]
  \label{fact:convergence:simplified}
  Let loss $\ell \in \Lbnd$, probability measure $\mu$, and hypotheses $\cH$ be given.
  Then there exists a unique conditional probability model $\bar\eta : \cX \times \{-1,+1\} \to [0,1]$
  and a function $f_1:\R\to\R_+$ with $f_1(\eps) \to 0$ as $\eps \downarrow 0$
  such that every $w\in L_1(\cH)$ satisfies
  \[
    \mint \left| \bar \eta(x,y) - \eta_w(x,y)\right|d\mu(x,y)
    \leq f_1\left(
      \cE(w)
    \right).
  \]
  In particular, every sequence $(w_i)_{i\geq 1}$ with $\lim_{i\to\infty} \cE(w_i) = 0$
  satisfies $\eta_{w_i} \to \bar \eta$ in $L_1(\mu)$.
\end{theorem}
Note that the existence of $\bar\eta$ is not immediate given the existence of sequences minimizing $\cR$
since the collection of mappings from $\cX$ to $[0,1]$ is not compact in the $L_1(\mu)$ metric in general.
Instead, the proof here constructs $\bar\eta$ directly via duality, and thereafter uses duality
to control these sequences.

\Cref{fact:convergence:simplified} carries two essential consequences.
First, our analysis provides
a convergence concept for algorithms utilizing
convex risk minimization that is more general than previous approaches
in the sense that it can handle, for instance, the
unregularized boosting methods of \citet[Algorithm 1]{zhang_yu_boosting}, or even any regularized scheme with regularization weakening to zero.
%immediately gain a convergence concept for their iterate sequence $(w_i)_{i\geq 1}$.
Secondly, the real-valued model $Hw$ can be used for classification
simply by taking its sign, which is exactly equivalent to
the sign of $\eta_w(\cdot,1) - 1/2$, that is, the more likely label
according to the corresponding conditional probability model $\eta_w$.
Therefore,
convergence properties of $(\eta_{w_i})_{i\geq 1}$ imply convergence properties of the classification errors made by $(Hw_i)_{i\geq 1}$, complementary to the results of
\citet{bartlett_jordan_mcauliffe} and \citet{zhang_convex_consistency};
see \Cref{fact:zo}.

Next comes the resolution of \textbf{(Q2)}: under the assumption $|\cH|<\infty$,
we show that it suffices to minimize the empirical risk $\hat\cR_n$.

\begin{theorem}%[Simplification of \Cref{fact:findim:gen}]
  \label{fact:findim:gen:simplified}
  Suppose the setting of \Cref{fact:convergence:simplified}, in particular the existence of $\bar \eta$,
  but additionally that $|\cH|<\infty$.
  There exists a nonincreasing function $f_2 : \R\to\R_+$
  such that, with probability at least $1-\delta$ over an i.i.d.\@ draw of size $n\geq \Omega(\ln(1/\delta))$ from $\mu$,
  every $w\in L_1(\cH)$ satisfies
  \[
    \mint \left| \bar \eta(x,y) - \eta_w(x,y)\right|d\mu(x,y)
    =
    \cO\BiggParens{
      f_2\left(\hcE_n(w)\right)
      \BiggParens{
        \sqrt{\hcE_n(w)}
        + \sqrt{\frac {\ln(n) + \ln(1/\delta)}{n}}
      }
    },
  \]
  where $\Omega(\cdot)$ and $\cO(\cdot)$ omit constants based on $\cH$, $\ell$, and $\mu$, but not on the
  sample, or on $w$.
  In particular, any sequence $(w_i)_{i\geq 1}$ with $\lim_{i\to\infty} \hcE_i(w_i) = 0$
  %where $\hat\cE_\ell$ is measured on $i$ examples %<--- SEE HOW TERRIBLE THIS LOOKS
  satisfies $\eta_{w_i} \to \bar \eta$ in $L_1(\mu)$ a.s.
  %XXX following comment is important, it's why I decided upon this notation with an n subscript...
  % \MJT{would be much nicer to say $\hat\cE_i(Hw_i)$...}
\end{theorem}

Note that perhaps the most natural approach to proving this theorem---namely, to apply properties
of Rademacher complexity of Lipschitz functions \citep{bbl_esaim}---introduces a dependence on the norm
of $\|w\|$.
Instead, the bound above only exhibits a dependence on $\hcE_n(w)$,
which can be made arbitrarily small by considering only nearly optimal choices.
Depending on $\hcE_n(w)$ rather than $\|w\|$ is essential as these minimizing sequences will generally exhibit unboundedly growing norms,
a fact often encountered in practice (see \Cref{sec:experiments}).
Note that while \Cref{fact:findim:gen:simplified} requires strictly convex losses,
it is proved via generalization bounds which can handle more than just $\Lbnd$, in particular
the hinge loss (see \Cref{fact:findim:gen:helper:Dc,fact:findim:gen:helper:D}).

\paragraph{Illustrative example.}
%\label{sec:intro:ex}

Suppose $\cX = [-1,1]^2$ and $\cH$ consists of the coordinate functions.
Consider $\ell=\ell_{\log}$,
i.e., logistic regression.  Suppose
that the measure $\mu$ puts all of the mass on points $x$ that fall
into two well-separated rectangular regions (depicted as red and blue in \Cref{fig:margins}), with points in the blue region having $\Pr[Y=1|x]=1$ and points in the red region having $\Pr[Y=-1|x]=1$.
From the figure, it is clear that there exist two distinct vectors, $w_1$ and $w_2$, both of which define the lines (perpendicular to them) separating positive and negative examples.

\begin{wrapfigure}{r}{0.35\textwidth}
  \vspace{-30pt}
  \begin{center}
    \includegraphics[width=0.3\textwidth]{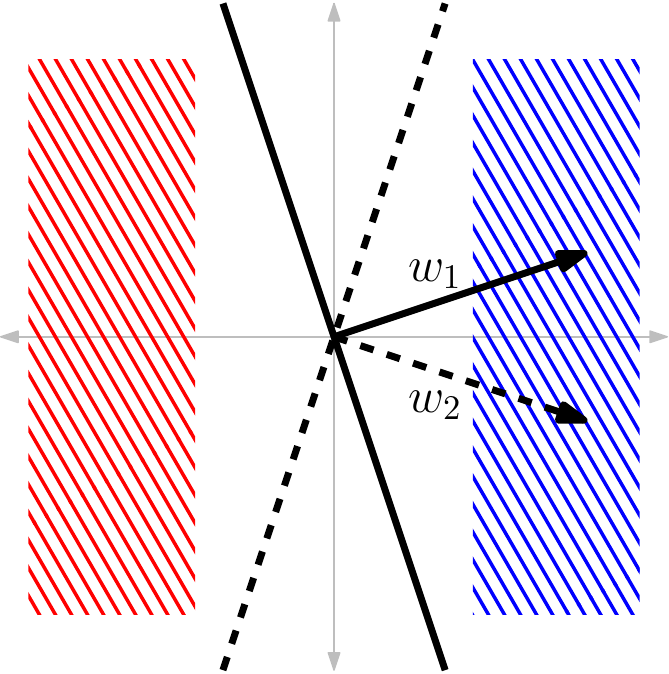}
  \end{center}
  \vspace{-20pt}
  \caption{A well-separated classification problem.}
  \label{fig:margins}
  \vspace{-10pt}
\end{wrapfigure}

The convex risk $\cR$ is minimized by both of the sequences $(iw_1)_{i\geq 1}$ and $(iw_2)_{i\geq 1}$;
moreover, the infimal risk is 0, which is not attained by
any $w\in L_1(\cH) = \R^2$, and
every minimizing sequence has norms growing unboundedly.

Conceivably, minimizing logistic loss could lead one algorithm to follow the sequence $(iw_1)_{i\geq 1}$ and another to follow $(iw_2)_{i\geq 1}$. Both of these sequences
%The two sequences $(\eta_{iw_1})_{i\geq 1}$ and $(\eta_{iw_2})_{i\geq 1}$ both
converge in the $L_1(\mu)$ metric;
their respective limit points, $\eta^{(1)}$ and $\eta^{(2)}$, are equal to 1, $1/2$, and 0 (for the positive class) on those points which have inner product,
respectively, positive, 0, and negative to $w_1$ or $w_2$.  Consequently, $\eta^{(1)} \neq \eta^{(2)}$. This shows that two different runs of logistic regression could give different probability estimates at some points. How then can \Cref{fact:convergence:simplified}
give a unique limit $\bar\eta$?
The resolution is that \Cref{fact:convergence:simplified} gives convergence in the $L_1(\mu)$ metric.  In particular, $w_1$ and $w_2$ only
disagree on the region between the two point clouds; this is a measure zero set, and thus $\eta^{(1)} = \eta^{(2)}$ $\mu$-a.e.

Note that in this setting, it is also straightforward to prove an analog of the uniform deviation bounds of \Cref{fact:findim:gen:simplified};
indeed, applying either VC theory \citep{bbl_esaim} or margin
bounds \citep{boosting_margin} will yield a bound that also
lacks dependence on $\|w\|$.  The distinction, however, is what both results say when applied to a sequence which does not achieve
zero classification error.  As will be shown in \Cref{fact:zo}, the classification error of these sequences may be erratic,
and therefore only loosely describes convergence behavior.  On the other hand, \Cref{fact:convergence:simplified}
and \Cref{fact:findim:gen:simplified} give a concrete object, $\bar\eta$, to which all minimizing sequences converge.

\paragraph{Classification errors and consistency.}
%\label{subsec:zo}

Let $\cRz(g) := \Pr[ Y \neq \sign(g(X))]$ denote the classification error of any mapping $g : \cX\to \R$,
where $\sign(r) := \one[r\geq 0]-\one[r<0]$.
Recall that the signs of $\eta_w(\cdot,1) - 1/2$ and $Hw$ agree, which
suggests that, because $\eta_w \to \bar\eta$
as provided by \Cref{fact:convergence:simplified,fact:findim:gen:simplified},
there might be a relationship between
$\bigParens{\cRz(Hw_i)}_{i\geq 1}$ and $\bigParens{\cRz(\bar \eta(\cdot,1) - 1/2)}_{i\geq 1}$.  However, convergence is stymied by the points where $\bar\eta = 1/2$,
that is, the points where $\sign(\bar\eta(\cdot,1) - 1/2)$
is discontinuous.
The following result provides that, excluding this set, the desired convergence indeed occurs;
in order to state it succinctly,
further let
$\eta_\mu(x,y) := \Pr[Y=y|x]$ denote the true conditional probability model,
and $\mu_\cX$ the marginal distribution along $\cX$.

\begin{proposition}
  \label[proposition]{fact:zo}
  Suppose the setting of \Cref{fact:convergence:simplified},
  and let $(w_i)_{i=1}^\infty$ be any sequence with $\eta_{w_i} \to \bar \eta$ in the $L_1(\mu)$ metric,
  and set $\Lambda := \{ (x,y) : \bar\eta(x,y) = 1/2 \}$.
  Then
  \[
    \limsup_{i\to\infty}
    \!\left|
      \cRz(Hw_i)
      -
      \cRz\!\Parens{\bar \eta(\cdot,1) - \frac 1 2}\!
    \right|\!
    \leq
    \limsup_{i\to\infty}
    \underbrace{
    \!\left|
      %lulz, getting this to fit on one line by packing a bunch of garbage notation just above...
      %\mint_{\bar \eta = 1/2} (2\Pr[Y=1|X=x] - 1)(1 - \eta_{w_i}(x,1))d\mu(x,y)
      \mint_{\Lambda} \BigParens{2\eta_\mu(x,1) - 1}\one\!\Bracks{\eta_{w_i}(x,1) < \frac 1 2}\!d\mu_\cX(x)
    \right|}_{\star}.
  \]
  Moreover, there exist choices of $(\mu,\cH,\ell)$ such that $\star > 0$ and the inequality is an equality.
\end{proposition}
The proposition implies that
%the classification error does converge to that of $\bar\eta$ if $\mu(\bar\eta=1/2)=0$. Otherwise,
the difference between the classification error of $\bar\eta$ and that of $\eta_{w_i}$ is bounded by $\mu(\bar\eta=1/2)$ in the limit. The fact that the bound in the proposition can be tight, i.e., there is a gap between the classification risks even as $\eta_{w_i}\to\bar\eta$, implies that the classification risk cannot be easily used to show convergence of $\eta_{w_i}$.
Similarly, as discussed with the example in \Cref{fig:margins}, any approach to the generalization analysis
that bounds classification error, such as VC theory, will be problematic since the classification
error can behave erratically, as provided by the possibility of $\star>0$ in \Cref{fact:zo}.

\newcommand{\MF}{\mathrm{MF}}

Finally, recall the classical consistency results
\citep{zhang_convex_consistency,bartlett_jordan_mcauliffe},
which may be summarized as follows.
Let $\MF$ denote the set of all measurable functions.
Then there exists a function $f_3 :\R \to \R_+$
with $f_3(\eps)\to 0$ as $\eps\downarrow 0$ so that every $w\in L_1(\cH)$ satisfies
\[
  \cRz(w) - \inf_{f\in\MF} \cRz(f) \leq f_3\BigParens{ \cR(w) - \inf_{f\in\MF} \cR(f) },
\]
where the last expression overloads $\cR(f) = \int \ell(-yf(x))d\mu(x,y)$.
As such, this result can be seen as a combination of \Cref{fact:convergence:simplified} and \Cref{fact:zo}
when $\SPAN(\cH)$ is a rich family of functions (e.g., dense in $\MF$).
Consequently, the results of the present work can be seen as complementary,
providing a specific convergence target $\bar \eta$ in the case of smaller $\SPAN(\cH)$
(e.g., when $\inf_{w\in L_1(\cH)}\cR(w) > \inf_{f\in\MF}\cR(f)$), rather than a single-sided bound as above.

% Of course, convergence to $\bar \eta$ could still very little in terms of the classification error rate
% $\cRz$;
% for instance, it is possible that $\cRz(1/2 - \bar \eta)$ is arbitrarily close to 1
% \citep{karthik_zeroone_margins}.

\paragraph{Outline.}

We close this introductory section with further notation.
In \Cref{sec:duality}, we construct $\bar\eta$ via convex duality,
and sketch the proofs of
\Cref{fact:convergence:simplified,fact:findim:gen:simplified}
in \Cref{sec:proof_outlines}.
%\Cref{sec:bibliographic} provides further bibliographic notes,
Many appendices collect further technical discussions and proof details.

\paragraph{Basic notation.}

Symbols defined in the preceding subsections---risk $\cR$, excess risk $\cE$, link function $\phi$, conditional probability model $\eta_w$---will continue
to be used in future sections.  The weighting space $L_1(\cH)$ should be viewed as the $L_1$ space over the counting measure on elements of $\cH$;
since $h\in\cH$ always has $\sup_x |h(x)|\leq 1$, it follows that
$\sup_x |(Hw)(x)| \leq \|w\|_1$.
Furthermore, in addition to the operator $H$, also let $A$ denote the
operator for which $(Aw)(x,y) = -y(Hw)(x)$,
whereby
\[
  \cR(w)
  %= \cR(Hw)
  = \mint \ell\bigParens{(Aw)(x,y)}d\mu(x,y) = \mint \ell(Aw)d\mu,
\]
where the last form  drops integration variables for succinctness.

We assume that $\mu$ can be \emph{disintegrated} \citep{pollard_disintegration}
into a marginal measure $\mu_\cX$ over $\cX$,
and a conditional probability $\eta_\mu(x,y) := \Pr[Y=y|x]$.
Let $\cZ := \cX\times \{-1,+1\}$ be the set of all $(x,y)$ pairs.
Given any subset $C\subseteq \cZ$, we define the intersection measure $\mu_C(S) := \mu(C\cap S)$
and conditional measure $\mu_{|C}$, where $\mu(C)>0$ implies $\mu_C(S) =
\mu_{|C}(S)\mu(C)$.
We use a ``hat'' symbol to
denote empirical measures, such as $\hat\mu$, $\hat \mu_C$, $\hat \mu_{|C}$.
To avoid ambiguity, we sometimes write $\cR(\cdot;\nu)$ and $\cE(\cdot;\nu)$ to denote risk and excess risk
when integration is over a measure $\nu$.
%; for instance, $\cR(w;\mu_C)$ will appear often.

Every loss $\ell:\R\to\R_+$ considered in this paper is a
\emph{classification loss}, meaning it is convex, non-decreasing,
and satisfies $\ell(0) > 0$ and $\inf_{z\in\R} \ell(z) = 0$.
The class of all such losses is denoted $\Lclass$.
The subset of these that are {strictly convex} and {twice continuously
differentiable} (i.e., $\ell''>0$) is denoted $\Ldiff$.
The more restrictive class $\Lbnd\subseteq\Ldiff$ will be defined in
\Cref{sec:proof_outlines}.
For classification losses, which are not necessarily differentiable,
we write $\ell'(z)$ to denote a fixed choice from the subgradient $\partial\ell(z)$; thus, a classification loss is described by a pair $(\ell,\ell')$ satisfying $\ell'(z)\in\partial\ell(z)$.
%\miro{drop $\ell'$.}

\section{Duality: The journey to the optimal conditional probability model $\bar\eta$}
\label{sec:duality}

This section shows the existence of the optimal conditional probability model $\bar\eta$. The key
challenge is the infinite dimensional setting, that is,
the fact that the hypothesis space $\cH$ and the
sample space $\cZ$ are infinite. To develop some intuition, we begin by studying the finite dimensional
case.

\subsection{Warm-up: Finite dimensional case}
\label{sec:warmup}

Assume for now that the hypothesis set is finite, $|\cH|=d$, and
the measure $\mu$ is uniform over $n$ data points. Consider
the problem of optimizing exponential loss over this measure, i.e.,
\begin{equation}
\label{eq:warmup:primal}
   \inf_{w\in\R^d}
   \Bracks{
      \sum_{i=1}^n e^{-y_i(Hw)(x_i)}
   }
\enspace.
\end{equation}
The conditional model for the exponential loss is
\begin{equation}  \label{eq:eta-for-exp}
  \eta_w(x,y)=\frac{e^{y(Hw)(x)}}{e^{y(Hw)(x)} + e^{-y(Hw)(x)}}
\enspace.
\end{equation}
Recalling the example from \Cref{fig:margins}, note how easily the infimum to \Eq{warmup:primal}
may fail to be attained.  In particular, if there exists $\hat w\in\R^d$ defining a
hyperplane which strictly separates the
positive and negative examples,
then the sequence $(j\hat w)_{j\geq 1}$ achieves zero risk in the limit, whereas every element
$w\in\R^d$
achieves a positive risk.  On the other hand,
$\eta_{j\hat w}(x_i,y_i)\to 1$ as $j\to\infty$.
So in this case, $\bar\eta$, which needs to be defined only over the examples $(x_i,y_i)$,
is described by $\bar\eta(x_i,y_i)=1$.

Similar to other studies of risk minimization stymied by the problem of missing minimizers
\citep{collins_schapire_singer_adaboost_bregman},
we consider the convex dual to \Eq{warmup:primal}. The dual of loss minimization of a linear model is the problem of maximizing entropy subject to constraints, where different losses yield different kinds of entropy~\citep{collins_schapire_singer_adaboost_bregman,AltunSm06}. The dual of \Eq{warmup:primal} is
\begin{equation}
\label{eq:warmup:dual}
   \max_{\q\in\R^n_+}
   \Bracks{
      \sum_{i=1}^n \bigParens{-\q_i\ln\q_i + \q_i}
   }
\quad
   \text{s.t.}
\quad
   \sum_{i=1}^n \q_i\bigParens{y_i h(x_i)} = 0
   \text{ for all }
   h\in\cH
\enspace.
\end{equation}
The objective on the left is an unnormalized entropy of the dual variable vector $q$,
representing an unnormalized reweighting of examples. The unnormalized entropy
is being maximized over the set of reweightings, which satisfy ``decorrelation'' constraints on the right.
Specifically, the constraints require that
the reweighting $\q$ be uncorrelated with every hypothesis, making the reweighted prediction problem
as hard as possible. Note that $\q=0$ is always feasible, but the unnormalized entropy pushes the
solution away from zero whenever feasible (the slope of entropy at zero is $-\infty$ (\Cref{fact:loss_prop}.v)). \Cref{fact:duality}
shows that the dual maximum is always attained, unlike the primal minimum. However, if both the primal maximum
$\bar w$ and dual maximum $\barq$ are attained, then $\barq_i=\exp\bigParens{-y_i(H\bar w)(x_i)}$.
For a general differentiable loss $\ell$, the optimality conditions yield $\barq_i=\ell'\bigParens{-y_i(H\bar w)x_i}$. If there is any example $j$ such that $x_j=x_i$, but the label is flipped ($y_j=-y_i$),
then we can rewrite $\barq_j$
as $\barq_j=\exp\bigParens{y_i(H\bar w)(x_i)}$ for exponential loss, and $\barq_j=\ell'\bigParens{y_i(H\bar w)(x_i)}$ for a general loss. Let $-i$ denote such an index $j$ if it exists. Contrasting the definition of $\eta_{\bar w}$
in \Eq{eta-for-exp} with the optimality condition for $\barq$ suggests defining
\[
  \bar\eta(x_i,y_i)=
\begin{cases}
    \barq_{-i} / \Parens{\barq_{-i}+\barq_i}
        & \text{if }\barq_i>0,
\\
      1 & \text{if }\barq_i=0,
\end{cases}
\]
where in the absence of the example with the flipped label, define $\barq_{-i}=\ell'\bigParens{-(\ell')^{-1}(\barq_i)}$ to emulate such an example; for exponential loss, $\barq_{-i}=1/\barq_i$. The value of $\bar\eta$ for $\barq_i=0$ is obtained by taking the limit $\barq_i\to 0$ (i.e., $\barq_{-i}\to\infty$ for exponential loss).
The next section shows that this $\bar\eta$ is the correct limit object,
even for an infinite sample space and an infinite hypothesis set.

\subsection{Infinite dimensional case}

Before constructing $\bar\eta$ and proving \Cref{fact:convergence:simplified},
we establish an infinite dimensional duality result similar
to the finite dimensional result from \Cref{sec:warmup}. In the primal, we now minimize
an integral rather than a sum. In the dual, we optimize over
unnormalized densities over $\cZ$.
Recall that the linear map $A$ returns functions
on $\cZ$ such that $(Aw)(x,y)=-y(Hw)(x)$.
Formally, we seek the following duality result:
\begin{align}
\label{eq:duality:verbose}
   \inf_{w\in L_1(\cH)}
   \Bracks{
   \mint \ell(Aw)d\mu
   }
\;
=
\;
&
   \max_{\q\in\cQ}
   \Bracks{
      -\mint \ell^*\BigParens{\q(x,y)}d\mu(x,y)
   }
\\
\notag
&
   \text{s.t.}
   \mint \q(x,y)\BigParens{y h(x)} d\mu(x,y)= 0
   \text{ for all }
   h\in\cH
\end{align}
where $\ell^*(s) := \sup_{r} [rs - \ell(r)]$ is the conjugate of $\ell$ (see \Cref{sec:banach}).
For example, when $\ell$ denotes the exponential loss,
we find that $\ell^*(s)=s\ln s - s$ for $s\ge0$ and $\ell^*(s)=\infty$ for $s<0$,
giving rise to the non-negativity constraint on $\q$ and the dual objective we already saw in \Eq{warmup:dual}.

A crucial technical
question is the choice of $\cQ$, i.e., the set that $\q$ is selected from.
Following the intuition of \Cref{sec:warmup}, the goal is to construct $\bar\eta(x,y) = \barq(x,-y) / \bigParens{ \barq(x,-y) + \barq(x,y) }$.
The space $\cQ$ should be large enough to allow construction of any conditional probability distribution $\eta$ for $\mu_\cX$. To achieve this, it suffices to make sure that all measures which are absolutely continuous with respect to $\mu$ have their densities included in $\cQ$.
In fact, our set can be slightly smaller: it just needs to include all densities for which the dual objective,
i.e., the integral $\int\ell^*(\q)d\mu$, is finite.

% The
% space $\cQ$ should be large enough to allow construction of any conditional probability distribution $\eta$
% for $\mu_\cX$. To achieve this, it suffices to make sure that all probability measures over $\cZ$ which
% are absolutely continuous with respect to $\mu$ have their densities
% included in $\cQ$.
% %
% \miro{Matus---is my reasoning correct here? We are still a tiny bit confusing.
% Can we make things clearer without getting too technical?}
% %
% The space $\cQ$ cannot be too large due to the ``topological pairing'' requirement
% (see \Cref{app:convex}), which requires that the linear map $w\mapsto\int -y(Hw)(x)\q\,d\mu(x,y)$ be continuous on
% $L_1(\cH)$ for all $\q\in\cQ$.
% %
% \miro{Matus---is the reasoning around pairing correct here?}

%The standard choice is of course an $L_p(\mu)$ space, however this runs into some difficulties.
%For example if the primal objective in \Eq{duality:verbose}, namely the map $f\mapsto \int \ell(f)d\mu$,
%is considered as a map from $L_1(\mu)$ to the reals, then the dual $q\mapsto \int\ell^*(q)d\mu$
%is from $L_\infty$ to the reals.  While this makes sense for entropies prohibit unbounded densities $\q$
%--- an approach which was fruitful for the logistic loss \citep{mjt_logistic} --- it is too restrictive for
%the exponential loss, where $\q$ can be unbounded but still achieve finite entropy.

One candidate class of functional spaces is $L_p(\mu)$, where $p\ge 1$. These are Banach spaces
of measurable functions with the norm defined by $\|f\|_p=\Parens{\int |f|^p\,d\mu}^{1/p}$. The
space $L_p(\mu)$ contains all measurable functions with $\|f\|_p <\infty$. However,
in our setting, we instead want to place restrictions on the allowed functions $\q$ based on
the integral $\int \ell^*(\q)d\mu$
%, appearing in the dual objective,
rather than $\int|\q|^p d\mu$.
Therefore, instead of working with $L_p(\mu)$ spaces, we work with their generalization called
\emph{large Orlicz spaces} (\citealp{leonard_orlicz}, and \Cref{sec:orlicz}), which allows us
to tailor the set $\cQ$ to $\ell^*$.

%Instead, the development here will use \emph{Orlicz spaces},
%which develop a Banach space via a generalized notion of ball constructed directly from the objective
%functions in \Eq{duality:verbose}; in addition to circumventing the issues mentioned with the exponential
%loss above, this choice will alleviate many other technical nuisances.

In detail, the construction of a {large Orlicz space} begins with a non-negative
convex function $\theta:\R\to[0,\infty]$ symmetric around zero (i.e., $\theta(r)=\theta(|r|)$), not identical to zero (i.e., $\theta(r)\to\infty$ as $r\to\infty$, by convexity), and with $\theta(0)=0$.
This function $\theta$ serves the same role as the $p$-th power
function in the construction of $L_p(\mu)$. The conditions that
we place on $\theta$ make it possible to define
``the unit ball'' of functions, analogous to the unit ball in $L_p(\mu)$, namely
\[
   \cB\coloneqq\Set{f \text{ measurable}:\:\mint \theta\bigParens{f(z)}d\mu(z)\le 1}
\enspace.
\]
This set is then used to define the norm %$\norm{\cdot}_\theta$:
$
   \|f\|_\theta = \inf\{r\ge 0:\:f\in r\cB\}
$,
where the norm equals $\infty$ if $f$ is outside the scaled ball $r\cB$ for all $r\ge 0$.
The \emph{large Orlicz space} $L_\theta(\mu)$ is defined to contain all measurable functions with $\|f\|_\theta<\infty$. For $p\ge 1$, the choice $\theta(s)=|s|^p$ recovers the $L_p(\mu)$ spaces.
(See \Cref{sec:orlicz} for further background.)

Now we are ready to answer what the space $\cQ$ should be.
Following the construction of \citep{leonard_entropy}, we begin by introducing
a symmetrized version of the loss $\ell$ with the first-order Taylor expansion at zero removed:
%\MJT{so what's the verdict on $\ell'$ for non-differentiable $\ell$?  Talking about Orlicz spaces a pain without it, but...}
%
\begin{equation}
\label{eq:beta}
    \beta(s) \coloneqq \max\Set{
      \ell( s) - \BigParens{\ell(0) + s \ell'(0)} ,\;
      \ell(-s) - \BigParens{\ell(0) + (-s) \ell'(0)}
    }
\enspace.
\end{equation}
It turns out that the Orlicz space $L_{\beta^*}(\mu)$, derived from the conjugate $\beta^*$, satisfies our desideratum on~$\cQ$:
it contains all the densities with respect to $\mu$ whose dual objective is finite (see \Cref{fact:orlicz_prop}.iii).
The next theorem spells out the duality result of \Eq{duality:verbose} with a more succinct representation
of constraints via \emph{adjoint} $A^\top$ of the operator $A$. The adjoint is a generalization
of the matrix transpose. The adjoint $A^\top$ is a linear operator
which maps $\q$ into a linear function on $L_1(\cH)$ defined by $(A^\top\q)(w)=\int (Aw)(z)\,\q(z)\,d\mu$. The constraint
of \Eq{duality:verbose} is equivalent to requiring $(A^\top\q)(w)=0$ for all $w$, i.e., $A^\top\q$ is
required to be the zero of the vector space
of linear functions on $L_1(\cH)$. Thus, the constraint can be written as $A^\top\q=0$, highlighting the fact that it is
a linear constraint on $\q$.

Apart from the duality result, the theorem also enumerates several important properties of the dual optimum, which are relevant for the construction of $\bar\eta$
in \Cref{def:bareta} below.
Properties (i) and (ii) show that $\bar\eta$ is a well-defined conditional probability.
Property (iii) implies that $\bar\eta(x,y)=\eta_{\bar w}(x,y)=\phi(y(H\bar w)(x))$ when the primal optimum exists and the loss
is differentiable. Property (iv) looks more technical: it implies that when the primal optimum $\bar w$ does not exist,
$\bar h(x)\coloneqq(\ell')^{-1}(\barq(x,-1))$
can serve a similar role as $H\bar w$, because $\bar\eta(x,y)=\phi(y\bar h(x))$; indeed, we use this construction of $\bar h$
in \Cref{subsec:sketch:convergence}.
%\miro{Part (vi) is going away} %done
%\MJT{maybe put definition $A:L_1(\cH)\to M_\beta(\mu)$ and $A^\top : L_\beta(\mu)\to L_\infty(\cH)$ here?}
%
\begin{theorem}
  \label{fact:duality}
  Let finite measure $\mu$ over $\cZ$,
  hypotheses $\cH$,
  and loss function $\ell \in \Lclass$
  be given, with $\beta$ defined by \Eq{beta}.
  Then
\begin{equation}
\label{eq:duality}
   \inf_{w\in L_1(\cH)}
   \Bracks{
      \mint \ell(Aw)d\mu
   }
\quad
=
\quad
   \max_{
   %\substack{
       \q\in L_{\beta^*}(\mu):\:\:
   %\\  \textup{s.t. }
       A^\top\q=0
   %}
   }
   \Bracks{
      -\mint \ell^*(\q)d\mu
   }
\enspace.
\end{equation}
%%
%% XXX no time for breg
%% \\
%% &= - \min \left\{ \textup{Breg}%_{\mint \ell^*d\mu}
%% (\q, \tilde \q) : \q \in L_{\beta^*}(\mu), \|H^\top \q\|_\infty = 0 \right\},
%% \label{eq:duality}
%% where $\textup{Breg}$ is the Bregman divergence with respect to convex function $\mint\ell^*d\mu$,
%% and $\tilde \q(x,y) = \argmin_{z\in \R} \ell^*(z)$ everywhere (which is well-defined by \Cref{fact:loss_prop}).
%%
  A dual optimum $\barq$ always exists, and can be chosen to satisfy the following, $\mu$-a.e.\@ over $(x,y)$:
  \begin{enumroman}
    \item
      $\barq(x,y) \ge 0$.
    \item
      $\barq(x,y) + \barq(x,-y) > 0$.
    \item
      $\barq(x,y) \in\partial\ell(A\bar w)(x,y)$ where $\bar w$ is a primal optimum (if it exists).
  \end{enumroman}
  Furthermore,
  \begin{enumroman}[resume]
    \item
      If $\ell\in \Ldiff$, then
      $(\ell')^{-1}\bigParens{\barq(x,y)}=-(\ell')^{-1}\bigParens{\barq(x,-y)}$, $\mu$-a.e.\@ over all $(x,y)$ for which $(\ell')^{-1}$ is defined at both $\barq(x,y)$ and $\barq(x,-y)$.
   \item
      If $\ell$ is differentiable, then
      $\barq$ is unique (up to $\mu$-null sets).
   %\item
   %  If $\ell\in \Ldiff$ has tightest Lipschitz constant $L := \sup_{r\neq s}(\ell(r)-\ell(s))/(r-s)$,
   %  then $\mu(\{z:\barq(z) \geq L\}) = 0$. \miro{going away}
  \end{enumroman}
   %XXX gonna do this one?  it's interesting...
   %\item
   %  For $\mu$-a.e. $(x,y)$,
   %  $\Pr[Y=1|X=x] \in (0,1) \Longrightarrow \min\{\barq(x,1), \barq(x,-1)\} > 0$.
   %  Might be good to discuss the contrapositive as well.
\end{theorem}
Using part (v), we obtain that the following defines
a unique $\bar\eta$ (up to $\mu$-null sets):
\begin{definition}
\label[definition]{def:bareta}
  Let $\ell \in\Lclass$ be differentiable and $\barq$
  be the dual optimum satisfying
  conditions (i) and (ii) of \Cref{fact:duality}. We define
  the optimal conditional model $\bar\eta$ as
  \begin{equation}
    \label{eq:bar_eta}
    \bar\eta(x,y) = \frac {\barq(x,-y)}{\barq(x,-y) + \barq(x,y)}
\enspace.
  \end{equation}
\end{definition}
This is the $\bar\eta$ that appears in \Cref{fact:convergence:simplified}.
This theorem will be proved in the next section.

% The remainder of this section will discuss the dual problem, the dual optimum $\barq$, the conditional probability
% model $\bar\eta$, and some basic structure they impart on the questions here (convergence of primal iterates),
% as well as finally defining $\Lbnd$.

% \miro{\textbf{Commenting things out; it's mostly covered in the previous section.}

% \footnotesize
% In order to crystallize the essential shape of the dual,
% consider the logistic loss $\ell(r) = \ln(1+\exp(r))$,
% which has conjugate $\ell^*(s) = s\ln(s) + (1-s)\ln(1-s) + \infty\cdot\one[s\in [0,1]]$.
% Notice that $\ell^*$ curves very sharply away from 0; in particular, the duality maximization in \Eq{duality} will
% try hard to avoid the value zero.  Indeed, the question of whether $\barq(x,y)$ is zero or positive will be
% the basic structural property which controls all the analysis here.

% Speaking more generally, first consider the constraint in the dual problem, which as mentioned above may be written
% $0 = \sup\{\mint (Aw) \barq d\mu : w\in L_1(\cH)\} = \sup \{ |\mint yh(x) \barq(x,y)d\mu(x,y)|:h\in\cH\}$
% \citep[Appendix 1.A]{mjt_thesis}; in other words, $\barq$ reweights $\mu$ so that every $h\in\cH$ appears uncorrelated,
% making the prediction problem is as hard as possible.  Note that the choice $\q = 0$ is always feasible, however it achieves
% a bad objective value: by \Cref{fact:loss_prop}, $\mint \ell^*(0)d\mu = 0$.  By contrast, \Cref{fact:loss_prop} also
% provides that the choice the uniform choice $\q = \ell'(0)$ maximizes the dual objective, however it may easily fail to be feasible.
% }

\section{Convergence and generalization via easy and difficult sets}
%Proof outlines for \Cref{fact:convergence:simplified} and \Cref{fact:findim:gen:simplified}}
\label{sec:proof_outlines}

We saw in \Cref{sec:warmup} that the conjugate $\ell^*$ of the exponential loss has
an infinite slope at zero;
the same turns out to be true for all losses in $\Ldiff$ (\Cref{fact:loss_prop}.v).
Informally, this means that the dual optimization avoids setting $\barq=0$ unless forced to do so by the decorrelation constraint $A^\top\barq=0$.
We will see that this distinction between the set of points where $\barq=0$ and the set where $\barq>0$ is fundamentally important to the analysis,
a fact seen before in the analysis of boosting \citep{mukherjee_rudin_schapire_adaboost_convergence_rate,primal_dual_boosting,mjt_logistic}.
%corresponding parts of the sample space separately.
We call these two sets of points ``easy'' and ``difficult'' (respectively) for reasons which we illustrate on an example.

\paragraph{An example.}
%\label{sec:example:easy:difficult} %%%% TODO

Consider
the example in \Cref{fig:difficult:set}, which builds on the example from \Cref{fig:margins}.
%Recall that $\cX=[-1,1]^2$, and there are two hypotheses coinciding with coordinates.
In addition to the two well-separated regions of positive and
negative examples, we now add an
alternating sequence of positive and negative point masses along the line $\gamma$ orthogonal
to the weight vector $w_1$. Each weight vector $w\in\R^2$ represents a linear predictor returning the inner product $x\mapsto w\cdot x$. The \emph{margin} of a data point $(x,y)$ with respect to this predictor is $y(w\cdot x)$.
The decorrelation
constraint (see Eq.~(\ref{eq:duality:verbose})) requires that the weighted margin of every hypothesis (and of every linear combination) according to the density $\q$ is equal to zero. The predictor described by $w_1$ gives a positive margin to all points in the two separated regions (the easy set) and zero margin to those along the line $\gamma$ (the difficult set). Hence, any $\q$ satisfying the decorrelation constraint must equal zero over these two regions. On the other hand, because the point masses along $\gamma$ are antisymmetric around zero, each of them can receive the density $\q(x,y)=\bar s$ where $\bar s$ is a minimizer of $\ell^*$
(it always exists by \Cref{fact:loss_prop}.i).

\begin{wrapfigure}{r}{0.35\textwidth}
  \vspace{-30pt}
  \begin{center}
    \includegraphics[width=0.3\textwidth]{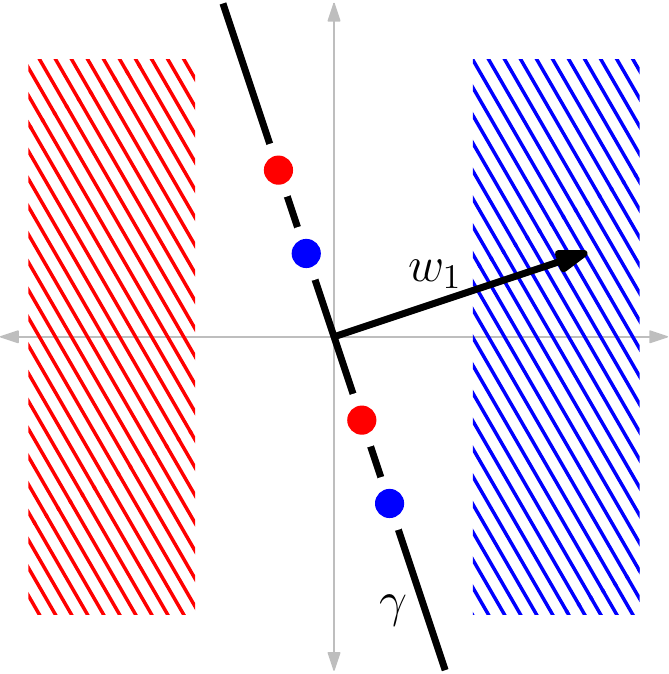}
  \end{center}
  \vspace{-20pt}
  \caption{Easy and difficult sets.}
  \label{fig:difficult:set}
  \vspace{-5pt}
\end{wrapfigure}

In the primal, the sequence $(iw_1)_{i\ge 1}$ still minimizes the risk as follows. First, the risk in the two regions goes to zero. Next, the risk of any weight vector $w$ over points along $\gamma$ is only a function of the projection of $w$ onto $\gamma$. Since the masses along $\gamma$ are antisymmetric and the loss function is convex (and increasing as the prediction is more wrong), the projection needs to be at the origin to minimize the risk along $\gamma$. This is exactly the case for $iw_1$ by orthogonality.

If the example were to be slightly perturbed, so that the point masses would still lie on $\gamma$ in an alternating pattern (but not antisymmetric), a minimizing sequence would take the form $(\hat w + iw_1)_{i\ge 1}$ where $\hat w\in\gamma$ would be the minimizer of the risk of the points along $\gamma$. Because of the alternating pattern such a minimizer would be bound to exist.

\paragraph{Preliminaries.}

Several aspects of the example carry over to the general setting. First, it can be shown that the risk on points where $\bar q=0$ converges to zero when the primal is minimized, that is, a perfect classification is achieved. Therefore, we call this set of points ``easy''. Second, the points where $\bar q>0$ cannot be further ``separated'' in the sense that any $w$ under which some non-null measure of these points receives a positive margin also yields a non-null measure of points with a negative margin. We call this set ``difficult''.

\begin{definition}
Given a finite measure $\mu$, hypotheses $\cH$, loss $\ell \in \Lclass$,
and a dual optimum $\barq$ satisfying the conditions of \Cref{fact:duality},
the \emph{difficult set} is defined as
$
  \cD\coloneqq\{z\in\cZ : \barq(z) > 0\}
$.
Its complement $\cD^c$ is called the \emph{easy set}.
\end{definition}

The next lemma (and the following corollary)
show that, similar to the example,
all of the risk is in fact due to the difficult set.
The lemma proves equality of the
dual objectives for $\mu$ and the restricted measure $\mu_\cD$, and furthermore that
$\barq$ is feasible and optimal for both problems. The corollary highlights the implications
in the primal, that by optimizing the risk on $\mu$, we optimize the risk
on the difficult set, and drive the risk on the easy set to zero. For technical reasons,
both results are stated for \emph{supersets} of difficult sets.

\begin{lemma}
  \label[lemma]{lemma:dual:D}
  Given a finite measure $\mu$,
  hypotheses $\cH$,
  loss $\ell \in \Lclass$,
  a difficult set $\cD$ and an associated dual optimum $\barq$,
  let $D$ be an arbitrary (measurable) superset of the difficult set: $\cD\subseteq D$.
  Then the dual optimal values for $\mu$ and $\mu_D$ are equal:
  \begin{align*}
%    \inf_w
%       \Bracks{ \mint \ell(Aw)d\mu }
%    &\;=\;
    \max_{\q\in L_{\beta^*}(\mu):\:A^\top \q = 0}
       \Bracks{ - \mint \ell^*(\q)d\mu }
    \;=\;
    \max_{\q\in L_{\beta^*}(\mu_D):\:A^\top \q = 0}
       \Bracks{ - \mint_D \ell^*(\q)d\mu }
%    \;=\;
%    \inf_w
%       \Bracks{ \mint_D \ell(Aw)d\mu }
  \enspace.
  \end{align*}
  The general dual optimum $\barq$ is feasible for both problems and attains both
  maxima.
Moreover, if  $\barq_D$ is a dual optimum for $\mu_D$, then
  $\hat \q(z) := \barq_D(z) \one[z\in D]$ is also a dual optimum for both problems.
\end{lemma}

\begin{corollary}
\label[corollary]{cor:hc_split}
  Let $D$ be a superset of a difficult set, $\cD\subseteq D$. Then
  $\cE(w;\mu_D) \le \cE(w)$ and
  $\cR(w;\mu_{D^c}) \le \cE(w)$
  for all $w\in L_1(\cH)$.
%    \mint_D\ell(Aw)d\mu - \inf_{v\in L_1(\cH)} \mint_D \ell(Av)\mu
%    &\leq
%    \mint\ell(Aw)d\mu - \inf_{v\in L_1(\cH)} \mint \ell(Av)d\mu,
%    \\
%    \mint_{D^c}\ell(Aw)d\mu
%    &\leq
%    \mint\ell(Aw)d\mu - \inf_{v\in L_1(\cH)} \mint \ell(Av)d\mu.
\end{corollary}

We wrap up this section by defining the class $\Lbnd$ appearing in our main
results. While the class may appear restrictive, it
contains the logistic and exponential losses by \Cref{fact:bbL3_containment}:
\begin{definition}
\label[definition]{def:lbnd}
  The class $\Lbnd\subseteq\Ldiff$ consists of strictly convex, twice continuously
  differentiable classification losses $\ell$, which in addition satisfy the following conditions:
\begin{enumromansquash}
\item The link function $\phi$, derived from $\ell$ as in \Eq{phi-def},
        is Lipschitz-continuous with constant $L_\phi$.
\item For some $c_\ell>0$,
      the derivative $\ell'$ satisfies $\ell'(r) \leq c_\ell \ell(r)$ whenever $r\le 0$.
\item For every finite measure $\mu$ over $\cZ$,
      there exists $c_{\ell,\mu} \geq 0$ with
      $\|f\|_{\beta} \leq c_{\ell,\mu} \int \ell(f)d\mu$
      for every measurable $f:\cZ\to\R_+$.
\end{enumromansquash}
\end{definition}

\subsection{Proof outline for \Cref{fact:convergence:simplified}}
\label{subsec:sketch:convergence}

%\MJT{I think I took care of this:}
% \miro{make it clearer here how the lemma+corollary implies that we can achieve zero error on $\cD^c$:
% by minimizing risk, excess risk goes to zero and risk on $\cD^c$ goes to zero.}

Recall that our goal is to show that risk minimization yields
convergence of $\eta_w$ to $\bar\eta$.
%As discussed earlier, we separately analyze risk minimization
%over easy and difficult sets.
First consider the easy set $\cD^c$. By \Cref{cor:hc_split},
minimizing $\cR(w)$, i.e., taking $\cE(w)$ to zero,
leads to $\cR(w;\mu_{\cD^c})$ becoming arbitrarily small.
This in turn means that most predictions $(Hw)(x)$ will not only have the correct sign, but will also have a large margin.
%
%indeed, by Markov's inequality,
%\[
%  \mu\left(\left\{
%      z\in \cD^c : \ell\bigParens{(Aw)(z)} \geq \sqrt{\cR(w;\mu_{\cD^c})}
%  \right\}\right)
%  \leq \frac {\cR(w;\mu_{\cD^c})}{\sqrt{\cR(w;\mu_{\cD^c})}}
%  = \sqrt{\cR(w;\mu_{\cD^c})}
%  \leq \sqrt{\cE(w)}.
%\]
%Generalizing and massaging this derivation gives
This observation can be used to obtain
the following bounds on a partition of the easy set $\cD^c$ into two sets:
$S_r$ and $\cD^c\setminus S_r$. The bound on $\mu(S_r)$ is also
a bound on $\int_{S_r} |\bar\eta-\eta_w|d\mu$ because $|\bar\eta-\eta_w|\le 1$.
Thus, together these bound $\int_{\cD^c} |\bar\eta-\eta_w|d\mu$.
%
%old version:
% The easier task is to control $\cD^c$, where the above result grants the possibility of zero error.
% While there might not exist a fixed $w\in L_1(\cH)$ which accomplishes this goal, note by Markov's
% inequality that \MJT{need to update this with the new ``$r$'' in \Cref{fact:Dc_controls}.}
% \[
%   \mu\left(\left\{
%       z\in \cD^c : \ell\bigParens{(Aw)(z)} \geq \sqrt{\cR(w;\mu_{\cD^c})}
%   \right\}\right)
%   \leq \frac {\cR(w;\mu_{\cD^c})}{\sqrt{\cR(w;\mu_{\cD^c})}}
%   = \sqrt{\cR(w;\mu_{\cD^c})}.
% \]
% Massaging the preceding derivation and using the fact $\bar\eta$ is 0 along $\cD^c$
% yields the following bound on $\|\eta_w - \bar\eta\|_1$ along $\cD^c$.

% \MJT{elsewhere I had wondered if the set $C$ is a necessary generalization?  since i had relaxed some other stuff in the findim
%   section.  Well, it's still used in the findim section, in the proof of \Cref{fact:findim:gen:helper}, since the hard core
% for other losses may be smaller than the ``universal'' one provided by $\ell=\exp$.}
%
%some deleted discussion:
% \miro{can we write this lemma for $\Lbnd$ rather than having different assumptions for its parts?}
% \MJT{Unfortunately, the proof of part 1 of \Cref{fact:findim:gen:helper} applies \Cref{fact:Dc_controls}
% for general $\ell\in\Lclass$, and I'd like to keep \Cref{fact:findim:gen:helper} at that level of generality.}
\begin{lemma}
  \label[lemma]{fact:Dc_controls}
  Given a finite measure $\mu$,
  hypotheses $\cH$,
  loss $\ell \in \Lclass$,
  and a difficult set $\cD$,
  let $D$ be an arbitrary (measurable) superset of the difficult set: $\cD\subseteq D$.
  Let any $w\in L_1(\cH)$ and $r > 0$ be given,
  and define $S_r:= \{ z \in D^c : \ell\bigParens{(Aw)(z)} \geq r \}$.
  Then:
  \begin{enumromansquash}
  \item $\mu(S_r) \leq \cR(w;\mu_{D^c})/r \leq \cE(w)/r$,
  \item
  $
    %\displaystyle
    \int_{D^c\setminus S_r}\left|
      \bar \eta - \eta_w
    \right|d\mu
    \leq
    r \mu(D^c\setminus S_r)\max\Braces{ 1/\ell(0),\; c_\ell/\ell'(0)}
    \quad\text{if $\ell \in \Lbnd$}$.
  \end{enumromansquash}
 %setting $\epsilon := \varepsilon(w,\ell,\mu,\cH)$
 %and $S:= \{ z \in D^c : \ell\bigParens{(Aw)(z)} \geq \sqrt{\epsilon}\}$ for convenience,
 %then $\mu(S) \leq \sqrt{\epsilon}$, and if additionally $\ell\in\Lbnd$,
 %then
 %\begin{align*}
 % %\mu(S)
 % %\leq \sqrt{\epsilon}
 % %\qquad\textup{and}\qquad
 %  \int_{D^c}\left|
 %    \bar \eta - \eta_w
 %  \right|d\mu
 %  \leq
 %  \mu(S)
 %  + \sqrt{\epsilon}\mu(D^c)\max\left\{ \frac {1}{\ell(0)}, \frac {c_\ell}{\ell'(0)}\right\}.
 %\end{align*}
 %where $\ell'$ for nondifferentiable $\ell$ signifies an arbitrary element of the subdifferential
 %set, and $\ell'^{-1}(r) := \sup\{q \in \R : \ell'(r) \geq \ell'(q)\}$.
 %\MJT{fix the $\ell'$ notation.}
 %\MJTDEBATE{include example simplified bound?  for exploss, $\ell'^{-1} = \ln$, thus second term is
 %$\nu(D^c) / (\sqrt{\epsilon} + 1/\sqrt{\epsilon}) \leq \nu(D^c)/2$.}
\end{lemma}

It remains to control $\eta_w$ over $\cD$.
%
%The key here is the decorrelation constraint $A^\top \barq=0$,
%which can be rearranged to say, for any $w\in L_1(\cH)$, that
%\begin{equation}
%  \mint_{Aw < 0} (Aw) \barq\,d\mu = -\mint_{Aw > 0} (Aw) \barq\,d\mu.
%  \label{eq:D_split}
%\end{equation}
%Expanding the definition of $A$, the inequality $0 < (Aw)(x,y) = -y(Hw)(x)$ means $Hw$ predicted incorrectly; consequently,
%\Eq{D_split} says that correct and incorrect predictions are balanced under
%reweighting $\barq$ for every $w\in L_1(\cH)$!
%
%
%
As mentioned earlier, the decorrelation constraint implies that the difficult set $\cD$ cannot be ``separated'' in the sense that any $w$ under which some subset of $\cD$ with a positive measure $\mu$ has a positive margin (i.e., correct predictions), also yields a positive measure of points in $\cD$ with a negative margin (i.e., incorrect predictions). Since the loss is increasing over negative margins, this structure implies that the risk over $\cD$ has a minimizer over each one-dimensional subspace (similar reasoning to the example of \Cref{fig:difficult:set}).
This one-dimensional property can be used in finite dimensions to argue that the risk must have a minimizer over the difficult set, and we pursue this line of reasoning in \Cref{subsec:sketch:findim:gen}.
But here, we need an alternative approach.

As discussed in \Cref{fact:duality}.iv, if $\ell\in\Ldiff$, then
$\bar\eta(x,y)=\phi(-\barf(x,y))$ with $\barf(x,y)=(\ell')^{-1}(\barq(x,y))$ whenever
$(\ell')^{-1}$ is defined for both $\barq(x,y)$ and $\barq(x,-y)$. Fortunately, this can be shown
to hold $\mu$-a.e.\@ over $\cD$. Thus, over $\cD$,
we can write $\abs{\bar\eta-\eta_w}=\bigAbs{\phi(-\barf)-\phi(-Aw)}$. The next
lemma uses a second-order Taylor expansion at $\barf$ to further derive a bound on this difference.

In \Cref{fact:D_controls}, we split
the difficult set into four subsets
and we either bound their mass, which in turn bounds
the integral of $\abs{\bar\eta-\eta_w}$, or
directly bound the integral. The integral is controlled directly over the subset $U$ by the mentioned Taylor bound, and so it requires the bounds on the range of $Aw$ and $\barf$ (via $\barq$), and a corresponding lower bound $\tau$ on the second derivative. The subset $S_+$ contains points with a large loss, so its mass is controlled by the risk. The control of the
subset $S_-$ is the most technical. The set includes points where the predictions are correct, but the density $\barq$ is large. The bound is based on the decorrelation constraint as well as property (iii) in \Cref{def:lbnd}. All three bounds depend on $w$ only via its risk; this is indeed key to establishing
\Cref{fact:convergence:simplified}. The set $V$ needs to be controlled separately.

%old version:
% $\cR$ must increase past some sweet spot.  Ideally, a sweet spot could be pinned down in the form
% of a primal optimum over $\cD$, which for instance is possible in the finite dimensional setting
% (cf. \Cref{fact:findim:gen:helper}).
% There is no such luck in the infinite dimensional case, however, even over $\cD$
% (cf. \Cref{ex:primal_optimum_failure}).  Even so, the primal optimum can be simulated; if first order optimality conditions held,
% the primal optimum would look like $(\ell^*)'(\barq)$.
% This synthetic primal optimum is good enough to stick into a Taylor expansion, applied \emph{pointwise},
% which then translates into a bound on $\|\eta_w - \bar\eta\|_1$ along $\cD$.
% % (Note that this use of Taylor expansion is why \Cref{fact:D_controls} will need $\ell$ to be twice continuously differentiable,
% % whereas \Cref{fact:Dc_controls} did not.)

\begin{lemma}
  \label[lemma]{fact:D_controls}
  Given a finite measure $\mu$ with $\mu(\cZ)\le 1$,
  hypotheses $\cH$,
  loss $\ell \in \Lbnd$,
  a difficult set $\cD$ and an associated dual optimum $\barq$,
  let a weighting $w\in L_1(\cH)$ be given, along with scalars $c_1 > 0$, $c_2 > 0$, $c_3 > c_2$,
  and
 %$\barf(z) := (\ell^*)'(\barq(z))$ for $z\in\cD$ and %XXX no longer needed.
  $\tau := \min\{ \inf_{|z|\leq c_1} \ell''(z) , \inf_{z\in[c_2,c_3]} \ell''((\ell^*)'(z)) \}$.
  Define the following sets:
  \begin{align*}
    U &\coloneqq \Set{z\in \cD : |(Aw)(z)| \leq c_1 \text{ and } c_2 \leq \barq(z) \leq c_3 },
&   S_+ &\coloneqq \Set{ z \in \cD : (Aw)(z) > c_1 },
\\  S_- &\coloneqq \Set{ z \in \cD : (Aw)(z) < -c_1 \text{ and } \barq(z) \ge c_2 },
&  V &\coloneqq \Set{ z \in \cD : \barq(z) < c_2 \text{ or } \barq(z) > c_3}.
  \end{align*}
  Then $\cD=U\cup S_+\cup S_-\cup V$,
  \[
    \mu(S_+)  \le
    \frac {\cR(w)}{c_1\ell'(0)},
    \quad
      %\mint \ell(Aw)d\mu
%    \qquad\textup{and}\qquad
 %\\
    \mu(S_-)  \le
    \frac {2 c_{\ell,\mu} \|\barq\|_{\beta^*}\cR(w)}{c_1c_2},
    \quad
      %\mint \ell(Aw)d\mu
    \mint_{U}\left|
      \bar\eta
      -
      \eta_w
      \right|d\mu
      \leq
      L_\phi\sqrt{\frac {2\cE(w;\mu_\cD)}{\tau}}.
%      \leq
%      L_\phi\sqrt{\frac {2\cE(w)}{\tau}}.
  \]
\end{lemma}
To prove \Cref{fact:convergence:simplified} from here,
first split $\int|\eta_w - \bar \eta|d\mu$ along $\cD$ and $\cD^c$,
and apply \Cref{fact:Dc_controls} and \Cref{fact:D_controls} to the two pieces;
the goal is to show that all terms go to zero as $\cE(w)\to 0$.
In the terms resulting from \Cref{fact:Dc_controls},
this is handled by
the choice $r:=\sqrt{\cE(w)}$.  Similarly, it is possible (although considerably more challenging)
to balance $c_1,c_2,c_3,\tau$ arising from \Cref{fact:D_controls}.

\subsection{Proof outline for \Cref{fact:findim:gen:simplified}}
\label{subsec:sketch:findim:gen}

In this section we sketch the proof of the generalization bound from the introduction
(\Cref{fact:findim:gen:simplified}). Unlike the foregoing results,
here we assume that the hypothesis space is finite, $|\cH|=d$.

Similar to \Cref{subsec:sketch:convergence}, the proof treats the easy set and the difficult
set separately. On the easy set, where zero risk is possible in the limit, linear predictors actually
achieve zero \emph{classification error} when viewed as half-space classifiers. Finite dimension $d$
then implies a finite VC dimension and the corresponding generalization bound. In the remainder,
we only focus on the difficult set.

We build on the fact that on the difficult set $\cD$ the risk is eventually increasing along any direction
which lies in the ``span'' of $\cD$ (similar to the example of \Cref{fig:difficult:set}).
In the
finite dimension $d$, this will imply a bound on the norm of the
optimizer of risk over $\cD$, and also enable the application of Rademacher complexity
to obtain a generalization bound.

We begin with a specific lower bound in each direction $w$ within the ``span''. The bound is obtained
by integrating over all points with a negative margin, i.e., $(Aw)(z)>0$.
Because of the lack of separators over $\cD$, the bound is non-zero. Taking
an infimum over all directions yields a uniform bound called \emph{balance}.
%
%alluding to
%to the fact that the difficult set ``balances'' non-zero mass of positive-margin and negative-margin
%examples along every ``non-orthogonal'' direction $w$ (under $\barq$).
%
While the following definition
is written for any measure $\mu$, it is going to be primarily applied with $\mu_\cD$ substituted for $\mu$:
\begin{definition}
  \label{defn:findim:ker_bal}
  The \emph{balance} associated with hypotheses $\cH$, $|\cH|=d$, and measure $\mu$ is defined as
  $
    \Bal(\mu)
    \coloneqq \inf\left\{
    \int |(Aw)(z)|_+\,d\mu(z) : w\in \Ker(\mu)^\perp, \|w\|_1= 1\right\}
  $,
  where $|s|_+\coloneqq\max\{s,0\}$ denotes the non-negative part, and
  $\Ker(\mu) \coloneqq \bigBraces{ w\in\R^d : (Aw)(z)=0, \text{ $\mu$-a.e.\@ over $z$} }$
  denotes the subspace of $\R^d$ with no effect on risk under $\mu$.
\end{definition}
The ``span''
corresponds to the orthogonal complement of the kernel $\Ker(\mu)$. In the example of \Cref{fig:difficult:set},
the difficult set consisted of the points on the line $\gamma$, and the kernel $\Ker(\mu_\cD)$ was the subspace
spanned by the vector $w_1$, which had no effect on the risk over points on $\gamma$.
The only interesting directions from the perspective of this risk were in the orthogonal complement $\Ker(\mu_\cD)^\perp$.

In finite dimension $d$, we obtain that $\Bal(\mu_\cD)>0$ whenever $\mu(\cD)>0$ (\Cref{fact:findim:D:0}).
This yields a non-trivial risk bound from the definition of balance, using the fact that $\ell(r)\geq \ell(0) + r\ell'(0) \geq r\ell'(0)$ (by convexity and non-negativity of $\ell$):
\[
  \cR(w;\mu_\cD)
  \geq \mint_{Aw > 0} \ell(Aw)d\mu_\cD
  \geq \mint_{Aw > 0} \ell'(0) (Aw) d\mu_\cD \geq \ell'(0) \|w\|_1 \Bal(\mu_\cD).
\]
Rearranging, we also obtain a norm bound $\|w\|_1 \leq \cR(w) / (\ell'(0)\Bal(\mu_\cD))$, which enables the use of
Rademacher complexity in the analysis of generalization on $\cD$.
%Also,
%as a corollary, we obtain the existence of a primal optimum $\bar w$ for $\cR(\cdot;\mu_{\cD})$ (\Cref{fact:findim:bar_w})

A less obvious consequence is that
for a given finite hypothesis class $\cH$
and measure $\mu$, there exists a maximal difficult set.
%this means that the difficult set for any particular loss in $\Ldiff$ must contain the difficult sets for
%all other losses (again, $\mu$-a.e.).
This difficult set, common to the entire class $\Ldiff$, is called
the \emph{canonical difficult set} $\Dcan$ (for concreteness, we define it for $\ell=\exp$).
Informally, its existence follows from
the property shared by all losses $\ell\in\Ldiff$ that
$(\ell^*)'(s)\uparrow \infty$ as $s\downarrow 0$ (\Cref{fact:loss_prop}.v); consequently,
the optimization prevents $\barq$ from taking on the value zero unless forced by constraints,
and thus yields the largest possible difficult set:
%since it is the set of $z$ with $\q(z)>0$,
%
\begin{definition}
  \label{def:findim:D}
  For a finite measure $\mu$ and a hypothesis set with $|\cH|<\infty$,
  the \emph{canonical difficult set} $\Dcan$
  is defined as any difficult set associated with $\ell=\exp$.
\end{definition}

\begin{proposition}
  \label[proposition]{fact:findim:Dcan}
  Given a finite measure $\mu$, a hypothesis set with $|\cH|<\infty$,
  and the corresponding canonical difficult set $\Dcan$,
  we have:
  \begin{enumromansquash}
    \item
      For any $\ell \in \Lclass$ and any corresponding difficult set $\cD$,
      we have $\cD\subseteq\Dcan$ $\mu$-a.e.
    \item
      For any $\ell \in \Ldiff$ and any corresponding difficult set $\cD$,
      we have $\cD=\Dcan$ $\mu$-a.e.
      %XXX natural to wonder whether this is tight, namely that taking some loss in \Lclass \setminus \Ldiff
      %XXX results in a strictly smaller hard core (in the measure-theoretic sense: strictly smaller measure).
      %XXX this should be true but I didn't resolve it, as it's not needed here..
  \end{enumromansquash}
\end{proposition}

We finish this section with the Rademacher complexity style bound on the
excess risk over the canonical difficult set $\cD^*$, based on the norm bound implied by the balance. The
key insight is that the quantities in the bound depend on $w$ only through the empirical risk $\cR(w;d\hmu_{|\Dcan})$.
\Cref{fact:findim:gen:simplified} is then proved by splitting
$\int |\eta_w - \bar\eta|d\mu$ along $\Dcan$ and $\Dcan^c$, and controlling the pieces
by a combination of \Cref{fact:Dc_controls} with the VC style bound (\Cref{fact:findim:gen:helper:Dc}) used
to select $r$, and \Cref{fact:D_controls} with the scalars chosen via \Cref{fact:findim:gen:helper:D}.
\begin{lemma}
  \label[lemma]{fact:findim:gen:helper:D}
  Let probability measure $\mu$,
  hypotheses $\cH$ with $|\cH|=d$,
  loss function $\ell \in \Lclass$,
  subgradient $\bars \in \partial \ell(0)$,
  and a canonical difficult set $\Dcan$ with $\mu(\Dcan) > 0$
  be given.
  Set
  $\tau{(r)} := \inf_{|z| \leq r} \ell''(z)$,
  $\Bal_\star\coloneqq\Bal(\mu_{|\Dcan})$
  and let
  $
     B_w := 2 + \bigBracks{\ell(0) + 2 \cR(w;d\hmu_{|\Dcan})}\big/(\bars\Bal_\star)
  $,
  and
  $
  n \ge 256 \ln(8d/\delta) \big/ \Bal_\star^2
  $.
  Then
  with probability at least $1-4\delta$ over a draw from $\mu_{|\Dcan}$ of size $n$,
  the following statements hold simultaneously for every $w\in L_1(\cH)$:
  \begin{enumromansquash}
    \item
      $\displaystyle
        |(Aw)(z)| \leq B_w
        \quad\text{for $\mu$-a.e. and $\hmu$-a.e. $z\in\Dcan$}
      $.

    \item
      $\displaystyle
        \cE(w, \mu_{|\Dcan})
        \leq
        \cE(w, \hmu_{|\Dcan})
        +
        10\ell(2B_w)\sqrt{\ln(8dB_w^2/\delta) \big/n}
      $.
    \item
      $\displaystyle
        \cE(w, \mu_{|\Dcan})
        \leq
        2\cE(w, \hmu_{|\Dcan})
        +
        \frac {1024 \ell'(2B_w)^2 \ln(8dB_w^2/\delta)}{n \Bal_\star^2 \tau{(B_w)}}
      \quad\text{if $\ell \in \Lbnd$}
      $.
  \end{enumromansquash}
\end{lemma}

\acks{The authors thank Rastislav Telg{\' a}rsky for pointing out that the topology should always be adapted to the problem at hand;
may he rest in peace.}

\addcontentsline{toc}{section}{References}
%\bibliographystyle{plainnat}
%following triggers an error when used with empty citation list and ntheorem loaded..
\bibliography{ab}

\begin{thebibliography}{28}
\providecommand{\natexlab}[1]{#1}
\providecommand{\url}[1]{\texttt{#1}}
\expandafter\ifx\csname urlstyle\endcsname\relax
  \providecommand{\doi}[1]{doi: #1}\else
  \providecommand{\doi}{doi: \begingroup \urlstyle{rm}\Url}\fi

\bibitem[Altun and Smola(2006)]{AltunSm06}
Yasemin Altun and Alex Smola.
\newblock Unifying divergence minimization and statistical inference via convex
  duality.
\newblock 2006.

\bibitem[Bartlett and Mendelson(2002)]{bartlett_mendelson_rademacher}
Peter~L. Bartlett and Shahar Mendelson.
\newblock Rademacher and gaussian complexities: Risk bounds and structural
  results.
\newblock \emph{JMLR}, 3:\penalty0 463--482, Nov 2002.

\bibitem[Bartlett et~al.(2005)Bartlett, Bousquet, and
  Mendelson]{bartlett_local_rademacher}
Peter~L. Bartlett, Olivier Bousquet, and Shahar Mendelson.
\newblock Local rademacher complexities.
\newblock \emph{The Annals of Statistics}, 33\penalty0 (4):\penalty0
  1497--1537, 08 2005.
\newblock \doi{10.1214/009053605000000282}.

\bibitem[Bartlett et~al.(2006)Bartlett, Jordan, and
  McAuliffe]{bartlett_jordan_mcauliffe}
Peter~L.\ Bartlett, Michael~I.\ Jordan, and Jon~D.\ McAuliffe.
\newblock Convexity, classification, and risk bounds.
\newblock \emph{Journal of the American Statistical Association}, 101\penalty0
  (473):\penalty0 138--156, 2006.

\bibitem[Boucheron et~al.(2005)Boucheron, Bousquet, and Lugosi]{bbl_esaim}
St{\'e}phane Boucheron, Olivier Bousquet, and G{\'a}bor Lugosi.
\newblock Theory of classification: a survey of recent advances.
\newblock \emph{ESAIM: Probability and Statistics}, 9:\penalty0 323--375, 2005.

\bibitem[Chang and Pollard(1997)]{pollard_disintegration}
Joseph~T. Chang and David Pollard.
\newblock Conditioning as disintegration.
\newblock \emph{Statistica Neerlandica}, 51\penalty0 (3):\penalty0 287--317,
  1997.

\bibitem[Collins et~al.(2002)Collins, Schapire, and
  Singer]{collins_schapire_singer_adaboost_bregman}
Michael Collins, Robert~E. Schapire, and Yoram Singer.
\newblock Logistic regression, {A}da{B}oost and {B}regman distances.
\newblock \emph{Machine Learning}, 48\penalty0 (1-3):\penalty0 253--285, 2002.

\bibitem[Devroye et~al.(1996)Devroye, Gy\"orfi, and Lugosi]{DGL}
L.~Devroye, L.~Gy\"orfi, and G.~Lugosi.
\newblock \emph{A probabilistic theory of pattern recognition}.
\newblock Springer, 1996.

\bibitem[Friedman et~al.(2000)Friedman, Hastie, and
  Tibshirani]{friedman_hastie_tibshirani_statboost}
Jerome Friedman, Trevor Hastie, and Robert Tibshirani.
\newblock Additive logistic regression: a statistical view of boosting.
\newblock \emph{Annals of Statistics}, 28\penalty0 (2):\penalty0 337--407,
  2000.

\bibitem[Friedman(2000)]{friedman_gradient_boosting}
Jerome~H. Friedman.
\newblock Greedy function approximation: A gradient boosting machine.
\newblock \emph{Annals of Statistics}, 29:\penalty0 1189--1232, 2000.

\bibitem[Guruswami and Raghavendra(2006)]{raghavendra_halfspace_hardness}
Venkatesan Guruswami and Prasad Raghavendra.
\newblock Hardness of learning halfspaces with noise.
\newblock In \emph{FOCS}, 2006.

\bibitem[Hiriart-Urruty and Lemar\'echal(2001)]{HULL}
Jean-Baptiste Hiriart-Urruty and Claude Lemar\'echal.
\newblock \emph{Fundamentals of Convex Analysis}.
\newblock Springer Publishing Company, Incorporated, 2001.

\bibitem[Kearns and Vazirani(1994)]{kearns_vazirani}
Michael Kearns and Umesh Vazirani.
\newblock \emph{An introduction to computational learning theory}.
\newblock MIT Press, 1994.

\bibitem[L{\' e}onard(2007)]{leonard_orlicz}
Christian L{\' e}onard.
\newblock Orlicz spaces.
\newblock \url{http://www.cmap.polytechnique.fr/~leonard/papers/orlicz.pdf},
  2007.
\newblock Accessed 2015-04-28.

\bibitem[L{\' e}onard(2008)]{leonard_entropy}
Christian L{\' e}onard.
\newblock Minimization of entropy functionals.
\newblock \emph{J. Math. Anal. Appl.}, 346:\penalty0 183--204, 2008.

\bibitem[Levy et~al.(2014)Levy, Hazan, and Koren]{colt_logistic_hazan}
Kfir Levy, Elad Hazan, and Tomer Koren.
\newblock Logistic regression: Tight bounds for stochastic and online
  optimization.
\newblock In \emph{COLT}, 2014.

\bibitem[Mukherjee et~al.(2011)Mukherjee, Rudin, and
  Schapire]{mukherjee_rudin_schapire_adaboost_convergence_rate}
Indraneel Mukherjee, Cynthia Rudin, and Robert Schapire.
\newblock The convergence rate of {A}da{B}oost.
\newblock In \emph{COLT}, 2011.

\bibitem[Rockafellar(1968)]{roc_conv_int_1}
R.~Tyrrell Rockafellar.
\newblock Integrals which are convex functionals {I}.
\newblock \emph{Pacific J. Math.}, 24:\penalty0 525--539, 1968.

\bibitem[Rockafellar(1970)]{ROC}
R.~Tyrrell Rockafellar.
\newblock \emph{Convex Analysis}.
\newblock Princeton University Press, 1970.

\bibitem[Rockafellar(1974)]{Rockafellar74}
R.~Tyrrell Rockafellar.
\newblock \emph{Conjugate Duality and Optimization}.
\newblock SIAM Publications, 1974.

\bibitem[Schapire and Freund(2012)]{schapire_freund_book_final}
Robert~E. Schapire and Yoav Freund.
\newblock \emph{Boosting: Foundations and Algorithms}.
\newblock MIT Press, 2012.

\bibitem[Schapire et~al.(1997)Schapire, Freund, Bartlett, and
  Lee]{boosting_margin}
Robert~E. Schapire, Yoav Freund, Peter Bartlett, and Wee~Sun Lee.
\newblock Boosting the margin: A new explanation for the effectiveness of
  voting methods.
\newblock In \emph{ICML}, pages 322--330, 1997.

\bibitem[Shalev-Shwartz and Ben-David(2014)]{shai_shai_book}
Shai Shalev-Shwartz and Shai Ben-David.
\newblock \emph{Understanding Machine Learning: From Theory to Algorithms}.
\newblock Cambridge University Press, 2014.

\bibitem[Shalev-Shwartz et~al.(2008)Shalev-Shwartz, Srebro, and
  Sridharan]{karthik_sc_fastrates}
Shai Shalev-Shwartz, Nathan Srebro, and Karthik Sridharan.
\newblock Fast rates for regularized objectives.
\newblock In \emph{NIPS}, 2008.

\bibitem[Telgarsky(2012)]{primal_dual_boosting}
Matus Telgarsky.
\newblock A primal-dual convergence analysis of boosting.
\newblock \emph{JMLR}, 13:\penalty0 561--606, 2012.

\bibitem[Telgarsky(2013)]{mjt_logistic}
Matus Telgarsky.
\newblock Boosting with the logistic loss is consistent.
\newblock In \emph{COLT}, 2013.

\bibitem[Zhang(2004)]{zhang_convex_consistency}
Tong Zhang.
\newblock Statistical behavior and consistency of classification methods based
  on convex risk minimization.
\newblock \emph{The Annals of Statistics}, 32:\penalty0 56--85, 2004.

\bibitem[Zhang and Yu(2005)]{zhang_yu_boosting}
Tong Zhang and Bin Yu.
\newblock Boosting with early stopping: Convergence and consistency.
\newblock \emph{The Annals of Statistics}, 33:\penalty0 1538--1579, 2005.

\end{thebibliography}

\appendix

\section{Convex analysis in Banach spaces}
\label{sec:banach}

This appendix covers convex analysis results for functional spaces. It is based on \citet{Rockafellar74} and \citet{roc_conv_int_1}.

\paragraph{Banach spaces.}
A \emph{Banach space} is a complete normed vector space.
%The completeness condition requires that all sequences $(x_i)_{i=1}^\infty$ satisfying
%$\lim_{i,j\to\infty}\norm{x_i-x_j}=0$ have limits in the space.
The space $\R^n$ with the Euclidean norm is a Banach space. Given a measure $\mu$ on $\cZ$ and $p\ge 1$, the Banach space $L_p(\mu)$
consists of all measurable functions $f:\cZ\to\R$ with the finite norm
$
  \norm{f}_p\coloneqq\Parens{\int |f|^p\,d\mu}^{1/p}
$.

The analog of an inner product for Banach spaces is a \emph{pairing}. Given two Banach spaces $U$ and $V$, their pairing is described by a bilinear form $U\times V\to \R$, denoted $\angles{u,v}$. Thus, each $u\in U$ describes a linear map $v\mapsto \angles{u,v}$ on $V$ and vice versa. Each Banach space is endowed with the topology implied by its norm, but other topologies are possible. We say that the topologies on $U$ and $V$ are \emph{compatible} with the pairing if the linear functions described by $u\in U$ and $v\in V$ are continuous, and if they comprise all continuous linear functions on $V$ and $U$, respectively. A Euclidean space $\R^n$ with the norm topology is compatibly paired with itself via standard inner product.
Given $1<p,q<\infty$ such that $1/p+1/q=1$, the spaces $L_p(\mu)$ and $L_q(\mu)$ with their norm topologies are a compatible pairing with the bilinear form
$
  \angles{f,g}=\int fg\,d\mu
$.
One construction of compatible pairings begins with a Banach space $U$ under norm topology, then takes its \emph{topological dual} $U'$, i.e.,
the space of all continuous linear functions on $U$, and endows $U'$ with the \emph{weak${}^*$ topology}.
In the rest of the paper, when we talk about ``paired Banach spaces'' we assume that they have been endowed with compatible topologies.

\paragraph{Convexity, conjugacy, subgradients.} Given a Banach space $U$, a function $F:U\to(-\infty,\infty]$ is called \emph{proper}
if it is not equal to $\infty$
everywhere. The set of points where $F$ is finite is called its \emph{domain} and denoted $\dom F$.
The \emph{epigraph} of $F$ is the set of points above the graph of the function $\Set{(u,t):\:u\in U,\,t\in\R,\,t\ge F(u)}$. The function $F$ is called \emph{convex} if its epigraph is convex. It is called \emph{closed} if its epigraph is closed.

Let $U$ and $V$ be paired Banach spaces.
Let $F:U\to(-\infty,\infty]$ be a closed proper convex function. The \emph{conjugate} of $F$ is defined by
$
  F^*(v) \coloneqq \sup_{u\in U} \bigBracks{ \angles{u,v}-F(u) }
$.
It is also a closed proper convex function and $F^{**}=F$ (Theorem 5 of \citealp{Rockafellar74}). From the definition of a conjugate, we get \emph{Fenchel's inequality}
\[
   F(u)+F^*(v)\ge\angles{u,v}
\enspace.
\]
The \emph{subgradient} of $F$ at $u$ is the set
$
  \partial F(u) \coloneqq \Set{v\in V:\: F(u') \ge F(u) + \angles{u'-u,v}\text{ for all }u'\in U}
$.
For a closed proper convex function $F$, the following statements
are equivalent (Corollary 12A and the foregoing discussion of \citealp{Rockafellar74}) (first-order optimality for conjugates):
\begin{enumroman}
\item $F(u)+F^*(v)=\angles{u,v}$,
\item $v\in\partial F(u)$,
\item $u\in\partial F^*(v)$.
\end{enumroman}

\paragraph{Integrals as convex functionals.} Consider a finite measure $\mu$ on $\cZ$, and assume
we are given a pairing of Banach spaces $U$ and $V$ via bilinear form $\angles{u,v}=\int u(z)v(z)\,d\mu(z)$,
i.e., $U$ and $V$ are subsets of measurable functions on $\cZ$. Let $f:\R\to(-\infty,\infty]$ be
a closed proper convex function. We study properties of the function $F$ on $U$ defined by the integral
\[
  F(u)=\int f\bigParens{ u(z) }\,d\mu(z)
\enspace.
\]
To establish its closedness and study conjugacy we need the following definition, adapted from
\citet{roc_conv_int_1} for the case of a finite measure $\mu$:
\begin{definition}
\label[definition]{def:decomposable}
We say that a Banach space of measurable functions on $\cZ$ is decomposable with respect to a finite measure $\mu$
if the following conditions hold:
\begin{enumroman}
\item $U$ contains every bounded measurable function from $\cZ$ to $\R$.
\item If $u\in U$ and $E$ is a measurable set, then $U$ contains $u\cdot\one_E$ where $\one_E$ is the
indicator of the set $E$.
\end{enumroman}
\end{definition}
The following proposition is a rephrasing of the corollary on page 534 of \citet{roc_conv_int_1}:
\begin{proposition}
\label[proposition]{prop:conj:int}
If $\mu$ is finite and $U$ and $V$ are decomposable then $F(u)$ is a closed proper convex
function, and its conjugate is
\[
  F^*(v)=\int f^*\bigParens{ v(z) }\,d\mu(z).
\]
\end{proposition}
Next proposition presents two additional results relating the properties of $F$ and $f$:
\begin{proposition}
\label[proposition]{prop:subgrad:strict}
If $\mu$ is finite and $U$ and $V$ are decomposable then
\begin{enumroman}
\item $v\in\partial F(u)$ if and only if $v(z)\in\partial f\bigParens{u(z)}$, $\mu$-a.e.\@ over $z$.
\item If $f$ is strictly convex, then so is $F$.
\end{enumroman}
\end{proposition}
\begin{proof}
To show part (i), use first-order optimality for conjugates to obtain that $v\in\partial F(u)$ if and only if
\begin{equation}
\label{eq:subgrad:int}
   F(u)+F^*(v)=\angles{u,v}
\enspace.
\end{equation}
Since Fenchel's inequality holds pointwise, i.e., $f\bigParens{u(z)}+f^*\bigParens{v(z)}\ge u(z)v(z)$, \Eq{subgrad:int}
is equivalent to
\[
   f\bigParens{u(z)}+f^*\bigParens{v(z)}= u(z)v(z),\quad\text{$\mu$-a.e.\@ over $z$,}
\]
which, again by first-order optimality for conjugates, is equivalent to
\[
   v(z)\in\partial f\bigParens{u(z)},\quad\text{$\mu$-a.e.\@ over $z$,}
\]
completing the proof of part (i). Part (ii) can be shown by contradiction. Assume that $F$ is not
strictly convex, i.e., $F$ is flat along a line segment connecting points $u_1$ and $u_2$ which
differ on a set of non-zero measure. Let $u=(u_1+u_2)/2$. The flatness of $F$ means that $F(u)=[F(u_1)+F(u_2)]/2$,
but pointwise, by convexity, $f\bigParens{u(z)}\le \Bracks{f\bigParens{u_1(z)}+f\bigParens{u_2(z)}}/2$,
so we must actually have $f\bigParens{u(z)}=\Bracks{f\bigParens{u_1(z)}+f\bigParens{u_2(z)}}/2$,
$\mu$-a.e.\@ over $z$. Since $u_1$ and $u_2$ differ on a set of non-zero measure, we obtain that
$f$ cannot be strictly convex.
\end{proof}

\paragraph{Fenchel's duality.} Given pairings $(X,Y)$ and $(U,V)$ of Banach spaces and a continuous linear
map $A:X\to U$, its \emph{adjoint} is a linear map $A^\top:V\to Y$ defined by
$
  \angles{x,\,A^\top v}=\angles{Ax,\,v}
$.
We finish this section by stating a version of Fenchel duality
used in this paper. It is a rephrasing of the duality in Example 11' and Eq.~(8.26) on page 50 of \citet{Rockafellar74}, adapted to stronger conditions (specifically, $\dom F=X$ and $\dom G=U$):
\begin{theorem}
\label{thm:fenchel}
Let $(X,Y)$ and $(U,V)$ be pairings of Banach spaces. Let $F:X\to\R$ and $G:U\to\R$ be closed proper convex functions and $A:X\to U$ be a continuous linear operator. Then
\[
  \inf_{x\in X} \BigBracks{ F(x) + G(Ax) }
  =
  \max_{v\in V} \BigBracks{ -F^*(-A^\top v)-G^*(v) }
\enspace.
\]
The point $\bar x$ is the primal minimizer if and only if there exists a dual maximizer $\bar{v}$ such that
\[
  -A^\top\bar v\in\partial F(\bar x)
\enspace,\quad
  \bar v\in\partial G(A\bar x)
\enspace.
\]
\end{theorem}

\section{Orlicz spaces}
\label{sec:orlicz}

The duality result of \Cref{sec:duality} is an application of Fenchel's duality (\Cref{thm:fenchel}).
As discussed in \Cref{sec:duality}, the key challenge in applying the duality is the choice of
appropriate pairings of Banach spaces. This appendix develops properties of specific
Banach spaces, called \emph{Orlicz spaces}, which will be sufficiently flexible to obtain
pairings that satisfy our desiderata.

Orlicz spaces generalize $L_p(\mu)$ spaces introduced in \Cref{sec:banach}.
The construction of an Orlicz space begins with a non-negative
convex function $\theta:\R\to[0,\infty]$ symmetric around zero, not identical to zero, and with $\theta(0)=0$,
which serves the same role as the $p$-th power
function in the construction of $L_p(\mu)$. Given the function $\theta$ and a measure $\mu$,
we first define the unit ball of functions
\[
   \cB\coloneqq\Set{f \text{ measurable}:\:\int \theta\bigParens{f(z)}d\mu(z)\le 1}
\enspace,
\]
which is then used to define the norm $\norm{\cdot}_\theta$:
\[
   \|f\|_\theta = \inf\{r\ge 0:\:f\in r\cB\}
\]
where the norm equals $\infty$ if $f$ is outside the scaled ball $r\cB$ for all $r\ge 0$.

The \emph{large Orlicz space} $L_\theta(\mu)$ and the \emph{small Orlicz space} $M_\theta(\mu)$ are
defined as
\begin{align*}
   L_\theta(\mu) &\coloneqq\Set{f\text{ measurable}:\:\exists r>0, \int \theta(rf)d\mu < \infty}
\enspace,
\\
   M_\theta(\mu) &\coloneqq\Set{f\text{ measurable}:\:\forall r>0, \int \theta(rf)d\mu < \infty}
\enspace.
\end{align*}
From the definition it is clear that
\[
  L_\theta(\mu)=\Set{f:\:\norm{f}_\theta<\infty}
\enspace,
\]
so for $p\ge 1$ and $\theta(s)=|s|^p$, we recover $L_p(\mu)$ spaces. The definition also implies $M_\theta(\mu)\subseteq L_\theta(\mu)$.
The following proposition summarizes key properties of Orlicz spaces used in this paper. Parts (i--iv) are paraphrased from
Proposition 1.4, Proposition 1.14, Proposition 1.18 and Theorem 2.2 of \citet{leonard_orlicz}:
\begin{proposition}
\label[proposition]{prop:orlicz}
Let $\mu$ be a finite measure and $\theta:\R\to[0,\infty]$ be
a closed convex function symmetric around zero, such that $\theta(0)=0$ and neither $\theta$
nor its conjugate $\theta^*$ are identically zero. Then the following hold:
\begin{enumroman}
\item $\theta^*$ is also symmetric around zero and $\theta^*(0)=0$.
\item $L_\theta(\mu)$ and $M_\theta(\mu)$ are Banach spaces with the norm $\norm{\cdot}_\theta$.
\item For all $f\in L_\theta(\mu)$ and $g\in L_{\theta^*}(\mu)$:
$
   \int\abs{fg}d\mu\le 2\norm{f}_\theta\norm{g}_{\theta^*}
$.
\item If $\theta$ is real-valued, i.e., $\dom\theta=\R$, then the topological dual of $M_\theta$
      is isomorphic to $L_{\theta^*}$.
\item If $\theta$ is real-valued, i.e., $\dom\theta=\R$, then $M_\theta$
      and $L_{\theta^*}$ are decomposable.
\end{enumroman}
\end{proposition}
\begin{proof}\textbf{of (v)}
Let $f$ be a bounded measurable function, say $|f|\le a$. Then $\theta(rf(z))\le\theta(ra)$,
so
\[
  \int\theta(rf)d\mu\le\theta(ra)\mu(\cZ)<\infty
  \text{ for all $r>0$,}
\]
implying $f\in M_\theta(\mu)$. Also, since $\dom\theta^*\ne\Set{0}$ and $\theta^*(0)=0$, there must
be some $\eps>0$ such that $[-\eps,\eps]\subseteq\dom\theta^*$, and
\[
  \int\theta^*\Parens{\frac{\eps}{a}f}d\mu\le\theta^*(\eps)\mu(\cZ)<\infty
\]
implying $f\in L_\theta^*(\mu)$. To argue that condition (ii) of \Cref{def:decomposable} holds,
note that if $f\in M_\theta(\mu)$ then any $g$ with $|g|\le|f|$ must also be in $M_\theta(\mu)$,
and similarly for $L_{\theta^*}(\mu)$.
\end{proof}

\section{Rademacher complexity}
\label{sec:rademacher}

This section collects various results from the literature on Rademacher complexity.
To start, given a set of vectors $V\subseteq \R^n$,
and letting $\sigma\in \{-1,+1\}^n$
denote a vector of $n$ independent
Rademacher random variables (i.e., $\Pr[\sigma_i = +1] = \Pr[\sigma_i = -1] = 1/2$ for all $i$),
define the \emph{Rademacher complexity $\fR$} of $V$ as
\[
  \fR(V) := \bbE\left( \sup_{v\in V} \frac 1 n\sum_{i=1}^n v_i\sigma_i \right).
\]
To define the Rademacher complexity of a function $f$ or function class $\cF$ applied to a sample $\cS:=(z_i)_{i=1}^n$,
 define $f\circ \cS := (f(z_i))_{i=1}^n \in \R^n$,
and similarly overload  $\cF\circ \cS \subseteq \R^n$,
finally defining $\fR(\cF) := \fR(\cF\circ \cS)$.
Note that these definitions match the presentation of \emph{local Rademacher complexity} \citep{bartlett_local_rademacher},
whereas the original definition included an absolute value around the innermost summation \citep{bartlett_mendelson_rademacher,bbl_esaim}.

The essential link between Rademacher complexity and deviation bounds is as follows.
\begin{lemma}[{\citealp[Theorem 26.5]{shai_shai_book}}]
  \label[lemma]{fact:rad:dev}
  Let loss $\ell$ and function class $\cF$ be given.
  Then with probability at least $1-\delta$ over a draw of size $n$ from $\mu$,
  \begin{align*}
    \sup_{f\in\cF} \left(\mint\ell(f)d\mu - \mint\ell(f)d\hmu_n\right)
    &\leq 2\fR(\ell\circ \cF\circ \cS)
    + 4 \sup_{\substack{z\in\cS\\f\in \cF}}|\ell(f(z))| \sqrt{\frac {2\ln(4/\delta)}{n}}
    \\
    &\leq
    4\max\Bigg\{1,\ \sup_{\substack{z\in\cS\\f\in \cF}}|\ell(f(z))|\Bigg\}
    \sqrt{\frac {\fR(\ell\circ\cF\circ\cS)^2}{2} + \frac {4\ln(4/\delta)}{n}}.
  \end{align*}
  %XXX second part comes from general inequality $\sqrt a + \sqrt b \leq \sqrt{2a+2b}$ for nonnegative reals
\end{lemma}

Thanks to \Cref{fact:rad:dev}, the task of controlling deviations has been reduced to the task of approximating $\fR$.
The following bounds are used throughout.

\begin{lemma}[{See also \citealp[Chapter 26]{shai_shai_book}}]
  \label[lemma]{fact:rad:prop}
  Let a collection of vectors $V\subseteq \R^n$ and a sample $\cS:=(z_i)_{i=1}^n$ be given.
  \begin{enumroman}
    \item
      For any scalar $c\in R$ and any $v_0 \in \R^n$,  %letting $cV := \{cv : v\in V}\subseteq \R^n$,
      $\fR(cV + v_0) \leq |c|\fR(V)$.
    \item
      For sets $(V_j)_{j=1}^\infty$ with $V_j\subseteq \R^n$
      %and $\sum_{j\geq 1} \sup_{v\in V_j} \|v\|_1 < \infty$
      and $0 \in V_j$ for all $j\geq 1$, it follows that
      $\fR(\cup_{j\geq 1} V_j) \leq \sum_{j\geq 1} \fR(V_j)$.
    \item
      For $z_i\in\R^d$
      and a set of linear predictors $\cW := \{ z\mapsto w\cdot z : w\in\R^d, \|w\|_1\leq B\}$,
      it follows that $\fR(\cW) = \fR(\cW\circ \cS)\leq B\sup_{z\in\cS}\|z\|_\infty  \sqrt{2\ln(2d)/n}$.
    \item
      For any $L$-Lipschitz function $\ell:\R\to\R$, it follows that $\fR(\ell \circ V) \leq L\fR(V)$.
  \end{enumroman}
\end{lemma}
Note that the aforementioned alternate form of $\fR$ using an absolute value breaks (i),
whereas it strengthens (ii) by allowing the condition $0\in V_j$ to be dropped.
\begin{proof}
  Proofs of parts (i), (iii), and (iv) can be found in \citep[Lemma 26.6, Lemma 26.11,  Lemma 26.9]{shai_shai_book};
  consequently, it only remains to handle (ii).
  For convenience, define $V_\infty := \cup_{j\geq 1} V_j$.
  Given any fixed $\sigma\in \{-1,+1\}^n$,
  the assumption $0\in V_j$ implies
  \[
    \sup_{v\in V_\infty} v\cdot\sigma \geq \sup_{v\in V_j} v\cdot\sigma \geq 0.
  \]
  Consequently, by Tonelli's theorem,
  \begin{align*}
    \fR(V_\infty)
    &= \bbE\left( \sup_{v\in V_\infty} \frac 1 n v\cdot\sigma \right)
    = \bbE\left( \sup_{j\geq 1}\sup_{v\in V_j} \frac 1 n v\cdot\sigma \right)
    \leq \bbE\left( \sum_{j\geq 1}\sup_{v\in V_j} \frac 1 n v\cdot\sigma \right)
    \\
    &= \sum_{j\geq 1}\bbE\left( \sup_{v\in V_j} \frac 1 n v\cdot\sigma \right)
    = \sum_{j\geq 1}\fR(V_j).
  \end{align*}
\end{proof}

\section{Experiments}
\label{sec:experiments}

%XXX screw the following, for now.
% \begin{remark}
%     \label{rem:bad_experiments}
%     The theoretically motivated regularization level is $c/\sqrt{n}$, where $c$ is a problem-dependent
%     constant.  For instance, an elementary application of Rademacher complexity gives $c \leq \|w\|_2\max_i\|x_i\|_2$
%     \citep{rad_paper}, but this is only an upper bound; an example more refined choice is to use
%     sparsity.  Consequently, it is certainly not claimed that \Cref{fig:reg_plots} condemns regularization
%     by condemning the choice $1/\sqrt{n}$, which omits crucial constants; instead, \Cref{fig:reg_plots} aims to show firstly that
%     mild or no regularization is fine, and secondly that it often helps.
% \end{remark}

%   \begin{itemize}
%       \item
%           Number datasets where best performance achieved with regularization $\cO(1/n)$:
%           \textbf{14 / 23}.
%       \item
%           Number datasets where best performance achieved with regularization $\cO(1/n^2)$:
%           \textbf{11 / 23}.
%       \item
%           Number datasets where unregularized $\|w\| \max_t\|x_t\| \sqrt{1/n} \geq $ test error:
%           \textbf{18 / 23}.
%       \item
%           Number datasets where unregularized $\|w\| \max_t\|x_t\| \sqrt{1/n} \geq 1$:
%           \textbf{16 / 23}.
%       \item
%           Number datasets where unregularized $\|w\| \max_t\|x_t\| \sqrt{1/n} \geq 100$:
%           \textbf{9 / 23}.
%   \end{itemize}

\begin{table}
    \caption{Description of Datasets}
    \label{table:datasets}
    \begin{center}
\begin{tabular}{l||r|r|r}
    Dataset & $n$ (\#examples) & $s$ (average sparsity) & $d$ (dimension)
\\
\hline
\hline
%20news &  18845  &  93.8854  & 101631
20news &  18845  &  93.9  & 101631
\\
%   a9a &  48841  &  13.8676  & 123
%   \\
%   abalone &  4176  &  8  & 8
%   \\
%   abalone &  4177  &  7.99952  & 10
%   \\
%activity &  165632  &  18.5489  & 20
activity &  165632  &  18.5 & 20
\\
%adult &  48842  &  11.9967  & 105
adult &  48842  &  12.0  & 105
\\
%   bank &  45210  &  13.9519  & 44  &
%   \\
%bio &  145750  &  73.4184  & 74
bio &  145750  &  73.4  & 74
\\
%   cal-housing &  20639  &  8  & 8
%   \\
%census &  299284  &  32.0072  & 401
census &  299284  &  32.0  & 401
\\
%   comp-activ-harder &  8191  &  11.5848  & 12
%   \\
%covtype &  581011  &  11.8789  & 54
covtype &  581011  &  11.9  & 54
\\
%   cup98-target &  95411  &  310.982  & 10825
%   \\
%eeg &  14980  &  13.9901  & 14
eeg &  14980  &  14.0  & 14
\\
%ijcnn1&  24995  &  13  & 22
ijcnn1&  24995  &  13.0  & 22
\\
%kdda&  8407751  &  36.349  & 19306083
kdda&  8407751  &  36.3  & 19306083
\\
%kddcup2009 &  50000  &  58.4353  & 71652
kddcup2009 &  50000  &  58.4  & 71652
\\
%letter &  20000  &  15.5807  & 16
letter &  20000  &  15.6  & 16
\\
%magic04 &  19020  &  9.98728  & 10
magic04 &  19020  &  10.0  & 10
\\
%maptaskcoref &  158546  &  40.4558  & 5944
maptaskcoref &  158546  &  40.5  & 5944
\\
%mushroom &  8124  &  22  & 117
mushroom &  8124  &  22.0  & 117
\\
%nomao &  34465  &  82.3306  & 174
nomao &  34465  &  82.3  & 174
\\
%poker&  946799  &  10  & 10
poker&  946799  &  10.0  & 10
\\
%rcv1&  781265  &  75.7171  & 43001
rcv1&  781265  &  75.7  & 43001
\\
%shuttle&  43500  &  7.04984  & 9
shuttle&  43500  &  7.0  & 9
\\
%skin &  245057  &  2.948  & 3
skin &  245057  &  2.9  & 3
\\
%   slice&  53500  &  134.575  & 384
%   \\
%   titanic &  2201  &  3  & 8
%   \\
%vehv2binary &  299254  &  48.5652  & 105
vehv2binary &  299254  &  48.6  & 105
\\
%w8a &  49749  &  11.6502  & 300
w8a &  49749  &  11.7  & 300
\\
%   year&  463715  &  90  & 90
\end{tabular}
\end{center}
\end{table}

\begin{figure}
  \begin{center}
    \includegraphics[width=0.8\textwidth]{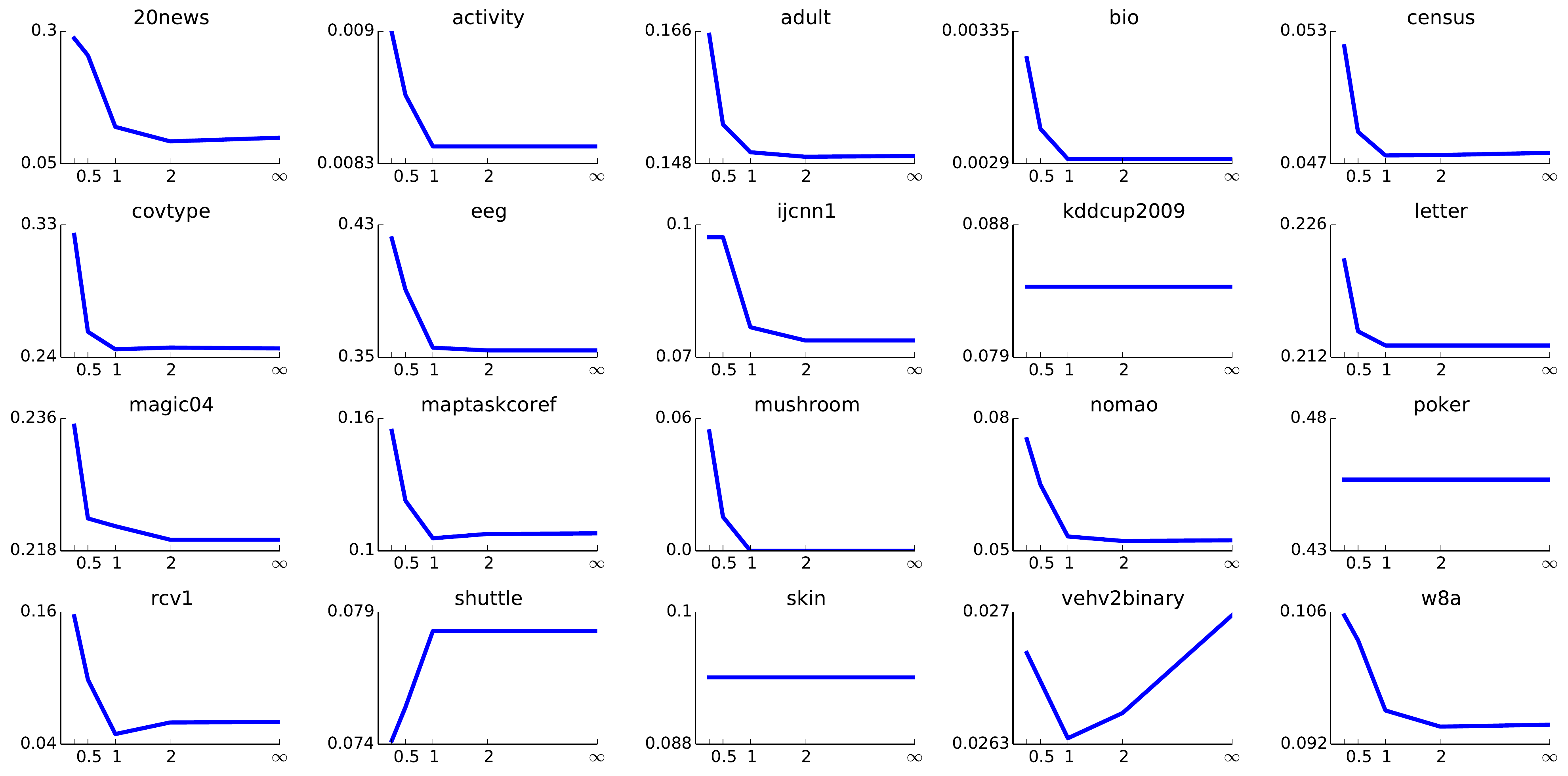}
  \end{center}
    \caption{Proportion of classification errors on various testing sets of linear classifiers
      trained by applying L-BFGS to regularized logistic regression (ERM with logistic loss);
        test error
        is on the vertical axis, and exponent $p$ of regularization coefficient
        $1/n^p$ is along the horizontal axis.
      For more detail, please see \Cref{sec:experiments}.}
    \label{fig:reg_plots}
\end{figure}

\begin{figure}
  \begin{center}
    \includegraphics[width=0.8\textwidth]{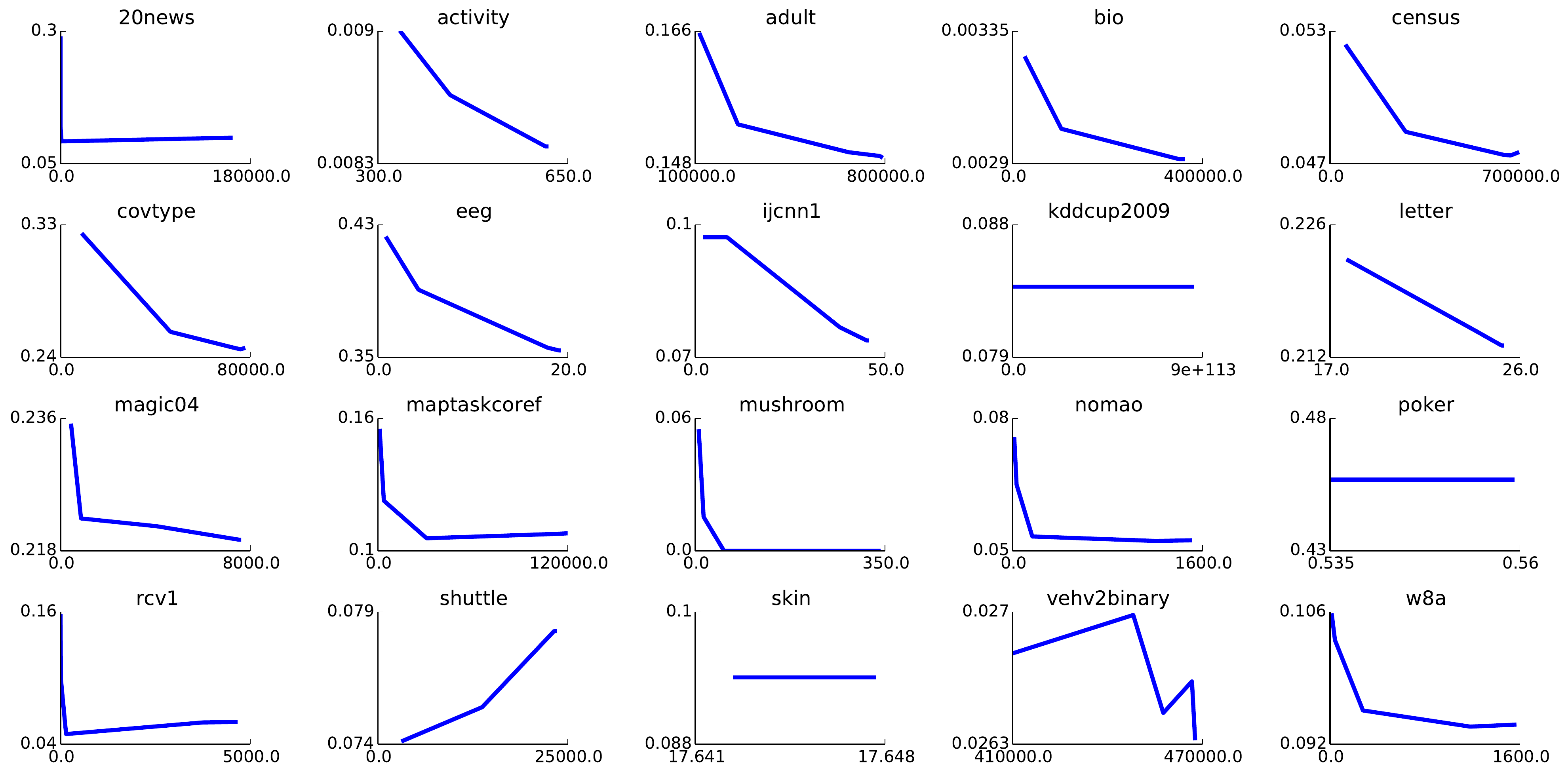}
  \end{center}
  \caption{Companion plot to \Cref{fig:reg_plots};
    vertical axis is once again the proportion of classification errors,
    but the horizontal axis is now the quantity $\|w\|_2 \max_i\|x_i\|_2$,
    meaning the norm of the vector output by L-BFGS, scaled by the data norm.
    This quantity is relevant since it appears in the standard
    Rademacher bounds for linear functions
    (see \Cref{sec:rademacher} and \citealp[Chapter 26]{shai_shai_book}).
    For more detail, please see \Cref{sec:experiments}.}
  \label{fig:normish_plots}
\end{figure}

In this appendix we demonstrate that the best performance on a wide variety of data sets can
be obtained with little or no regularization.
While there is some discussion of some methods' ability to seemingly avoid overfitting
\citep{boosting_margin,friedman_gradient_boosting},
this observation is primarily folklore, which served as a motivation for
our experiments,
depicted in \Cref{fig:reg_plots}. They were conducted as follows:
\begin{enumerate}
  \item
    We collected twenty datasets from a variety of sources (UCI, KDD Cup,
    libsvm data repository, and a few others), as described in \Cref{table:datasets}.
  \item
    Each dataset was split into 5 different (training, testing) pairs of size (80\%, 20\%).
  \item
    $\ell$ was chosen to be the logistic loss $\ln(1+\exp(\cdot))$
    and $\cH$ consisted of the coordinates, yielding the setting of
    logistic regression.
  \item
    We minimized the regularized empirical risk, i.e., $\hat\cR_n(w)+\lambda\|w\|_2^2/2$, where
    $\lambda$ was given the form $1/n^p$, where $p$ ranged over $\{ 1/2, 1, 2, \infty\}$,
    with $1/n^\infty = 0$.
  \item
    L-BFGS was applied to this regularized variant of $\hcR_n$ for each training/test split and each setting of
    the regularization parameter.  Each point in \Cref{fig:reg_plots} is the median across the
    five splits of the data.  Standard L-BFGS code was used (via \texttt{scikit-learn}),
    with very relaxed termination criteria in order to avoid early stopping
    ($\texttt{pgtol} = 10^{-9}$, $\texttt{factr} = 100$).
    In order to provide evidence that early stopping was avoided,
    please see \Cref{fig:normish_plots}, which roughly captures the norms of the selected predictors.
\end{enumerate}

    Note that even as the norm of $w$ increases, the classification error converges, and in most cases it
    is in fact minimized at large norms.
It is essential that the plots depict classification error,
whereby \Cref{fact:findim:gen:simplified} and \Cref{fact:zo}
explain why they behave stably.
By contrast, if the goal were to recover specific iterates or control the loss
itself, there are lower bounds indicating a dependence on norms is necessary
\citep{colt_logistic_hazan}.

\section{Properties of classification losses}

\subsection{Basics}

\begin{lemma}
  \label[lemma]{fact:loss_prop}
  \begin{enumroman}
    \item
      If $\ell \in \Lclass$, then $\lim_{z\to-\infty}\ell(z) = 0$,
      $\ell^*(0) = 0$,
      $\ell^*(s) = \infty$ whenever $s < 0$,
      and $\bars\in \partial\ell(0)$ satisfies $\ell^*(\bars) = \min_{s\in\R} \ell^*(s) = -\ell(0) < 0$.
   %\item
   %  If $\ell\in\Lclass$ with tightest Lipschitz constant $L := \sup_{r\neq s}|\ell(r)-\ell(s)|/|r-s|$,
   %  then $L < \infty$ implies $\dom(\ell^*)=[0,L]$ and $L = \infty$ implies $\dom(\ell^*) = [0,\infty)$.
    \item
      If $\ell \in \Ldiff$, then $\ell'>0$ and $\lim_{z\to-\infty}\ell'(z) = 0$.
    \item
      If $\ell \in \Ldiff$, then $\liminf_{z\to-\infty}\ell''(z) = 0$.
    \item
      If $\ell \in \Ldiff$ and $\ell$ is Lipschitz, then $\liminf_{z\to\infty}\ell''(z) = 0$.
      %here's a fancier proof that might show $\lim$ and not just $\liminf$, but I haven't looked in a while.
      %since I guess it gives the result for all rationals or something like that, continuity of $\ell''$
      %shold give it?   again I'm going by memory of it and didn't read even though it's so short ....
    %%\MJTDEBATE{Proof here isn't quite enough, it just shows there's a subsequence ($q_z$) going
    %%  to zero, does not show going to zero everywhere.  For that with instead look at
    %%  \[
    %%    \lim_{z\to\infty} \ell'(z)
    %%    = \lim_{z\to\infty} (\ell'(z + \eps) - \ell''(q_z) \eps)
    %%  \]
    %%  for all $\eps > 0$.  actually I think this works out, since for large enough
    %%  $z$ the gap between $\ell'(z)$ and $\ell'(z+\eps)$ gets arbitrarily small (say,
    %%  less than some arbitrary $\sigma > 0$, so have $\ell''(q_z) \leq \sigma / \eps$,
    %%  so whatever, pick $\sigma \in (0, \eps^2)$.  will need to similarly fix
    %%$\lim_{z\to-\infty}\ell''(z)$.}
    \item
      If $\ell \in \Ldiff$,
      then $\lim_{s\downarrow 0}(\ell^*)'(s) = -\infty$.
      Additionally, if $\lim_{r\to\infty} \ell'(r)\eqqcolon L<\infty$,
      then $\lim_{s\uparrow L} (\ell^*)'(s) = \infty$.
  \end{enumroman}
\end{lemma}
\begin{proof}
  \begin{enumromansep}
    \item
      The first property follows from $\inf_{z\in\R}\ell(z) = 0$, which for convex $\ell$ with
      $\lim_{z\to-\infty} \ell(z) > 0$ implies $\ell$ is not nondecreasing.

      The second property follows from
      $\ell^*(0) = \sup_{r\in \R}(0\cdot r - \ell(r)) = - \inf_{s\in\R} \ell(s) = 0$.

      Next, since $\lim_{z\to-\infty}\ell(z) = 0$,
      $s<0$ implies
      \[
        \ell^*(s) = \sup_{r\in \R}(rs - \ell(r)) \geq \lim_{r\to-\infty}(rs - \ell(r)) = \infty.
      \]

      Lastly, because $\ell$ is closed,
      $\bars\in \partial \ell(0)$ implies $0 \in \ell^*(\bars)$, which is the first order optimality condition,
      giving $\ell^*(\bars) = \min_{z\in\R} \ell^*(z)$.
      Moreover, by Fenchel's inequality,
      $\ell^*(\bars) + \ell(0) = 0\cdot \bars = 0,$ meaning $\ell^*(\bars) = -\ell(0)$,
      and lastly $\ell(0) > 0$, because $\ell\in\Lclass$.

   %\item
   %  \MJT{if time, find a shorter proof.}
   %  The proof when $L < \infty$ can be found in \citep[Lemma 28, item 3]{mjt_logistic}.
   %  When $L = \infty$, first note that $\ell^*$ is finite along $[0,\bars]$ where $\bars\in\partial\ell(0)$,
   %  since $\ell^*$ is convex
   %  and $\{0,\bars\}\subseteq \dom(\ell^*)$ was established in the preceding part.  Thus, consider $L_0 > \bars$,
   %  whereby $L = \infty$ and the nondecreasing property of $\partial\ell$
   %  implies there exist $r\in\R$ and $s\in\partial\ell(r)$ with $s \geq L_0$.
   %  Since $\ell^*$ is nondecreasing after $\bars$, it follows from Fenchel's inequality that
   %  \[
   %    \ell^*(L_0) \leq \ell^*(s) = rs - \ell(r) < \infty.
   %  \]
   %  Since $L_0> \bars$ was arbitrary, the result follows.

    \item
      If there existed $z'$ with $\ell'(z') = 0$, then $\ell''>0$ implies $\ell'(z'-1) < 0$, contradicting the fact that $\ell$ is nondecreasing.

      %for this I need $\ell'$ to be continuous, so that's why $\Ldiff$ and not $\Lclass$...
      Next, Mean Value Theorem  grants for every $z<0$ a $q_z \in [2z, z]$ such that
      \[
        0
        = \lim_{z\to-\infty} \ell(z)
        = \lim_{z\to-\infty} (\ell(2z) + \ell'(q_z)(z - 2z))
        = \lim_{z\to-\infty} (-z)\ell'(q_z),
      \]
      which necessitates $\lim_{z\to-\infty} \ell'(z) = 0$ since
      $\ell'$ is nondecreasing and $\ell'\geq 0$.

    \item
      Similarly to the above derivation for first derivatives,
      Mean Value Theorem grants for every $z<0$ a $q_z \in [2z, z]$ such that
      \[
        0
        = \lim_{z\to-\infty} \ell'(z)
        = \lim_{z\to-\infty} (\ell'(2z) + \ell''(q_z)(z - 2z))
        = \lim_{z\to-\infty} (-z)\ell''(q_z),
      \]
      which necessitates $\liminf_{z\to-\infty}\ell''(z) = 0$ by positivity of $\ell''$.
      %XXX might be able to use continuity of \ell'' and countably infinitely many such limit sequences
      %XXX to establish a limit, not just a liminf.  do I need this?

    \item
      Since $\lim_{z\to-\infty} \ell'(z) = 0$ (as above) and $\ell'' > 0$ and $\ell$ is Lipschitz,
      then there exists $L\geq 0$ with $\lim_{z\to\infty} \ell'(z) = L < \infty$.
      Similarly to the proof of the preceding property, Taylor's theorem grants for every
      $z>0$ a $q_z \in [z, 2z]$ with
      \[
        L
        = \lim_{z\to\infty} \ell'(z)
        = \lim_{z\to\infty} (\ell'(2z) + \ell''(q_z)(z - 2z))
        = L + \lim_{z\to\infty} (-z)\ell''(q_z),
      \]
      which again necessitates $\liminf_{z\to\infty}\ell''(z) = 0$ by positivity of $\ell'$.

    \item
      By strict convexity of $\ell$, $\ell^*$ is differentiable over the interior of its domain \citep[Theorem E.4.1.1]{HULL}. By part (i), $\dom\ell^*$ includes $0$ and $\bar{s}>0$, so we can write $\lim_{s\downarrow 0} (\ell^*)'(s) = \lim_{s\downarrow 0} (\ell')^{-1}(s) = -\infty$. Where the last step follows because $\ell'$ is strictly increasing and $\lim_{r\to-\infty}\ell'(r)=0$.

      Given $\lim_{r\to\infty}\ell'(r)=L< \infty$, we obtain as before
      $\lim_{s\uparrow L} (\ell^*)'(s) = \lim_{s\uparrow 0} (\ell')^{-1}(s) = \infty$.
      %
      %MEH this is just a statement about relbd(dom(\ell^*)), but I don't feel like talking about that...
  \end{enumromansep}
\end{proof}

\begin{proposition}
  \label[proposition]{fact:eta:basic}
  Let $\ell\in\Ldiff$ be given.
  \begin{enumroman}
    \item
      The link $\phi$ is a monotone increasing bijection between $\R$ and $(0,1)$,
      and moreover continuously differentiable.

    %Don't want to remove but too lazy to prove (deadline in hours.  ah wlel.
  %%\item
  %%  If $\ell\in\Ldiff$ and $\ell$ is Lipschitz, then $\eta$ is also Lipschitz (meaning $L_\phi < \infty$).

    \item
      If $\phi$ is convex over $(-\infty,0]$ and concave over $[0,\infty)$,
        then $L_\phi = \ell''(0) / (2\ell'(0))$.
        (This holds in particular for the logistic and exponential losses,
        which therefore have $L_\phi$, respectively, equal to $1/4$ and $1/2$.)
  \end{enumroman}
\end{proposition}
\begin{proof}
  \begin{enumromansep}
    \item
      Note that
      \begin{equation}
        \phi'(z)
        %= \frac {\ell''(z)(\ell'(z) + \ell'(-z)) - \ell'(z)(\ell''(z) - \ell''(-z))}{(\ell'(z) + \ell'(-z))^2}
        = \frac {\ell''(z)\ell'(-z) + \ell'(z)\ell''(-z)}{(\ell'(z) + \ell'(-z))^2};
        \label{eq:eta:grad}
      \end{equation}
      which is positive and continuous, because $\ell\in\Ldiff$, so $\phi$ is increasing.
      Note that $\lim_{r\to\infty}\phi(r)=1$ and $\lim_{r\to-\infty}\phi(r)=0$, because $\ell'$
      is increasing and $\lim_{r\to-\infty}\ell'(r)=0$. The bijection statement follows by continuity.
      %XXX gap was introduced since no longer have the \lim but only \liminf
 %% \item
 %%   By the assumed properties of $\ell$ and \Cref{fact:loss_prop},
 %%   $\lim_{z\to-\infty}\ell''(z) = \lim_{z\to\infty}\ell''(z) = 0$,
 %%   thus there exists $B$ such that $|z| > B$ implies $|\ell''(z)| \leq 1$.
 %%   On the other hand, since $\ell''$ is continuous, it must attain a maximum along
 %%   $[-B, +B]$.  Combining these two facts, $\ell''$ is bounded.   Since $\ell'$ is also bounded
 %%   and $\ell'(z) + \ell'(-z) \geq \ell'(0) > 0$, it follows that $\eta'$ as in
 %%   \Eq{eta:grad} is also bonded, meaning $\eta$ is Lipschitz.

    \item
      By assumption, $\phi'$ is largest at 0.  By the form of $\phi'$ given in \Eq{eta:grad}
      above, it follows that $\phi'(0) = \ell''(0) / (2\ell'(0))$.
      The convexity/concavity property may be manually checked for the exponential and logistic
      losses, since they respectively give $\phi''$ to be
      \[
        - \frac {4e^{2x}(e^{2x} - 1)}{(1+e^{2x})^3}
        \qquad\textup{and}\qquad
        - \frac {e^x(e^x - 1)}{(1+e^x)^3}.
      \]
 %\begin{align*}
 %  \eta_\ell''(z)
 %  &=
 %  (\ell'(z) + \ell'(-z))^{-3}\big(
 %  \ell'(-z)^2 \ell'''(z) - \ell'(z)^2 \ell'''(-z)
 %  + (\ell'(z) \ell'(-z))(\ell'''(z) - \ell'''(-z))
 %  \\
 %  &+ 2 \ell'(z) \ell''(-z)^2 -2 \ell'(-z) \ell''(z)^2
 %  + 2 \ell''(z) \ell''(-z)(\ell'(z) - \ell'(-z))
 %  \big)
 %\end{align*}
 %\MJT{blah, is there an easy (symmetry?) argument that 0 is the only inflection point?  blah.}
  \end{enumromansep}
\end{proof}

\subsection{Elements of $\Lbnd$}
\label{subsec:bbL_3}

%XXX I think this is solved now?
% \MJT{I skipped this section and will come back to it when doing \Cref{fact:D_controls}.  see the
% notes there.  Looks like I have to spruce up on the definition of $\Lbnd$.}

\begin{lemma}
  \label[lemma]{fact:exp_bbL2_helper}
  Let finite non-null measure $\mu$ over $\cZ$ and function $f:\cZ\to \R$ be given with
  $f\geq 0$ $\mu$-a.e.\@ and $\int \exp(f)d\mu<\infty$. Set
  $b := \int \exp(f) d\mu / \mu(\cZ)$.
  Then $\int \exp(f / b)d\mu \leq\mu(\cZ) e^{1/e}$.
\end{lemma}
\begin{proof}
  Note that $b \geq \int \exp(0)d\mu / \mu(\cZ) = 1$.
  Consequently, the function $r \mapsto r^{1/b}$ is concave,
  and thus Jensen's inequality (applied to the normalized measure $\mu/\mu(\cZ)$) grants
  \begin{align*}
    \int \exp(f / b)d\mu
    &= \mu(\cZ) \int \exp(f)^{1/b} d\mu/\mu(\cZ)
    \leq \mu(\cZ) \left(\int \exp(f) d\mu/\mu(\cZ) \right)^{1/b}
    = \mu(\cZ) b^{1/b}.
  \end{align*}
  Next it will be shown that the function $g(z) := z^{1/z}$ is maximized over $(0,\infty)$ at $\bar z:=e$,
  which gives the result.  To this end, note
  \[
    g'(z) = z^{1/z}z^{-2}(1-\ln(z)),
  \]
  which is positive for $z \in (0, \bar z)$, zero at $\bar z$, and negative for $z>\bar z$.
\end{proof}

\begin{lemma}
  \label[lemma]{fact:exp_bbL2}
  Let finite measure $\mu$ over $\cZ$ and function $f:\cZ\to \R$ be given with
  $f\geq 0$ $\mu$-a.e.\@ and $\mu(\cZ) \leq 2$.
  If $\ell\in\Lclass$ denotes the exponential loss $\ell = \exp$,
  then
  $\|f\|_{\beta} \leq \int \exp(f)d\mu/\mu(\cZ)$.
\end{lemma}
\begin{proof}
  If $\int \exp(f)d\mu = \infty$, there is nothing to show, thus suppose $\int\exp(f)d\mu<\infty$.
  Since $\ell''(z) = \exp(z)\ge 1$ if $z\ge 0$ and $\le 1$ if $z\le 0$, the Taylor
  expansion yields, for $z\ge 0$,
  \[
    \exp(z) \geq 1 + z + z^2/2
    \qquad\textup{and}\qquad
    \exp(-z) \leq 1 - z + z^2/2.
  \]
  Consequently, for any $z\geq 0$,
  \[
    \exp(z) - (1 + z)
    \geq z^2/2
    \geq \exp(-z) - (1 - z),
  \]
  which means $\beta(z) = \exp(z) - (1 + z)$ when $z\geq 0$.
  Combining this with \Cref{fact:exp_bbL2_helper},
  setting $b := \int \exp(f) d\mu / \mu(\cZ)$ for convenience,
  \begin{align*}
    \int \beta(f/b)d\mu
    &= \int \left(\exp(f/b) - 1 - f/b\right)d\mu
    \leq \mu(\cZ)(e^{1/e} - 1)
    \leq 1.
  \end{align*}
  By the definition of $\|\cdot\|_{\beta}$,
  it follows that $\|f\|_{\beta} \leq b$.
\end{proof}

\begin{lemma}
  \label[lemma]{fact:lipschitz_bbL2}
  Let finite measure $\mu$ over $\cZ$ and function $f:\cZ\to \R$ be given with
  $f\geq 0$ $\mu$-a.e..
  If $\ell\in\Lclass$ is $L$-Lipschitz,
  then
  $\|f\|_{\beta} \leq L\int \ell(f)d\mu / \ell'(0)$.
\end{lemma}
\begin{proof}
  To start, for any $r \geq 0$, since $\ell$ is nondecreasing,
  \begin{align*}
    \beta(r)
    &= \max\{\ell(r) - (\ell(0) + r\ell'(0)), \ell(-r) - (\ell(0) - r\ell'(0))\}
    \\
    &\leq \max\{\ell(0) + rL - (\ell(0) + r\ell'(0)), \ell(0) - (\ell(0) - r\ell'(0))\}
    \\
    &\leq r \max\{L-\ell'(0), \ell'(0)\}.
  \end{align*}
  Setting $b := L\int \ell(f)d\mu/\ell'(0)$,
  \begin{align*}
    \int \beta(f/b)d\mu
    &\leq \max\{L-\ell'(0), \ell'(0)\}
    \frac {\ell'(0)\int f d\mu}{L\int \ell(f)d\mu}
    \\
    &\leq \max\{L-\ell'(0), \ell'(0)\}
    \frac {\ell'(0)\int f d\mu}{L(\ell(0) + \ell'(0) \int f d\mu)}
    \\
    &\leq \max\{L-\ell'(0), \ell'(0)\}
    \frac {\ell'(0)\int f d\mu}{L\ell'(0) \int f d\mu}
    \\
    &\leq 1,
  \end{align*}
  where the last step follows because $\ell'(0)\le L$.
  By the definition of $\|\cdot\|_{\beta}$,
  it follows that $\|f\|_{\beta} \leq b$.
\end{proof}

\begin{proposition}
  \label[proposition]{fact:bbL3_containment}
  Let finite measure $\mu$ over $\cZ$ with $\mu(\cZ)\leq 2$ and hypotheses $\cH$ be given.
  %XXX originally I wanted $L_\phi < \infty for general elements of $L_2 .. oh well.
 %\MJT{
 %XXX last minute change because I broke the related proof, owing back to a breakage
 %swapping a $\lim$ with a $\liminf$.  I'd fix it (and maybe the fix is there) but deadline is in hours...
 %>>> then $\ell \in \Lbnd$ with $c_\ell \leq L / \ell'(0)$. <<<
 %}
  Then $\ell\in\Lclass$ having a finite Lipschitz constant $L$
  entails $c_{\ell,\mu} \leq L / \ell'(0)$,
  and $\ell = \exp$ entails $c_{\ell,\mu} \leq 1/\mu(\cZ)$.
  Secondly, $\ell = \ln(1+\exp(\cdot))$ entails $L_\phi = 1/4$,
  and $\ell = \exp$ entails $L_\phi = 1/2$.
  Thirdly, $\ell = \ln(1+\exp(\cdot))$ entails $c_\ell = 2$,
  and $\ell = \exp$ entails $c_\ell = 1$.
  In particular, in either case, the loss is within $\Lbnd$.
\end{proposition}

%\begin{proofof}{\Cref{fact:bbL3_containment}}
\begin{proof}
  Everything but the bounds on $c_\ell$ have already been provided by
  \Cref{fact:lipschitz_bbL2},
  \Cref{fact:exp_bbL2},
  and
  \Cref{fact:eta:basic}.
  For $c_\ell$, the bound is immediate for $\ell = \exp$ (since then $\ell = \ell'$),
  thus consider $\ell = \ln(1+\exp(\cdot))$.
  Noting the second-order Taylor expansion of $\ln$ along $[1,1+q]$ with $q\leq 1$ is
  $\ln(1+q) \geq \ln(1) + q \ln'(1) + \inf_{s\in [1,2]} q^2 \ln''(s) / 2$,
  then $r\leq 0$ implies
  \begin{align*}
    \ell(r)
   = \ln(1+e^r)
    &\geq e^r
   - \sup_{s\in [1,2]}\frac {e^{2r}}{2s^2}
   %\\
   %&
    = e^r \left(1
    - \frac {e^r}{2}\right)
    \geq \frac {e^r}{2}
    \geq \frac {e^r}{2(1+e^r)}
    = \frac{\ell'(r)}{2}.
  \end{align*}
\end{proof}

\section{Proof of \Cref{fact:zo}}

The proof of \Cref{fact:zo} is split into two lemmas; first, an upper bound establishing the general
inequality, and second, an example showing the right-hand side of the inequality can be positive and tight.
The proof of this upper bound is a straightforward
consequence of standard manipulations for classification error \citep[Theorem 2.1]{DGL}.

\begin{lemma}
  \label[lemma]{fact:zo:ub}
  Let probability measure $\mu$,
  hypotheses $\cH$,
  and loss $\ell\in\Ldiff$ be given.
  For any $w\in L_1(\cH)$,
  \begin{align*}
    \left|
    \cRz(Hw) - \cRz(\bar \eta(\cdot,1) - 1/2)
    \right|
    &\leq
    \int_{\bar \eta = 1/2} (\eta_\mu(x,1) - \eta_\mu(x,-1))(1-\one[\eta_w(x,1) \geq 1/2]) d\mu_\cX(x)
    \\
    &+
    \underbrace{%
      2\int_{\bar \eta \neq 1/2} \min\left\{1,
        \left(\frac {|\eta_\mu(x,1) - 1/2|}{|\bar\eta(x,1) - 1/2|}\right) |\bar\eta(x,1) - \eta_w(x,1)|
      \right\}d\mu_\cX(x)
    }_{\star},
  \end{align*}
  where $\star \to 0$ as $\int|\bar \eta - \eta_w|d\mu \to 0$.
\end{lemma}
%\begin{proofof}{\Cref{fact:zo:ub}}
\begin{proof}
  Following the derivation of \citet[Theorem 2.1]{DGL},
  for any $g : \cX \to \{-1,+1\}$ and any $x\in\cX$,
  \begin{align*}
    \Pr[ g(X) \neq Y | X=x ]
    &= 1 - \Pr[g(X) = Y | X=x]
    \\
    &= 1 - \left( \one[g(x) = 1] \eta_\mu(x,1) + \one[g(x) = -1] \eta_\mu(x,-1)\right).
  \end{align*}
  Consequently, for any $g_1:\cX\to \{-1,+1\}$, $g_2:\cX\to \{-1,+1\}$,
  and any $x\in \cX$,
  \begin{align}
    &\Pr[ g_1(X) \neq Y | X=x ] - \Pr[ g_2(X) \neq Y | X=x ]
    \notag\\
    &= \eta_\mu(x,1)(\one[g_2(x) = 1] - \one[g_1(x) = 1]) + \eta_\mu(x,-1)(\one[g_2(x) = -1] - \one[g_1(x) = -1])
    \notag\\
    &= (\eta_\mu(x,1) - \eta_\mu(x,-1))(\one[g_2(x) = 1] - \one[g_1(x) = 1]).
    \label{eq:zo:1}
  \end{align}
  With this in mind, define $g_1(x) := \one[\eta_w(x,1) \geq 1/2]$ and $g_2(x) := \one[\bar \eta(x,1) \geq 1/2]$,
  whereby the signs of $(Hw)(x)$ and $\eta_w(x,1) - 1/2$ agree,
  and
  \begin{align*}
    \MoveEqLeft \cRz(Hw) - \cRz(\bar \eta - 1/2)
    \\
    &=
    \Pr[ g_1(X) \neq Y] - \Pr[g_2(X) \neq Y]
    \\
    &=
    \underbrace{%
      \int_{\bar \eta = 1/2}
      \left(
        \Pr[ g_1(X) \neq Y | X=x] - \Pr[g_2(X) \neq Y|X=x]
      \right)
      d\mu_\cX(x)
    }_{\triangle}
    \\
    &\qquad+
    \underbrace{%
      \int_{\bar \eta \neq 1/2}
      \left(
        \Pr[ g_1(X) \neq Y|X=x] - \Pr[g_2(X) \neq Y|X=x]
      \right)
      d\mu_\cX(x).
    }_{\square}
  \end{align*}
  To bound these terms, applying \Eq{zo:1} to the first term and using $g_2(x) = 1$ along $\bar\eta=1/2$ yields
  \begin{align*}
    \triangle
    &= \int_{\bar\eta = 1/2} (\eta_\mu(x,1) - \eta_\mu(x,-1))(1 - \one[g_1(x) = 1])d\mu_\cX(x).
  \end{align*}
  For the second term, note
  \[
    \one[g_1(x) \neq g_2(x)] \leq \min\left\{1, \frac {|\eta_w(x,1) - \bar\eta(x,1)|}{|\bar \eta(x,1) - 1/2|}\right\}.
  \]
  Combining this with \Eq{zo:1},
  \begin{align*}
    |\square|
    &\leq
    2\int_{\bar \eta \neq 1/2}
    \min\left\{
      |\eta_\mu(x,1) - 1/2|
      ,
      \left(\frac {|\eta_\mu(x,1) - 1/2|}{|\bar \eta(x,1) - 1/2|}\right)
      |\eta_w(x,1) - \bar\eta(x,1)|
    \right\}d\mu_\cX(x)
    \\
    &\leq \star,
  \end{align*}
  with $\star$ given in the statement in the statement.
  To see that $\star \to 0$ as $\|\eta_w-\bar\eta\|_1\to 0$,
  first note, for any $\sigma \in (0,1/2]$, that
  \begin{align*}
    \star
    &\leq
    2\int_{|\bar \eta - 1/2|\in(0 ,\sigma)}
    1 d\mu_\cX(x)
    +
    2\int_{|\bar \eta - 1/2| > \sigma}
      \left(\frac {|\eta_\mu(x,1) - 1/2|}{|\bar \eta(x,1) - 1/2|}\right)
      |\eta_w(x,1) - \bar\eta(x,1)|
    d\mu_\cX(x)
    \\
    &\leq
    %2\mu(\{x\in\cX : |\bar\eta(x,1) - 1/2| \in (0,\sigma)\})
    2\int_{|\bar \eta - 1/2|\in(0 ,\sigma)}
    1 d\mu_\cX(x)
    +
    \frac 1 \sigma \int
      |\eta_w(x,1) - \bar\eta(x,1)|
    d\mu_\cX(x).
  \end{align*}
  Since the first term goes to 0 as $\sigma\to 0$, it suffices to choose $\sigma:=\sqrt{\|\eta_w-\bar\eta\|_1}$
  and the result follows.
\end{proof}
%\end{proofof}

  In order to establish the tightness of the bound,
  consider any $\eps \in [0,1)$, let $\cX=[-1,1]$, and define the following probability measure $\mu$ over $\cX\times \{-1,+1\} = \cZ$:
  \[
    \mu(x,\pm 1)\;
    \begin{cases}
      a := -1,&b:= 1-\eps;\\
      \mu_\cX(a) = \frac {1-\eps}{2-\eps},
      &
      \mu_\cX(b) = \frac {1}{2-\eps};
      \\
      \eta_\mu(a,+1) = 1,
      &
      \eta_\mu(b,+1) = 1.
    \end{cases}
  \]

\begin{lemma}
  \label[lemma]{fact:zo:lb:new}
  Let scalar $\eps \in [0,1)$,
  probability measure $\mu$ as above,
  hypotheses $\cH := \{ h\}$ where $h(x) = x$,
  and loss $\ell\in\Ldiff$ be given.
  Then the sequence $(w_i)_{i=1}^\infty$ with $w_i := (-1)^i/i$
  satisfies $\eta_{w_i}\to\bar\eta$ and
  \begin{align*}
    \cRz(Hw_i) - \cRz(\bar \eta(\cdot,1) - 1/2)
    &=
    \int_{\bar \eta = 1/2} (\eta_\mu(x,1) - \eta_\mu(x,-1))\one[\eta_w(x,1) < 1/2] d\mu_\cX(x)
    \\
    &=
    \begin{cases}
      \frac {1}{2-\eps}
      &
      \textup{when $i$ is odd},
      \\
      \frac {1-\eps}{2-\eps}
      &
      \textup{when $i$ is even}.
    \end{cases}
  \end{align*}
\end{lemma}
%\begin{proofof}{\Cref{fact:zo:lb}}
\begin{proof}
  Note that $\cR$ has primal optimum $\barw = 0$:
  evaluating the gradient of $\cR$ at $\barw$ gives
  \[
    -a\ell'(-a\barw)\mu_\cX(a) -b\ell'(-b\barw)\mu_\cX(b)
    = \ell'(0)\left(\frac{1-\eps}{2-\eps}\right) -\ell'(0) \left(\frac {1-\eps}{2-\eps}\right)
    = 0.
  \]
  By \Cref{fact:duality}, $\barq = \ell'(0)$ $\mu$-a.e., thus
  $\bar\eta = \phi(0) = 1/2$ $\mu$-a.e., and $\cRz(\bar\eta(\cdot,1)-1/2) = 0$.

  Turning now to $w_i$, since $\bar\eta = 1/2$ and $\eta_\mu = \one[\bar\eta \geq 1/2]$ everywhere,
  \begin{align*}
    \cRz(Hw_i) - \cRz(\bar\eta(\cdot,1)-1/2)
    &= \cRz(Hw_i)
    = \sum_{x\in \{a,b\}} \mu_\cX(x) \one[ \eta_{w_i}(x,1)< 1/2 ]
    \\
    &=
    \int_{\bar \eta = 1/2} (\eta_\mu(x,1) - \eta_\mu(x,-1))\one[\eta_{w_i}(x,1) < 1/2] d\mu_\cX(x).
  \end{align*}
  Moreover, when $i$ is odd, then $\cRz(Hw_i) = \mu_\cX(b)$, whereas $i$ being even implies $\cRz(Hw_i) = \mu_\cX(a)$.

  Lastly, the convergence statement follows since $w_i \to \barw$, thus $\eta_{w_i} = \phi(Hw_i) \to \phi(H\bar w) = \bar\eta$
  by continuity of $\phi$ (cf. \Cref{fact:eta:basic}).
\end{proof}

\begin{proofof}{\Cref{fact:zo}}
  The proof follows by instantiating the bound in \Cref{fact:zo:ub} for each $w_i$, and applying $\limsup_{i\to\infty}$ to
  the absolute value of both sides.  On the other hand, \Cref{fact:zo:lb:new} with any $\eps \in [0,1)$ provides the instance with $\star > 0$
  and both $\limsup$s being equal. Note that the existence of oscilation exhibited in \Cref{fact:zo:lb:new} does not depend
  on our particular definition of $\sign(0)$.
\end{proofof}

\section{Proofs from \Cref{sec:duality}}

To prove the main duality result (\Cref{fact:duality}), we rely on a pairing of Orlicz spaces
$M_\beta$ and $L_{\beta^*}$ implied by \Cref{prop:orlicz}.iv for a specific choice
of $\beta$ introduced in \Eq{beta}. We begin by showing how the norms $\norm{\cdot}_\beta$ and $\norm{\cdot}_{\beta^*}$
relate to the primal and dual objectives.

Recall that $\beta$ is a symmetrized version of a loss $\ell\in\Lclass$ with the first-order Taylor expansion at zero subtracted, and
it thus represents the curvature of $\ell$:
\[
    \beta(s) \coloneqq \max\Set{
      \ell( s) - \BigParens{\ell(0) + s \ell'(0)} ,\;
      \ell(-s) - \BigParens{\ell(0) + (-s) \ell'(0)}
    }
\enspace.
\]
Note that this $\beta$ satisfies the conditions on $\theta$ in \Cref{prop:orlicz} and it is finite on $\R$,
so we obtain the Banach space pairing between $M_\beta(\mu)$ with norm topology and $L_{\beta^*}(\mu)$ with weak${}^*$ topology.

\begin{lemma}
  \label[lemma]{fact:orlicz_prop}
  Given a finite measure $\mu$ over $\cZ$ and a loss function $\ell \in \Lclass$, the following
  hold:
  \begin{enumromansep}
    \item
      If $f \in M_\beta(\mu)$, then $\int\ell(f)d\mu < \infty$.
    \item
      $\beta^*(s) \le \ell(0) + \min\Set{\ell^*\bigParens{\ell'(0)-|s|},\,\ell^*\bigParens{\ell'(0)+|s|}}$.
    \item
      Let $\nu$ be any measure absolutely continuous with respect to $\mu$,
      and let $f$ denote its density with respect to $\mu$, meaning $f := d\nu/d\mu$.
%      \citep[Theorem 3.8]{folland}.
      Then $\int \ell^*(f)d\mu < \infty$ implies $f \in L_{\beta^*}(\mu)$.
  \end{enumromansep}
\end{lemma}

\begin{proof}
  \begin{enumromansep}
    \item
      Since $f\in M_\beta(\mu)$ means $\int\beta(f)d\mu<\infty$,
      the definition of $\beta$ and property $\ell\geq 0$ grant
      %XXX have to join the first integrals sicne technically don't know stuff is integarly before invoking the \beta upper bound..
      \begin{align*}
        \int \ell(f)d\mu
        &\leq
        \int \BigParens{\ell(f) + \ell(-f)}d\mu
        \\
        &= \int \BigParens{
               \ell(f) - \bigParens{\ell(0) + \ell'(0)f}
             + \ell(-f) - \bigParens{\ell(0) - \ell'(0)f}
            }
            d\mu
         + 2\int \ell(0)d\mu
        \\
        &\leq 2 \int \beta(f)d\mu + 2\ell(0)\mu(\cZ)
        \\
        &<\infty
    \enspace.
      \end{align*}

    \item
      For convenience, define
      \[
        \ell_+(r) := \ell(r) - \bigParens{\ell(0) + r\ell'(0)}
        \qquad\textup{and}\qquad
        \ell_-(r) := \ell(-r) - \bigParens{\ell(0) - r\ell'(0)}
    \enspace,
      \]
      and note (e.g., from definition of conjugate or by Theorem 12.3 of \citealp{ROC}) that
      \[
        \ell_+^*(s) := \ell(0) + \ell^*\bigParens{\ell'(0) + s}
        \qquad\textup{and}\qquad
        \ell_-^*(s) := \ell(0) + \ell^*\bigParens{\ell'(0) - s}
    \enspace.
      \]
      Since $\beta = \max\{\ell_+,\ell_-\}$,
      then, by definition of conjugate, $\beta^* \leq \min\{\ell_+^*, \ell_-^*\}$, yielding the result.

    \item
      Let $\bar{s}\coloneqq\ell'(0)$. By \Cref{fact:loss_prop}.i,
      $\ell^*$ is minimized at $\bar{s}>0$, so it must be non-increasing on $[0,\bar{s}]$ and non-decreasing on $[\bar{s},\infty)$. Also,
      by \Cref{fact:loss_prop}.i, $\ell^*(0)=0$, so $\ell^*\le 0$ on $[0,\bar{s}]$. Part~(ii) therefore implies
    \[
      \beta^*(s)\le
      \begin{cases}
      \ell(0) & \text{if $|s|\le \bar{s}$}
    \\
      \ell(0) + \ell^*(2|s|) & \text{if $|s|>\bar{s}$.}
      \end{cases}
    \]
      Let $f=d\nu/d\mu$, i.e., $f=|f|$ ($\mu$-a.e.) and assume that $\int\ell^*(f)d\mu<\infty$.
      Using the previous bound on $\beta^*$, write
      \begin{align*}
        \int \beta^*(f/2) d\mu
      &\le
        \ell(0)\mu(\cZ)
        +
        \int_{f/2>\bar{s}} \ell^*(f)d\mu
      \\
      &=
        \ell(0)\mu(\cZ)
        +
        \int \ell^*(f)d\mu
        -
        \int_{f/2\le \bar{s}}\ell^*(f)d\mu
      \\
      &\le
        \ell(0)\mu(\cZ)
        +
        \int \ell^*(f)d\mu
        +
        \ell(0)\mu\Parens{\Set{f/2\le \bar{s}}}
      \\
      &< \infty
    \enspace,
      \end{align*}
      where the next to last step follows, because $\ell^*(s)\geq\ell^*(\bars) = -\ell(0)$ by \Cref{fact:loss_prop}.i.
  \end{enumromansep}
\end{proof}

\begin{proofof}{\Cref{fact:duality}}
  The duality law will be proved via Fenchel's duality (\Cref{thm:fenchel}).
%It turns out that the Orlicz space $L_{\beta^*}(\mu)$ derived from its conjugate $\beta^*$ satisfies our desiderata on $\cQ$:
%it contains all the required measurable functions and also satisfies the topological pairing requirements necessary for duality.
%
%, where $x\in X$ replaced by $w\in L_1(\cH)$,
%$F(x)$ is identically zero and $G(Ax)$ is the risk functional $\int\ell(Aw)d\mu$. In the dual, $v\in V$ is replaced by $\q\in\cQ$.
%
%Specifically, the Banach space $\cQ$ needs to be large enough to ... but small enough ...
%
%The non-existence of a primal optimum necessitates much of the technical development here.
  To begin, we need to define Banach space pairings. One of them is $\bigParens{L_1(\cH),\,L'}$ where $L'$ is the topological
  dual of $L_1(\cH)$ and the other is $\bigParens{M_\beta(\mu),\,L_{\beta^*}(\mu)}$, which is a valid pairing as argued at the
  beginning of this appendix.

  We invoke \Cref{thm:fenchel} with
  $F:L_1(\cH)\to\R$, $G:M_\beta(\mu)\to\R$ defined by
\[
   F(w) = 0\text{ for all $w$,}
\quad
   G(f) = \int\ell(f)d\mu
\]
  and $A:L_1(\cH)\to M_\beta(\mu)$ defined as in \Cref{sec:intro}. Note that
  $F^*(u)=\ind[u=0]$ where $\ind$ denotes the convex indicator, yielding the constraint $A^\top\q=0$.
  To prove \Eq{duality}, it remains to show that
  $A$ is continuous as a map from $L_1(\cH)$ to $M_\beta(\mu)$,
  $G$ is finite on $M_\beta(\mu)$ and
\[
   G^*(\q) = \int\ell^*(\q)d\mu
\enspace.
\]
  Finiteness of $G$ follows by \Cref{fact:orlicz_prop}.i; the expression for the conjugate
  $G^*$ follows by \Cref{prop:conj:int}, because $M_\beta(\mu)$ and $L_{\beta^*}(\mu)$ are
  decomposable (by \Cref{prop:orlicz}). Finally, to argue continuity of $A$, consider $w,w'\in L_1(\cH)$.
  From the definition of $A$, $\abs{(Aw)(z)}\le\norm{w}_1$, so $Aw$ is a bounded measurable function
  and hence in $M_\beta(\mu)$ (by decomposability). Also,
\begin{equation}
\label{eq:A:w:w'}
  \bigAbs{\bigParens{A(w'-w)}(z)}\le\norm{w'-w}_1
\enspace.
\end{equation}
  Let $f_1(z)=1$ for all $z$.
  For any $f$ and $g$ such that $\abs{f}\le\abs{g}$, we have
  $\norm{f}_\beta\le\norm{g}_\beta$, so \Eq{A:w:w'} implies
\[
  \bigNorm{A(w'-w)}_\beta\le\norm{w'-w}_1\norm{f_1}_{\beta}
\enspace,
\]
  showing the continuity of $A$, because $\norm{f_1}_\beta$ is finite (by decomposability).

  It remains to show the properties of the dual optima:
  \begin{enumromansep}
    \item
      The bound follows since $\ell^*(s) = \infty$ whenever $s<0$ by \Cref{fact:loss_prop}.

    \item
      Any dual optimum $\barq$ may be modified on a $\mu$-null set to obtain
      $\hatq$ satisfying the condition.
      To start,
      define $S \coloneqq \{ x \in \cX : \barq(x,1) = \barq(x,-1) = 0 \}$;
      from part (i), $\barq \geq 0$ ($\mu$-a.e.), so it suffices to produce $\hatq$
      by modifying $\barq$ on a $\mu$-null subset of $S$.

      Recall that $\eta_\mu(x,y)$ represents the conditional probability of $y$ given $x$,
      i.e., $d\mu(x,y)=\eta_\mu(x,y)d\mu_\cX(x)$ and $\eta_\mu(x,-1)+\eta_\mu(x,1)=1$.
      We will write $\cY=\Set{-1,1}$.
      First consider those points where $\eta_\mu(x,y) \in (0,1)$; in particular, the set
      \[
        S_0 := \left\{ x \in S : \eta_\mu(x,1) \in (0,1) \right\},
      \]
      and, for the sake of contradiction, suppose that $\mu_\cX(S_0) > 0$.
      Pick $\bars\in\partial\ell(0)$, whereby $\ell^*(\bars) < \ell^*(0) = 0$ by \Cref{fact:loss_prop}.
      Define $\q\in L_{\beta^*}(\mu)$ as
      \[
        \q(x,y) := \begin{cases}
          \barq(x,y)
          &\textup{when } x\not\in S_0,
          \\
          \bars
          &\textup{when } x\in S_0 \textup{ and } \eta_\mu(x,-y)\ge\eta_\mu(x,y),
          \\
          \bars \cdot \frac {\eta_\mu(x,-y)}{\eta_\mu(x,y)}
          &\textup{when } x\in S_0 \textup{ and } \eta_\mu(x,-y)<\eta_\mu(x,y).
        \end{cases}
      \]
      We show that $\q$ is dual-feasible and achieves a better objective value than $\barq$.
      By construction, $\q\in L_{\beta^*}(\mu)$ (since $\barq\in L_{\beta^*}$, which is decomposable, and the adjustment is bounded),
      and moreover, for every $w\in L_1(\cH)$,
      \begin{align*}
        (A^\top\q)(w)
        &= \int_{S_0\times\cY} (Aw)\q\,d\mu + \int_{S_0^c\times\cY} (Aw)\barq\,d\mu
        \\
        &=
        \int_{S_0} (Hw)(x) \BigParens{\q(x,-1)\eta_\mu(x,-1) - \q(x,1)\eta_\mu(x,1)} d\mu_\cX(x)
        \\
        &\quad{}
        +
        \Parens{\int (Aw)\barq\,d\mu-\int_{S_0\times\cY} (Aw)\barq\,d\mu}
        \\
        &= 0 + (0 - 0)
      \enspace,
      \end{align*}
      where the last step follows from the definition of $\q$, feasibility of $\barq$ and the fact that $S_0\subseteq S$.
      Thus, $\q$ is feasible.
      On the other hand,
      \[
        \int \ell^*(\q)d\mu = \int \ell^*(\barq)d\mu + \int_{S_0\times\cY}\ell^*(\q)d\mu
      \enspace,
      \]
      because $\barq=0$ along $S_0\times\cY$.
      By construction, $\q\in(0,\bars]$ along $S_0\times\cY$. Further, $\ell^*(s) < 0$ for $s\in(0,\bars]$ by \Cref{fact:loss_prop}, so
      $\ell^*(\q)<0$ along $S_0\times\cY$. Hence, $\mu_\cX(S_0) > 0$ implies $\q$ attains a lower objective value than $\barq$, a contradiction;
      thus $\mu_\cX(S_0) = 0$.

      It has been shown that $\barq(x,y) + \barq(x,-y) > 0$ over $(x,y)$ with
      $\eta_\mu(x,y) \in (0,1)$, $\mu$-a.e.; consequently, it suffices to consider $(x,y)$ with $\eta_\mu(x,y) \in \{0,1\}$.
      Define $\hatq\in L_{\beta^*}(\mu)$ as
      \[
        \hatq(x,y) := \begin{cases}
          \barq(x,y)
          &\textup{when } \eta_\mu(x,y) \in (0,1],
          \\
          \bars
          &\textup{when } \eta_\mu(x,y) = 0.
        \end{cases}
      \]
      Since the adjustment is only on points where $\eta_\mu(x,y) = 0$,
      then $\hatq = \barq$ $\mu$-a.e., and thus is also a dual solution. Furthermore, since $\mu_\cX(S_0) = 0$,
      then $\mu_\cX$-a.e.\@ over $x\in S$, we have $\hatq(x,-1) + \hatq(x,1) \geq\bars > 0$
      as desired.

    \item
      This follows directly from \Cref{thm:fenchel} and \Cref{prop:subgrad:strict}.i.
      %haha, I didn't noticed you'd misplaced the above, and made this verbose version...
     %When a primal optimum $\barw \in L_1(\cH)$ exists,
     %the above invocation of \Cref{thm:fenchel} also grants $\barq \in \partial G(A\barw)$,
     %and it follows that $\barq(z) \in \partial\ell((A\barw)(z))$ for $\mu$-a.e. $z\in\cZ$
     %by the subgradient rule in \Cref{prop:subgrad:strict} (thanks to decomposability,
     %provided by \Cref{prop:orlicz} as discussed above).

    \item
      Consider a sequence $(w_i)_{i=1}^\infty$ minimizing the primal. By \Eq{duality} and since $A^\top\barq=0$, this means that
      \begin{equation}
      \label{eq:duality:pf:1}
           \int\ell(Aw_i)d\mu + \int\ell^*(\barq)d\mu - \angles{A^\top\barq,\,w_i}\;\to\;0
      \end{equation}
      as $i\to\infty$. Let $r_i=Aw_i$. Since $\angles{A^\top\barq,\,w_i}=\angles{\barq,\,Aw_i}=\angles{\barq,r_i}$, \Eq{duality:pf:1} can be rearranged to
      \[
           \int\BigBracks{\ell(r_i)+ \ell^*(\barq) - \barq r_i}d\mu\;\to\;0
      \enspace.
      \]
      By Fenchel's inequality, the integrand is non-negative, so we actually have
      \begin{equation}
      \label{eq:duality:pf:2}
         \ell(r_i(z)) + \ell^*(\barq(z)) - \barq(z)r_i(z)\;\to\;0\quad\text{$\mu$-a.e.\@ over $z\in\cZ$.}
      \end{equation}
      Denote the set of points $z$ where $\ell^*$ is differentiable at $\barq(z)$ as
      $S$. Define $\barr(z)\coloneqq(\ell^*)'(\barq(z))$
      for $z\in S$. Over $z\in S$, we have by first-order optimality for conjugates that $\ell^*(\barq)=\barq\barr-\ell(\barr)$,
      and $\barq=\ell'(\barr)$, and thus \Eq{duality:pf:2} implies
      \[
         \ell(r_i) - \ell(\barr) - \ell'(\barr)(r_i-\barr)\;\to\;0\quad\text{$\mu$-a.e.\@ over $z\in S$.}
      \]
      Hence, from strict convexity of $\ell$ we obtain that $r_i\to\barr$, $\mu$-a.e.\@ over $z\in S$.
      Now, let $S_\cX\coloneqq\{x\in\cX:\: (x,1)\in S\text{ and }(x,-1)\in S\}$ be the set of points $x$ where $\ell^*$ is
      differentiable at both $\barq(x,1)$ and $\barq(x,-1)$. From the definition of $r_i$, we have $r_i(x,1)+r_i(x,-1)=0$ and thus
      we must also have $\barr(x,1)+\barr(x,-1)=0$, $\mu_\cX$-a.e.\@ over $x\in S_\cX$. Unrolling the definition of $\barr$
      yields the desired result.

    \item
      If $\ell$ is differentiable, then $\ell^*$ is strictly convex \citep[Theorem E.4.1.2]{HULL},
      whereby $\int\ell^*d\mu$ is also strictly convex by \Cref{prop:subgrad:strict}.ii, and thus
      the dual optimizer is unique up to $\mu$-null sets.

  \end{enumromansep}
  %%XXX no time:
  %\MJT{proof needs material for $\textup{Breg}$.}
\end{proofof}

To close, note an additional technical property of $\barq$ which will be useful in various proofs.

\begin{lemma}
  \label[lemma]{fact:barq:endpoint}
  Given finite measure $\mu$,
  hypotheses $\cH$,
  and loss $\ell\in\Ldiff$ with $L := \lim_{r\to\infty}\ell'(r)$,
  it follows that every dual optimum $\barq$ satisfies $\mu(\{z\in\cZ : \barq(z) \geq L\}) = 0$.
\end{lemma}
\begin{proof}
  Note that $\ell^*$ is strictly convex (by differentiability of $\ell$) and differentiable everywhere except possibly at the endpoints of its domain (by strict convexity of $\ell$).
  If $L=\infty$, there is nothing to show, thus suppose $L<\infty$,
  which entails $\dom(\ell^*) \subseteq [0,L]$ (since the image of the derivative map $\ell'$ is the domain of the conjugate derivative map $(\ell^*)'$, and this coincides, up to the endpoints, with $\dom\ell^*$). So it suffices to show that $\mu(\barq=L)=0$.
  
  Note that $\ell'(0) \in (0,L)$ since $\ell'' > 0$.
  Define
  a scalar $N := (L + \ell'(0)) / 2$, set $D:=\{z\in\cZ : \barq(z) \in (0,L] \}$,
  and partition $D$ into the three pieces
  \begin{align*}
    R_1
    &:= \{ z\in \cZ : \barq(z) \in (0, N] \},
    \\
    R_2
    &:= \{ z\in \cZ : \barq(z) \in  (N, L) \},
    \\
    R_3
    &:= \{ z\in \cZ : \barq(z) =  L \}.
  \end{align*}
  We next study the integral $\int \ell^*((1-\alpha)\barq)d\mu$ for small values of $\alpha$ over these pieces.

  \begin{description}
    \item[($R_2$)]
     %Since $\ell$ is differentiable,
     %then $\ell^*$ is strictly convex \citep[Theorem E.4.1.2]{HULL},
     %meaning $\ell^*$ is increasing along $[\ell'(0), L]$.
      Since $\ell^*$ is increasing along $[\ell'(0), L]$,
      then every sufficiently small $\alpha>0$ and every $z\in R_2$ satisfies $\ell^*((1-\alpha)\barq(z)) < \ell^*(\barq(z))$,
      and in particular
      \[
        \int_{R_2} \ell^*((1-\alpha)\barq)d\mu
        \leq \int_{R_2} \ell^*(\barq)d\mu.
      \]

    \item[($R_1$)]
      Consider the function
      \[
        F(\alpha) = \int_{R_1} \ell^*((1-\alpha)\barq)d\mu.
      \]
      This is a univariate convex function which is finite on a neighborhood
      of $0$. Pick $\tau>0$ such that $[-\tau,\tau]$ lies in this neighborhood.
      Since this is a closed bounded subset of the relative interior of $\dom F$,
      we obtain (by \citealp[Theorem 10.4]{ROC}) that $F$ is Lipschitz-continuous
      on $[-\tau,\tau]$. Let $L'$ be its Lipschitz constant on $[-\tau,\tau]$.
      For $|\alpha|\le\tau$, we obtain
      \[
        \int_{R_1}  \ell^*((1-\alpha)\barq)d\mu
        \leq
        \alpha L' + \int_{R_1} \ell^*(\barq)d\mu.
      \]
%      (The point of constructing $R_1$ to end at $N\in(\ell'(0),L)$
%      was to allow the above map to consider $\alpha < 0$;
%      this wiggle room allows the local Lipschitz property to be applied at
%      $\alpha=0$.)

    \item[($R_3$)]
      Note $\lim_{z\uparrow L} (\ell^*)'(z) = \infty$ (by \Cref{fact:loss_prop}), thus the definition of subgradient grants
      \begin{align*}
        \int_{R_3} \ell^*((1-\alpha)\barq)d\mu
        &= \mu(R_3)\ell^*((1-\alpha)L)
        \\
        &\leq \mu(R_3)\left( \ell^*(L) - (\ell^*)'((1-\alpha)L)(L - (1-\alpha)L) \right)
        \\
        &= -\alpha L \mu(R_3) (\ell^*)'((1-\alpha)L)
        +
        \int_{R_3} \ell^*(\barq)d\mu.
      \end{align*}
  \end{description}

  To finish,
  first note $\barq \in [0,L]$ for $\mu$-a.e.\@ $z\in\cZ$ (since otherwise $\int\ell^*(\barq)d\mu > \int\ell^*(0)d\mu = 0$),
  and $\ell^*((1-\alpha)\barq) = 0$ wherever $\barq = 0$.
  Combining these pieces,
  since $\barq$ is optimal and $(1-\alpha)\barq$ is feasible for $\alpha\in[0,1]$,
  then for sufficiently small $\alpha > 0$,
  \begin{align*}
    \int \ell^*(\barq)d\mu
    &\leq
    \int \ell^*((1-\alpha)\barq)d\mu
    \\
    &=
    \int_{R_1} \ell^*((1-\alpha)\barq)d\mu
    +
    \int_{R_2} \ell^*((1-\alpha)\barq)d\mu
    +
    \int_{R_3} \ell^*((1-\alpha)\barq)d\mu
    \\
    &\leq
    \alpha L' + \int_{R_1} \ell^*(\barq)d\mu
    +
    \int_{R_2} \ell^*(\barq)d\mu
    - \alpha L \mu(R_3) (\ell^*)'((1-\alpha)L)
    + \int_{R_3} \ell^*(\barq)d\mu.
    \\
    &=
    \alpha\Big(L' - L \mu(R_3) (\ell^*)'((1-\alpha)L)\Big) + \int\ell^*(\barq)d\mu,
  \end{align*}
  which rearranges to give
  \[
    L \mu(R_3) \underbrace{(\ell^*)'((1-\alpha)L)}_{\triangle}
    \leq
    L'.
  \]
  Since $L > 0$ and $\triangle \to \infty$ as $\alpha \downarrow 0$ whereas $L'$ is constant,
  it follows that $\mu(R_3) = 0$.
%  This gives the full result, since $\mu(\{z\in\cZ:\barq(z) > L\}) = 0$
%  was argued above as a consequence of $\dom(\ell^*) = [0,L]$.
\end{proof}

\section{Proof of \Cref{lemma:dual:D} and \Cref{cor:hc_split}}

This brief appendix section collects proofs of two results from the introductory part
of \Cref{sec:proof_outlines}.
%in a nutshell, $\cR$ can be split along $\cD$ and $\cD^c$.

%\begin{proof}[Proof of \Cref{fact:hc_split}]
\begin{proofof}{\Cref{lemma:dual:D}}
  Applying \Cref{fact:duality}, to both $\mu$ and $\mu_D$,
  \begin{align*}
    \inf_{w\in L_1(\mu)} \cR(w) &=
    \max_{\q\in L_{\beta^*}(\mu):\:A^\top \q = 0}
       \Bracks{ - \int \ell^*(\q)d\mu },
       \\
       \inf_{w\in L_1(\mu)} \cR(w;\mu_D) &=
    \max_{\q\in L_{\beta^*}(\mu_D):\:A^\top \q = 0}
       \Bracks{ - \int_D \ell^*(\q)d\mu }.
     \end{align*}
  Of course, $\barq$ attains the first dual maximum over $\mu$; note, as follows, that
  it also attains the dual maximum over $\mu_D$.
  First, $\barq$ is feasible for the second problem, since $\barq \in L_{\beta^*}(\mu)$ and $\barq=0$ on $D^c$, so we also have $\barq \in L_{\beta^*}(\mu_D)$,
  and for every $v\in L_1(\cH)$,
  \[
    0
    = \int (Av)\barq\,d\mu
%    = \int_D (Av)\barq d\mu + \int_{D^c} (Av)\barq d\mu
    = \int_D (Av)\barq\,d\mu.
  \]
  Furthermore, since $\ell \in \Lclass$ implies $\ell^*(0) = 0$ (by \Cref{fact:loss_prop}),
  it follows that
  \[
    \int \ell^*(\barq)d\mu
%    = \int_D \ell^*(\barq)d\mu + \int_{D^c} \ell^*(\barq)d\mu
    = \int_D \ell^*(\barq)d\mu.
  \]
  Consequently,
  \begin{equation}
    \max_{\substack{\q\in L_{\beta^*}(\mu) \\ A^\top \q = 0}}
       \Bracks{ - \int \ell^*(\q)d\mu }
       = -\int \ell^*(\barq)d\mu
       = -\int_D \ell^*(\barq)d\mu
      \leq
      \max_{\substack{\q\in L_{\beta^*}(\mu_D)\\ A^\top \q = 0}}
       \Bracks{ - \int_D \ell^*(\q)d\mu }.
       \label{eq:dual:D:1}
  \end{equation}

  Now consider any dual optimum $\barq_D$ over $\mu_D$,
  and set $\hatq(z) := \barq_D(z)\one[z\in D]$.
  Mimicking the derivations above, $\hatq$ is feasible and optimal over $D$
  (indeed, $\barq_D$ and $\hatq$ only differ on a $\mu_D$-null set).
  Similarly, however, $\hatq$ is also feasible for the full problem over $\mu_D$,
  and $\int \ell^*(\hatq)\mu_D = \int\ell^*(\hatq)d\mu$,
  implying that the inequality in \Eq{dual:D:1} is an equality, and
  $\barq$ and $\hatq$ are optimal for both $\mu$ and $\mu_D$.
\end{proofof}

\begin{proofof}{\Cref{cor:hc_split}}
  Using the fact that the dual and thus also primal optimal values coincide
  for $\mu$ and $\mu_D$, as well as the fact that $\ell\ge 0$, we obtain
  \begin{align*}
    \cE(w;\mu_D)
    = \int_D\ell(Aw)d\mu - \inf_{v\in L_1(\cH)} \int_D \ell(Av)d\mu
    & \leq \int\ell(Aw)d\mu - \inf_{v\in L_1(\cH)} \int \ell(Av)d\mu
    = \cE(w)
  \end{align*}
  directly, and similarly
  \begin{align*}
    \cR(w;\mu_{D^c})
    &= \int\ell(Aw)d\mu - \int_D\ell(Aw)d\mu
    \\
    &\leq \int\ell(Aw)d\mu - \inf_{v\in L_1(\cH)} \int_D\ell(Av)d\mu
    \\
    &= \int\ell(Aw)d\mu - \inf_{v\in L_1(\cH)} \int\ell(Av)d\mu
    \\
    &= \cE(w).
  \end{align*}
%\end{proof}
\end{proofof}

\section{Proofs from \Cref{subsec:sketch:convergence}}

Before proving \Cref{fact:Dc_controls} in full,
we establish a general form of its first part.
Unlike the proof of the second part of \Cref{fact:Dc_controls},
the first part does not rely upon the structure of $\cD^c$ in any way;
indeed it is simply Markov's inequality.

\begin{lemma}
  \label[lemma]{fact:loss_markov}
  Let finite measure $\mu$, hypotheses $\cH$, loss $\ell\in\Lclass$, and arbitrary set $C\subseteq \cZ$ be given.
  Then for any $w\in L_1(\cH)$ and $r >0$, the set $S_r := \{z\in C : \ell\bigParens{(Aw)(z)} \geq r \}$
  satisfies $\mu(S_r) \leq \cR(w;\mu_C)/r$.
\end{lemma}
\begin{proof}
  Emulating the proof of Markov's inequality,
  every $z\in\cZ$ satisfies
  \[
    r \one[ z\in S_r ]
    \leq r \one[ \ell\bigParens{(Aw)(z)} \geq r]
    \leq \ell\bigParens{(Aw)(z)},
  \]
  thus integrating both sides along $C$ and dividing by $r$ gives
  \[
    \mu(S_r)
    \leq \frac {\int_{C}\ell(Aw)d\mu}{r}
    = \frac {\cR(w;\mu_{C})}{r}.
  \]
\end{proof}

%\begin{proof}[Proof of \Cref{fact:Dc_controls}]
\begin{proofof}{\Cref{fact:Dc_controls}}
  Part (i) is proved by applying \Cref{fact:loss_markov} with $C:= D^c$,
  and then applying \Cref{cor:hc_split} for the inequality $\cR(w;\mu_{D^c}) \leq \cE(w)$.

  For part (ii), first note that if $r\ge\ell(0)$, we are done, because $|\bar \eta - \eta_w|\le 1$. Now consider $r<\ell(0)$.
  Since $\bar \eta = 1$ for $\mu$-a.e.\@ $z\in D^c$ and $\eta_w \geq 0$ by definition,
  then
  \begin{align*}
    \int_{D^c \setminus S_r}\left|
    \bar \eta - \eta_w
    \right|d\mu
    &=
    \int _{D^c\setminus S_r} (1-\eta_w) d\mu
    =
    \int _{D^c\setminus S_r} \frac {\ell'((Aw)(x,y))}{\ell'((Aw)(x,y)) + \ell'((Aw)(x,-y))} d\mu
    =: \heartsuit,
  \end{align*}
  thus it remains to control $\heartsuit$.
  Since every $z\in D^c\setminus S_r$ has $\ell\bigParens{(Aw)(z)} < r < \ell(0)$,
  the increasing property of $\ell$ implies $(Aw)(z) \leq 0$.
  Consequently, it follows that $\ell'\bigParens{(Aw)(z)} \leq c_\ell\ell\bigParens{(Aw)(z)} < c_\ell r$,
  and also that
  $\ell'((Aw)(x,-y)) = \ell'(-(Aw)(x,y)) \geq \ell'(0)$ since $\ell'$ is nondecreasing by convexity.
  Combining these bounds,
  \begin{align*}
    \heartsuit
    &=
    \int_{D^c\setminus S_r} \frac {1}{1 + \ell'((Aw)(x,-y)) / \ell'((Aw)(x,y))} d\mu(x,y)
   %\\
   %&
    \leq
    \frac {\mu(D^c\setminus S_r)}{1 + \ell'(0) / (c_\ell r)},
  \end{align*}
  which gives the desired bound after rearrangement, noting that $c_\ell r>0$.
\end{proofof}
%\end{proof}

In order to prove \Cref{fact:D_controls}, it will be necessary to establish an additional
structural property of dual optima.
In particular, recall the function $\barf$, which is used in the proof of \Cref{fact:D_controls},
and which is equal to $(\ell')^{-1}(\barq)$ whenever
$(\ell')^{-1}$ is defined for both $\barq(x,y)$ and $\barq(x,-y)$.
It is this final condition---needing both $(x,y)$ and $(x,-y)$---which requires the
extra work here.

For the purposes of \Cref{fact:D_controls}, it will suffice to establish that $\mu$-a.e $(x,y)\in \cD$
satisfies $(x,-y)\in\cD$, which is precisely the following \namecref{fact:terrible}.
This result is in fact a consequence of \Cref{fact:Dc_controls}: the idea is that for those points with $(x,y)\in\cD$ but
$(x,-y)\in\cD^c$, applying \Cref{fact:Dc_controls} grants that every low error predictor must achieve
small error on this latter set.
But this leads to a contradiction, since it necessitates that the error on the mirrored points, which reside in $\cD$,
must be large.

\begin{lemma}
  \label[lemma]{fact:terrible}
  Let finite measure $\mu$, hypotheses $\cH$, and loss $\ell\in\Lbnd$ be given.
  Then there exists a dual optimum $\barq$ and corresponding difficult set $\cD$
  such that $\mu$-a.e.\@ over $(x,y)\in\cD$ we also have $(x,-y)\in\cD$.
\end{lemma}
\begin{proof}
  Let an arbitrary dual optimum $\barq_0$ be given as provided by \Cref{fact:duality},
  and let $\cD_0$ denote the corresponding difficult set.
  If this provided $\barq_0$ already satisfies the necessary properties, the proof is done, therefore suppose it does not.

  Define three sets
  \begin{align*}
    K_0
    &:=
    \left\{
      (x,y) \in \cD_0
      :
      (x,-y)\in \cD_0^c,
      \eta_\mu(x,y) = 0
    \right\},
    \\
    K_1
    &:=
    \left\{
      (x,y) \in \cD_0
      :
      (x,-y)\in \cD_0^c,
      \eta_\mu(x,y) = 1
    \right\},
    \\
    K_+
    &:=
    \left\{
      (x,y) \in \cD_0
      :
      (x,-y)\in \cD_0^c,
      \eta_\mu(x,y) \in (0,1)
    \right\},
  \end{align*}
  and an adjusted dual optimum %interestingly, the second case seems to be a null set as well... but whatever, this is easier than proving that
  \[
    \barq(x,y)
    :=
    \begin{cases}
      0
      &\textup{when $(x,y) \in K_0$},
      \\
      \ell'(0)
      &\textup{when $(x,-y) \in K_1$},
      \\
      \barq_0(x,y)
      &\textup{otherwise}.
    \end{cases}
  \]
  Since $\mu(K_0) = 0 = \mu(\{(x,y) : (x,-y) \in K_1\})$ by construction,
  then $\barq = \barq_0$ $\mu$-a.e., meaning $\barq$ is also a dual optimum to \Eq{duality}.
  Defining $\cD := \{z\in\cZ : \barq(z) > 0\}$,
  if $(x,y) \in \cD$ and $(x,-y) \in \cD^c$, then it must hold that $(x,y) \in K_+$.
  The proof is done if $\mu(K_+) = 0$; this will constitute the remainder of the proof.

  Assume contradictorily that $\mu(K_+) > 0$.
  Define
  \[
    U_\xi := \left\{ (x,y) \in K_+ : \min\{\eta_\mu(x,y), \eta_\mu(x,-y)\} \geq \xi\right\}.
  \]
  By continuity of measures, $\lim_{\xi\downarrow 0} \mu(U_\xi) = \mu(\cup_{\xi>0}U_\xi) = \mu(K_+)$,
  thus there exists a fixed $\tau>0$ so that $U := U_\tau$ has $\mu(U)\geq \tau$.
  For convenience, define $U_- := \{(x,-y) : (x,y) \in U\}$ (and use $S_-$ for this ``flipped sign'' transformation of any set $S\subseteq \cZ$).  By the conditions on $U$,
  then $\mu(U_-) \geq \tau \mu(U) \geq \tau^2$,
  and for any set $C\subseteq \cZ$,
  \begin{equation}
    \mu(U \cap C_-)
    \geq \tau \mu(U_- \cap C).
    \label{eq:sotired:30}
  \end{equation}

  Now choose $\eps_0>0$ so that $\ell(-\ell^{-1}(\sqrt{\eps_0})) > 6\cR(0)/\tau^3$,
  set $\eps :=  \min\{ \eps_0, \tau^4/4, \cR(0)\}$,
  and choose $w\in L_1(\cH)$ with $\cE(w) \leq \eps$.
  Applying \Cref{fact:Dc_controls} to $w$ with $r:= \sqrt{\eps}$, the
  set
  \[
    S_r := \left\{
      z\in\cD^c : \ell((Aw)(z)) \ge r
    \right\}
  \]
  satisfies $\mu(S_r) \leq \eps/r=r$.  For convenience, define $V:= \cD^c \setminus S_r$,
  whereby $\mu(V) \geq \mu(\cD) - r$, and every $z\in V$ has $\ell((Aw)(z)) \leq \sqrt{\eps}$,
  which will be more useful in the form $(Aw)(z) \leq \ell^{-1}(\sqrt{\eps})$.
  Furthermore, since $U_- \subseteq \cD^c$,
  \[
    \tau^2 \leq \mu(U_-)
    = \mu(U_- \cap V) + \mu(U_- \cap V^c)
    \leq \mu(U_- \cap V) + \mu(\cD^c \cap V^c)
    = \mu(U_- \cap V) + \mu(S_r),
  \]
  which rearranges to give $\mu(U_- \cap V) \geq \tau^2 - \mu(S_r) \geq \tau^2 - \sqrt{\eps} \geq \tau^2/2$.
  Note by \Eq{sotired:30} that
  \[
    \mu(U\cap V_-)
    \geq \tau \mu(U_- \cap V)
    \geq \frac {\tau^3}{2},
  \]
  and $z\in V_-$ has $(Aw)(z) \geq -\ell^{-1}(\sqrt{\eps})$,
  and more importantly $\ell((Aw)(z)) \geq \ell(-\ell^{-1}(\sqrt{\eps})) > 6\cR(0)/\tau^3$.
  Consequently, since $\cE(w)\le\cR(0)$ and $\cR(0)>0$,
  \begin{align*}
    \cR(w)
    &\geq \int_{U \cap V_-} \ell(Aw)d\mu(x,y)
    \\
    &\geq \mu(U\cap V_-) \ell(-\ell^{-1}(\sqrt{\eps}))
    \\
    &\geq 3\cR(0)
    \\
    &\ge \cR(0) + \cE(w) + \inf_{v\in L_1(\cH)}\cR(v)
    \\
    &= \cR(0) + \cR(w)
    \\
    &> \cR(w),
  \end{align*}
  a contradiction.
\end{proof}

\iffalse
With \Cref{fact:terrible} out of the way, the proof of \Cref{fact:D_controls} now follows.
The proof has two essential elements.  First, by the constraint $A^\top \barq = 0$,
every $w\in L_1(\cH)$ satisfies
\[
  \int_{Aw > 0} (Aw)\barq d\mu = -\int_{Aw < 0} (Aw) \barq d\mu,
\]
which means that the presence of many correct predictions (the case $0 > (Aw)(x,y) = -y(Hw)(x)$) implies the presence of
many incorrect predictions (the case $0 < (Aw)(x,y) = -y(Hw)(x)$).  Making this sketch precise will control the various
sets in the statement.
Note briefly that aspects of this part of the proof resemble arguments made for the logistic loss
\citep[Proof of Lemma 44]{mjt_logistic};
those derivations were tied to the $(L_1(\mu), L_\infty(\mu))$ topology pair and moreover made use of
the corresponding H\"older inequality, which is handled more generally here thanks to Orlicz spaces
and their H\"older inequality (cf. \Cref{prop:orlicz}).

The second element is controlling $\bar\eta$ over the well-behaved set $U$.  Thanks to the restrictions on $U$,
this ends up being a combination of a Taylor expansion, and the fact that $\phi(-\barf) = \bar\eta$ $\mu$-a.e.\@ over $U$,
a consequence of \Cref{fact:terrible} and \Cref{fact:duality}.iv.
\fi

%\begin{proof}[Proof of \Cref{fact:D_controls}]
\begin{proofof}{\Cref{fact:D_controls}}
  First consider $S_+$; by $\ell \geq 0$ and convexity, $\ell(c_1)\ge \ell(0)+ c_1\ell'(0)\ge c_1 \ell'(0)$,
  and
  \[
    \cR(w) \geq \int_{S_+}\ell(Aw)d\mu \geq \ell(c_1) \mu(S_+) \geq c_1\ell'(0)\mu(S_+),
  \]
  which rearranges to give $\mu(S_+)\leq \cR(w) / (c_1 \ell'(0))$.

  To control $S_-$ we take advantage of $S_+$: the region $\cD$ is a set of points where
  it is impossible for $S_-$ to be large without $S_+$ being large as well, and $\barq$ is a witness
  to this fact.
  To start, note by $A^\top \barq = 0$ and $\barq = 0$ on $\cD^c$ that
  \[
    0 =  \ip{Aw}{\q} = \int_\cD(Aw)\barq d\mu =
    \int_{Aw > 0} (Aw)\barq d\mu
    + \int_{Aw < 0} (Aw)\barq d\mu,
  \]
  which rearranges to yield
  \[
    \int_{Aw > 0} (Aw)\barq d\mu
    = -  \int_{Aw < 0} (Aw)\barq d\mu.
  \]
  Combining this with H\"older's inequality for Orlicz spaces (see \Cref{prop:orlicz}),
  \begin{align*}
    2\|\barq\|_{\beta^*}\left\|Aw\,\one[Aw > 0]\right\|_{\beta}
    &\geq \left|\ip{\barq}{Aw\,\one[Aw > 0]}\right|
    \\
    &= \left| \ip{\barq}{Aw\,\one[Aw < 0]} \right|
    \\
    &\geq c_1 c_2 \mu(S_-).
  \end{align*}
  Now using the definition of $c_{\ell,\mu}$ and rearranging,
  \[
    \cR(w)
    \geq \int_{Aw > 0} \ell(Aw)d\mu
    \geq \frac
    {c_1 c_2 \mu(S_-)}
    {2c_{\ell,\mu} \|\barq\|_{\beta^*}},
  \]
  which gives the desired bound on $\mu(S_-)$.

%%XXX not using this any more...
% The bound on $\cD\setminus U$ now follows by plugging the above bounds into
% \begin{align*}
%   \mu(\cD\setminus U)
%   \leq \mu(S_+)
%   + \mu(\{z \in S_- : \barq(z) \geq c_2\})
%   + \mu(\{z \in \cD : c_2 > \barq(z) \ \lor \ \barq(z) > c_3\}).
% \end{align*}
% Moreover, the limiting statement on $\mu(\{z \in \cD : c_2 > \barq(z) \ \lor \ \barq(z) > c_3\})$
% follows from continuity of measures, since the sets $(\cD_i)_{i\geq 1}$ with
% $\cD_i := \{z\in\cD : 1/i > \barq(z) \ \lor \ \barq(z) > i\}$ satisfy
% \[
%   \cap_{i\geq 1} \cD_i = \cD^c
%   \qquad\textup{and}\qquad
%   \lim_{j\to\infty} \mu(\cD \setminus \cap_{i\leq j} \cD_j) = \mu(\cD\setminus \cD) = 0.
% \]

  In order to control $|\bar\eta - \eta_w|$ on $U$,
  suppose without loss of generality that $\mu$-a.e.\@ $(x,y) \in \cD$ satisfies $(x,-y)\in\cD$ (see \Cref{fact:terrible}),
  and define a scalar $L:=\lim_{r\to\infty}\ell(r)$,
  a set $\cD' := \{z\in\cD : z < L\}$,
  and a function
  \[
    \barf(z) := \begin{cases}
      (\ell^*)'(\barq(z))
      &\textup{when $z\in\cD'$},
      \\
      0
      &\textup{otherwise,}
    \end{cases}
  \]
  Note that $\barf$ is well-defined (and measurable) by construction,
  since strict convexity of $\ell$ implies differentiability of $\ell^*$ along the interior
  of $\dom(\ell^*)$ \citep[Theorem E.4.1.1]{HULL}, which coincides with the set $\cD'$
  (because the domain of $(\ell^*)'$ is the image of $\ell'$ by first-order optimality for conjugates).
  %XXX old comments:
  % By construction, $f$ is measurable., and $\int_\cD f \cdot d\mu$ is an element of
  % $L_{\beta^*}'(\mu)$ \citep[Theorem 2.3]{leonard_entropy};
  % no attempt is made here to control this object everywhere, instead it will be controlled pointwise
  % (in contrast with, e.g., the object $\widetilde w$ of \citet{leonard_entropy}).
  % To this end, by
  By Taylor's theorem, for every $z\in \cD'$ there exists $q_z \in [(Aw)_z, \barf(z)]$ with
  \begin{align*}
    \ell\bigParens{(Aw)(z)}
    &= \ell(\barf(z)) + \ell'(\barf(z))\bigParens{(Aw)(z) - \barf(z)} + \frac 1 2 \bigParens{(Aw)(z) - \barf(z)}^2\ell''(q_z)
    \\
    &= -\ell^*(\barq(z)) + \barq(z)\bigParens{(Aw)(z)} + \frac 1 2 \bigParens{(Aw)(z) - \barf(z)}^2\ell''(q_z)
    \\
    &\geq -\ell^*(\barq(z)) + \barq(z)\bigParens{(Aw)(z)} + \frac \tau 2 \bigParens{(Aw)(z) - \barf(z)}^2\one[ z\in U],
  \end{align*}
  where the second line made use of $\barq(z) = \ell'(\barf(z))$
  and Fenchel's inequality.
  All terms in this final bound are integrable over $\cD'$,
  and moreover either $\cD=\cD'$, or $L<\infty$ and $\mu(\cD\setminus\cD') = 0$ by \Cref{fact:barq:endpoint},
  thus applying $\int_\cD$ to both sides gives
  \begin{align*}
    \int_{\cD} \ell(Aw)d\mu
    &\geq -\int_\cD\ell^*(\barq)d\mu
    + \ip{\barq}{Aw}
    + \frac \tau 2 \int_U (Aw - \barf)^2d\mu
    \\
    &= \inf_{v\in L_1(\cH)} \int_\cD\ell(Av)d\mu
    + \frac \tau 2 \int_U (Aw - \barf)^2d\mu,
  \end{align*}
  which made use of $A^\top\barq = 0$ and the fact that $\barq$ also
  maximizes the dual problem restricted to $\mu_\cD$ (by \Cref{lemma:dual:D}).
  Rearranging the preceding Taylor expansion gives
  \begin{equation}
    \int_U (Aw - \barf)^2d\mu \leq \frac {2\cE(w;\mu_\cD)}{\tau}.
    \label{eq:D_controls:1}
  \end{equation}
  The next step is to convert between $\barf$ and $\bar\eta$.
  To this end, recall from the construction of $\barf$ and subsequent discussion
  that $\barf(z) = (\ell^*)'(\barq(z))$ for $\mu$-a.e.\@ $z\in\cD$ (and $\mu$-a.e.\@ $(x,y)\in\cD$ has $(x,-y)\in\cD$),
  thus \Cref{fact:duality}.iv
  grants
  \[
    \barf(x,y) = (\ell^*)'(\barq(x,y)) = -(\ell^*)'(\barq(x,-y)) = -\barf(x,-y)
    \qquad\textup{for $\mu$-a.e.\@ $(x,y)\in\cD$}.
  \]
  In particular,
  this grants $\phi(-\barf(z)) = \bar\eta(z)$ for $\mu$-a.e.\@ $z\in\cD$,
  which combined with \Eq{D_controls:1}
  and the notation $L_\phi$ for the Lipschitz constant of $\phi$
  means
  \begin{align*}
    \int_U\left|
    \eta - \eta_w
    \right|d\mu
    &=
    \int_U\left|
    \phi(-\barf(z)) - \phi\bigParens{-(Aw)(z)}
    \right|d\mu(z)
    \\
    &\leq
    L_\phi \int_U\left|
    \barf - Aw
    \right|d\mu
    \\
    &\leq
    L_\phi \sqrt{\int_U\left|
      \barf - Aw
      \right|^2d\mu
    }
    \\
    &\leq
    L_\phi \sqrt{\frac {2\cE(w;\mu_\cD)}{\tau}},
  \end{align*}
  where the penultimate step used Jensen's inequality.
  %this one removed due to space!:
%%The final remaining inequality in the original statement
%%then follows from $\cE(w;\mu_\cD) \leq \cE(w)$ (by \Cref{cor:hc_split}).
%\end{proof}
\end{proofof}

%To finish this proof appendix, the proof of \Cref{fact:convergence:simplified} is a combination
%of \Cref{fact:Dc_controls}, \Cref{fact:D_controls}, and an unpleasant variety of massage.

\begin{proofof}{\Cref{fact:convergence:simplified}}
  First note that the bound for a single $w\in L_1(\cH)$ immediately implies the convergence result,
  thus it suffices to prove the bound.

  To this end, let $w\in L_1(\cH)$ be given,
  set $\eps := \cE(w)$, and before defining $f_1$ (which will not depend on $w$), define two helper functions:
  \begin{align*}
       \tau(r)&=\inf_{\abs{z}\le r}\ell''(z)\enspace,
  \\   g_\eps &= \begin{cases}
           \min\Set{r\ge0: \tau(r)\le 2\sqrt{\eps}}
           &\text{if $2\sqrt{\eps}\le\tau(1)$,}
       \\
           1
           &\text{if $2\sqrt{\eps}>\tau(1)$.}
       \end{cases}
  \end{align*}
%  \begin{align*}
%    g_\eps
%    &:= \max\left\{1,
%    \sup \left\{ c_1 \geq 1 : \inf_{|z|\leq c_1} \ell''(z) \geq 2\sqrt{\eps}\right\}\right\}.
%  \end{align*}
  The key properties are that $\tau(r)>0$, it is continuous, non-increasing, and $\lim_{r\to\infty}\tau(r)=0$ because
  $\liminf_{r\to-\infty} \ell''(r) = 0$ (by \Cref{fact:loss_prop}).
  On the other hand, the definition of $g_\eps$ implies that
  \[
     \tau(g_\eps)=\min\Set{2\sqrt{\eps},\,\tau(1)}
  \enspace,
  \]
  which means that $g_\eps\to\infty$ as $\eps \downarrow 0$.

  Next, $f_1$ will be constructed by splitting $\|\bar \eta - \eta_w\|_1$ along $\cD$ and $\cD^c$,
  and subsequently using \Cref{fact:D_controls} and \Cref{fact:Dc_controls} to control each term.
  When applying \Cref{fact:Dc_controls}, the bound may be simplified by using $r:= \sqrt{\eps}$ and $\mu(\cD^c\setminus S_r) \leq 1$.
  When applying \Cref{fact:D_controls}
  (and using \Cref{cor:hc_split} to give $\cE(w;\mu_\cD) \leq \eps$),
  the bound may be simplified by setting
  $c_1 := g_\eps$,
  $c_2 := \max\{c_1^{-1/2}, \ell'(-g_\eps)\}$,
  and
  $c_3 := \ell'(g_\eps)$.
  With these definitions, it follows that the $\tau$ of \Cref{fact:D_controls},
  which equals $\min\{ \inf_{|z|\leq c_1} \ell''(z) , \inf_{z\in[c_2,c_3]} \ell''((\ell^*)'(z)) \}$, coincides with $\tau(g_\eps)$.
  If $c_3 < c_2$, set $f_1(\eps) = 1$; otherwise, \Cref{fact:D_controls} may be applied,
  and together with the terms from \Cref{fact:Dc_controls} it follows that
  \begin{align*}
    \int |\bar\eta - \eta_w|d\mu
    &=
    \int_{\cD^c} |\bar\eta - \eta_w|d\mu
    + \int_{\cD} |\bar\eta - \eta_w|d\mu
    \\
    &\leq
    \sqrt{\eps} % + \underbrace{\frac {\sqrt{\eps}}{\sqrt{\eps} + \ell'(-\ell'^{-1}(\sqrt{\eps}))}}_{\heartsuit}
    + \sqrt{\eps}\max\left\{ \frac {1}{\ell(0)}, \frac {c_\ell}{\ell'(0)}\right\}
    \\
    &\qquad+ \left(\eps + \inf_{v\in L_1(\cH)} \int \ell(Av)d\mu\right)
      \left(
        \frac {1}{g_\eps \ell'(0)}
        + \frac {2 c_{\ell,\mu}\|\barq\|_{\beta^*}}{\sqrt{g_\eps}}
      \right)
      \\
      &\qquad + \underbrace{L_\phi\sqrt{\frac{2\eps}{\tau(g_\eps)}}}_{\star}
      \\
      &\qquad + \underbrace{\mu(\{z \in \cZ : \barq(z) \in (0, \max\{\ell'(-g_\eps), g_\eps^{-1/2}\})
      \lor \barq(z) > \ell'(g_\eps)\})}_{\triangle}
      \\
    &=: f_1(\eps).
  \end{align*}
  By construction, $f_1$ is well-defined, does not depend on $w$, and satisfies the desired inequality; it remains to be shown that
  $f_1(\eps)\to 0$ as $\eps\downarrow 0$.  It suffices to consider $\triangle$ %, $\heartsuit$,
  and $\star$, since all
  other terms contain $\eps$ in a numerator, or $g_\eps$ in a denominator (where, as shown before,
  $g_\eps\to\infty$ as $\eps\downarrow 0$), without any worry of cancellations mitigating these effects.

  %XXX new form of \Cref{fact:Dc_controls} obviates need for an argument here, but leaving stuff in as comments in case of weirdness...
 %First, to handle $\heartsuit$, \Cref{fact:loss_prop} grants ${\ell'}^{-1}(\sqrt{\eps})$ either grows unboundedly or converges to a positive
 %constant as $\eps \downarrow 0$; either way, $\heartsuit$ has $\sqrt{\eps}$ in the numerator and a denominator bounded below by a constant,
 %and thus $\heartsuit\to 0$ as $\eps\downarrow 0$.

  To handle $\triangle$, first expand the terms as
  \[
    \triangle \leq
    \underbrace{\mu(\{z \in \cZ : \barq(z) \in (0, \max\{\ell'(-g_\eps), g_\eps^{-1/2}\}) \})}_{\square}
    +
    \underbrace{\mu(\{z \in \cZ : \barq(z) > \ell'(g_\eps)\})}_{\Diamond}.
  \]
  $\square \to 0$ as $\eps \downarrow 0$, since $g_\eps^{-1/2} \to 0$ as $\eps\downarrow 0$
  and since $\ell'(-r) \to 0$ as $r\to\infty$ by \Cref{fact:loss_prop}.
  Lastly, to show $\Diamond\to 0$, there are two cases.  First, if $\ell'$ grows unboundedly, then $\ell'(g_\eps)$ will
  cover all values as $\eps\downarrow 0$.  On the other hand, if $L := \lim_{r\to\infty}\ell'(r) < \infty$,
  then  $\mu(\{z\in\cZ: \bar q\geq L\}) = 0$ as provided by \Cref{fact:barq:endpoint}
  means once again that $\ell'(g_\eps)$ will cover all values ($\mu$-a.e.) as $\eps\downarrow 0$.

  Lastly, to handle $\star$, we use $\tau(g_\eps)=\min\Set{2\sqrt{\eps},\tau(1)}$
  to obtain that
  \[
   L_\phi\sqrt{2\eps/\tau(g_\eps)}=L_\phi\max\{\sqrt{2\eps/\tau(1)},\eps^{1/4}\}
   \enspace,
  \]
  which goes to zero as $\eps\to 0$.
\end{proofof}

\section{Proofs from \Cref{subsec:sketch:findim:gen}}

As in the main text, this appendix first develops the quantity $\Bal(\mu)$,
and then uses it to develop the deviation bounds.

\subsection{Basic properties of $\Bal(\mu)$}

To start, note the range of values for $\Bal(\mu)$.  When $\Ker(\mu)^\perp \neq \{0\}$, then
$\Bal(\mu) \in [0,\mu(\cZ)]$, but the case $\Ker(\mu)^\perp = \{0\}$ means, via usual conventions
on infima, that $\Bal(\mu) = \infty$.  This represents a certain degeneracy in the learning problem;
indeed, it is a scenario where there is nothing to learn, since equivalently $\Ker(\mu)=\R^d$,
and thus every element of $\R^d$ has no impact on the problem.

With this in mind, the first \namecref{fact:bal} relates boundedness and risk.

\begin{lemma}
  \label[lemma]{fact:bal}
  Let finite measure $\mu$, hypotheses $\cH$, loss $\ell\in\Lclass$, and $\bars\in \partial\ell(0)$ be given.
  Then every $w \in \Ker(\mu)^\perp$ satisfies
  \[
    \|w\|_1 \leq \frac {\cR(w)}{\bars\Bal(\mu)}.
  \]
\end{lemma}
\begin{proof}
  By definition of $\Bal(\mu)$, since $\ell \geq 0$,
  \[
    \cR(w)
    \geq \int_{Aw > 0} \ell(Aw)d\mu
    \geq \int_{Aw > 0} (\ell(0) + \bars(Aw))d\mu
    \geq \|w\|_1 \bars \Bal(\mu),
  \]
  which rearranges to give the result.
  (As a sanity check, the case $\Bal(\mu) = \infty$ means $\Ker(\mu)^\perp = \{0\}$, whereby
  $\|w\|_1 = 0$ automatically.)
\end{proof}

Next, note that the infimand within the definition of $\Bal(\mu)$ is Lipschitz continuous.

\begin{lemma}
  \label[lemma]{fact:bal_lipschitz}
  Let finite measure $\mu$ and hypotheses $\cH$, $|\cH|=d$, be given,
  and define the function $f(w) := \int_{Aw > 0} (Aw) d\mu$.
  Then, for every $w,w'\in\R^d$,
  \[
    |f(w) - f(w')| \leq \|w-w'\|_1\mu(\cZ).
  \]
\end{lemma}
\begin{proof}
  Let $w,w'$ be given, and define $N(v) := \{ z\in \cZ : (Av)(z) > 0\}$.
  Since $|h(z)| \leq 1$ for every $h\in \cH$ (whereby $|(Av)(z)| \leq \|v\|_1$ for every $z$ and $v$),
  \begin{align*}
    |f(w) - f(w')|
    &=\left|
    \int_{N(w)\cap N(w')} A(w-w')d\mu
    + \int_{N(w)\setminus N(w')} (Aw) d\mu
    - \int_{N(w')\setminus N(w)} (Aw') d\mu
    \right|
    \\
    &\leq \|w-w'\|_1 \mu(N(w) \cap N(w'))
    +
    \left|
    \int_{N(w)\setminus N(w')} (Aw) d\mu
    \right|
    + \left|
    \int_{N(w')\setminus N(w)} (Aw') d\mu
    \right|.
  \end{align*}
  Since the second and third terms are symmetric, it suffices to consider the second.
  To this end, note that
  \[
    z \in N(w)\setminus N(w')
    \quad\Longrightarrow\quad (Aw)(z) > 0
  \]
  and
  \[
    z \in N(w)\setminus N(w')
    \quad\Longrightarrow\quad (Aw)(z)
    = (Aw')(z) + (A(w-w'))(z)
    \leq 0 + \|w-w'\|_1,
  \]
  which combine to give
  \[
    z \in N(w)\setminus N(w')
    \quad\Longrightarrow\quad |(Aw)(z)|
    \leq \|w-w'\|_1,
  \]
  and thus
  \begin{align*}
    \left|
    \int_{N(w)\setminus N(w')} (Aw) d\mu
    \right|
    &\leq
    \int_{N(w)\setminus N(w')}
    \left|
    Aw
    \right|d\mu
   %\\
   %&
    \leq
    \|w-w'\|_1 \mu(N(w) \setminus N(w')).
  \end{align*}
  The result follows.
\end{proof}

It will now be shown that $\Bal(\mu_D) > 0$ whenever $\mu_D > 0$.
To prove this, the preceding \namecref{fact:bal_lipschitz} showed that the infimand in the definition of $\Bal(\mu)$
is continuous; on the other hand, since $|\cH|<\infty$,
the domain of the infimum is compact, which together with the aforementioned continuity
gives attainment at a necessarily positive point.

\begin{lemma}
  \label[lemma]{fact:findim:bal_positive}
  Let finite measure $\mu$, hypotheses $\cH$, loss $\ell \in \Lclass$,
  and dual variable $\q \in L_{\beta^*}(\mu)$ with $\q \geq 0$ $\mu$-a.e.\@
  and $A^\top \q = 0$ be given,
  and set $D := \{ z \in \cZ : \q(z) > 0\}$.
  If $\mu(D) > 0$ and $|\cH|<\infty$, then $\Bal(\mu_D) > 0$.
\end{lemma}
\begin{proof}
  If $\Ker(\mu)^\perp = \{0\}$, then $\Bal(\mu)=\infty > 0$ immediately,
  thus suppose $\Ker(\mu)^\perp$ is a nontrivial subspace, meaning in particular
  that there exists $w\in\Ker(\mu)^\perp$ with $\|w\|_1 = 1$.
  By \Cref{fact:bal_lipschitz}, the map $w \mapsto \int_{Aw > 0} (Aw) d\mu_D$ is continuous;
  since moreover the (nonempty) set $C=\{ w \in \Ker(\mu)^\perp : \|w\|_1 = 1\}$ is compact when
  $|\cH|<\infty$, it follows that the minimization in the definition of $\Bal(\mu_D)$ is attained
  at some point in $C$.  The remainder of the proof establishes that the integral is indeed
  positive everywhere on $C$.

  Consider any $w \in C$. Since $C\cap\Ker(\mu_D)=\emptyset$,
  it must hold that $\mu_D(\{z \in \cZ : (Aw)(z) \neq 0 \}) > 0$ (else $w\in\Ker(\mu_D)$),
  and thus at least one of the two expressions
  $\int_{Aw > 0}(Aw)d\mu_D$ and $\int_{Aw < 0}(Aw) d\mu_D$ must be nonzero.
  If the first is nonzero, it is positive, and the proof is done,
  thus suppose that only the second expression is nonzero,
  which necessarily means it is negative.
  Since $A^\top \q = 0$, then
  $\ip{Aw}{\q} = 0$, which can be split into negative and positive parts to yield
  \[
    \int_{Aw > 0}(Aw)\q\,d\mu_D = - \int_{Aw < 0}(Aw)\q\,d\mu_D > 0
  \]
  as desired.
\end{proof}

\Cref{fact:findim:bal_positive} was stated for general $\q\in L_{\beta^*}(\mu)$ due to its use in
future lemmas; however, by instantiating it for a dual optimum $\barq$, it follows that
$\Bal(\mu_{\cD}) > 0$ whenever $\mu(\cD)>0$.

\begin{proposition}
  \label[proposition]{fact:findim:D:0}
  Let $\mu$ be a finite measure, $\cH$ be a hypothesis set with $|\cH|<\infty$,
  and $\ell\in \Lclass$ be a loss with corresponding difficult set $\cD$.
  Then $\Bal(\mu_\cD) > 0$ whenever $\mu(\cD) > 0$.
\end{proposition}
\begin{proof}
  The result follows by applying \Cref{fact:findim:bal_positive} to $\barq$ and $\cD$,
  noting that they satisfy the desired properties by \Cref{fact:duality} and the definition of $\cD$.
\end{proof}

The next two properties will establish the interplay between $\Bal$, $\cD$, and also primal-dual optimal
pairs $(\bar w, \barq)$.

\begin{lemma}
  \label[lemma]{fact:findim:primal_optimum}
  Let finite measure $\mu$, hypotheses $\cH$ with $|\cH|<\infty$, loss $\ell \in \Lclass$, and $\bars\in\partial\ell(0)$ be given.
  If $\Bal(\mu) > 0$,
  then there exists a primal-dual
  optimal pair $(\bar w, \barq)$ to \Eq{duality} which satisfies
  $\bar w \in \Ker(\mu)^\perp$,
  and $\|\bar w\|_1 \leq \ell(0)\mu(\cZ)/ (\bars \Bal(\mu))$,
  and
  $\barq(z) \in\partial\ell((Aw)(z))$ for $\mu$-a.e.\@ $z\in\cZ$.
\end{lemma}
\begin{proof}
  From the definition of $\Ker(\mu)$, it suffices to optimize the primal over
  $\Ker(\mu)^\perp$, and by
  \Cref{fact:loss_prop}, the primal optimization can be further restricted to the compact
  convex set
  \[
    \Set{w\in\Ker(\mu)^\perp:\:\|w\|_1 \leq \ell(0)\mu(\cZ)/ (\bars \Bal(\mu))}
  \enspace,
  \]
  where a minimum $\bar w$ is attained (by continuity of convex functions on
  $\R^d$). The relationship with $\barq$ follows by \Cref{fact:duality}.
\end{proof}

% Rearranging the previous result, we obtain the following:

% \begin{lemma}
%   \label[lemma]{fact:findim:bar_w}
%   Let finite measure $\mu$, hypotheses $\cH$ with $|\cH| < \infty$, loss $\ell \in \Lclass$, and $\bar s\in\partial\ell(0)$ be given.
%   If $\Bal(\mu) > 0$,
%   then there exists a primal-dual
%   optimal pair $(\bar w, \barq)$ to \Eq{duality} which satisfies
%   $\bar w \in \Ker(\mu)^\perp$,
%   and $\|\bar w\|_1 \leq \ell(0)\mu(\cZ) / (\bar s\Bal(\mu_\cD)$,
%   and $\barq(z) \in \partial\ell((Aw)(z))$ for every $z$.
% \end{lemma}

% \begin{proof}
%   This \namecref{fact:findim:bar_w} is mainly a restatement of \Cref{fact:findim:primal_optimum},
%   additionally using the optimality of $\barw$
%   to imply $\cR(\bar w;\mu) \leq \cR(0;\mu)= \ell(0)\mu(\cZ)$,
%   which combined with \Cref{fact:bal} gives
%   \[
%     \|w\|_1
%     \leq \frac {\cR(\bar w;\mu)}{\bars\Bal(\mu)}
%     \leq \frac {\ell(0)\mu(\cZ)}{\bars\Bal(\mu)}
%   \]
%   as desired.
% \end{proof}

The remainder of this subsection will build towards the construction of
the canonical difficult set $\Dcan$:
the difficult sets $\cD$ provided by losses in $\Ldiff$ are ``maximal'' in the measure-theoretic sense.
To this end, the following \namecref{fact:findim:hc_rels:helper} is essential.

\begin{lemma}
  \label[lemma]{fact:findim:hc_rels:helper}
  Let finite measure $\mu$,
  hypotheses $\cH$ with $|\cH|<\infty$,
  loss $\ell \in \Ldiff$,
  and a corresponding difficult set $\cD$ be given.
  For any set $S$ with $\Bal(\mu_S) > 0$,
  then $\mu(S \setminus \cD) = 0$.
\end{lemma}
\begin{proof}
  Suppose contradictorily that $\mu(S \setminus \cD) > 0$,
  which entails $\mu(S) > 0$,
  and let $\barq$ denote the dual optimum associated with $\cD$.

  Applying \Cref{fact:findim:primal_optimum} to loss $\ell$ and measure $\mu_{S}$,
  it follows from $\Bal(\mu_S) > 0$ that there exists a primal optimum $\bar w_S$
  and corresponding dual optimum
  $\barq_S$ with $\barq_S \in \partial\ell(A\bar w_S)$ $\mu_S$-a.e., and consequently $\barq_S >0$
  $\mu$-a.e.\@ since $\ell \in \Ldiff$.

  Define $\hat \q(z) := \barq(z) + \barq_S(z)\one[z\in S]$,
  whereby, for any $w\in \R^d$,
  \[
    \ip{Aw}{\hat \q}
    = \ip{Aw}{\barq}
    + \ip{Aw}{\barq_S\one[z\in S]}
    = \int (Aw)(\barq) d\mu
    + \int_{S} (Aw) \barq_S d\mu
    = 0 + 0.
  \]
  Additionally, $\hat \q \geq 0$ $\mu$-a.e.\@ with $\hat \q > 0$ $\mu$-a.e.\@ along $D := \cD \cup S$,
  thus \Cref{fact:findim:bal_positive} and \Cref{fact:findim:primal_optimum} may be applied to obtain
  a dual optimum $\barq_D$ which is positive $\mu$-a.e.\@ along $D$.  Of course, $\barq$ was feasible for
  the problem restricted to $D$, and by strict convexity of $\int \ell^*(\q)d\mu$ (see \Cref{prop:subgrad:strict}),
  it follows that
  $\int \ell^*(\barq_D)d\mu_{D} < \int \ell^*(\barq)d\mu_{D}$.
  But this is a contradiction, since $z\mapsto \barq_D(z)\one[z\in D]$ is feasible for the full
  problem without changing its objective value, and meanwhile $\barq$ was optimal for the
  full problem.
  %blah:
 %\MJTDEBATE{I think basically the same thing was used elsewhere (maybe in proof of
 %  \Cref{lemma:dual:D}?) and thus can be separated into a lemma... could help readability as
 %well, this proof becomes somewhat confusing...}
\end{proof}

%\begin{proof}[Proof of \Cref{fact:findim:D}]
\begin{proofof}{\Cref{fact:findim:Dcan}}
  We will show that
  if $\ell_1\in \Lclass$ and $\ell_2\in\Ldiff$ with corresponding difficult
  sets $\cD_1$ and $\cD_2$, then $\mu(\cD_1\setminus \cD_2) = 0$,
  which will yield the proof.
  For (i), it suffices to instantiate the claim with $\ell_1 = \ell$ and $\ell_2 = \exp$,
  and (ii) follows by instantiating the claim once with $\ell_1 = \ell$ and $\ell_2=\exp$,
  and a second time with $\ell_1 = \exp$ and $\ell_2 = \ell$.

  The proof of the general claim is as follows.
  If $\mu(\cD_1) = 0$, then $\mu(\cD_1\setminus \cD_2) = 0$ automatically,
  thus suppose $\mu(\cD_1) > 0$.
  In this case, \Cref{fact:findim:bal_positive} grants $\Bal(\mu_{\cD_1}) > 0$,
  and thus applying \Cref{fact:findim:hc_rels:helper} with loss $\ell_2$ and
  $S:=\cD_1$ gives $\mu(\cD_1\setminus \cD_2) = 0$.
%\end{proof}
\end{proofof}

\subsection{Splitting $\hcR_n$ along $\cD$ and $\cD^c$}
\label{sec:findim:hcR_and_D}

As granted by the development of $\Bal(\mu)$, recall from the main text that there exists a canonical difficult set $\Dcan$,
which by \Cref{fact:findim:Dcan} is not tied to any specific loss.
The goal of this section is to show, as stated in \Cref{fact:findim:hc_split},
that $\hcR$ can be split along $\Dcan$, just like $\cR$ (cf. \Cref{cor:hc_split}),
despite $\Dcan$ being constructed over $\mu$ rather than $\hmu$.

As the first step, we establish
the existence of arbitrarily good predictors over $\Dcan^c$.

\begin{lemma}
  \label[lemma]{fact:findim:hc_split:helper}
  Let finite measure $\mu$,
  hypotheses $\cH$ with $|\cH|=d$,
  and canonical difficult set $\Dcan$ be given.
  Then for every $\eps > 0$,
  there exists $v \in \R^d$ such that $(Av)(z) = 0$ for $\mu$-a.e.\@ $z\in \Dcan$,
  and $\mu(\{ z \in \Dcan^c : (Av)(z) \geq -1 \}) \leq \eps$.
\end{lemma}
\begin{proof}
  Throughout this proof, set $\ell := \exp\in\Lbnd$, whereby $\cD = \Dcan$ by definition of $\Dcan$,
  and let $\eps > 0$ be given.

  There are now two cases to consider; first consider the simpler case $\mu(\Dcan) = 0$.
  %XXX wtf why did I think I needed the following.
 %Expanding the definition of $\Dcan$, this means $\mu(\{z \in \cZ : \barq(z) = 0\}) = 0$,
 %thus \Cref{fact:loss_prop} and \Cref{fact:duality} grant
 %\[
 %  0 = -\int \ell^*(\barq)d\mu = \inf_{w\in L^1(\cH)}\cR(w).
 %\]
 %Consequently,
  Choose any $v\in L_1(\mu)$ with $\cE(v) \leq \eps\ell(-1)$,
  and first note that $(Av)(z) = 0$ for $\mu$-a.e.\@ $z\in\Dcan$ without any effort since $\mu(\Dcan) = 0$.
  On the other hand, by \Cref{fact:Dc_controls} (with $r:= \ell(-1)$),
  \begin{align*}
    \mu(\{ z \in \Dcan^c : (Av)(z) \geq -1 \})
    &= \mu(\{ z \in \Dcan^c : \ell((Av)(z)) \geq \ell(-1) \})
   %\\
   %&
    \leq \frac {\cE(w)}{\ell(-1)}
    \leq \eps,
  \end{align*}
  which completes the proof under the assumption $\mu(\Dcan) = 0$.

  Now consider the case $\mu(\Dcan) > 0$,
  whereby \Cref{fact:findim:D:0} grants $\Bal(\mu_{\Dcan}) > 0$.
  Let $\bars\in\partial\ell(0)$ be arbitrary and set
  \[
    B := \frac {1 + \inf_{v\in \R^d} \cR(v;\mu_{\Dcan})}{\bars\Bal(\mu_{\Dcan})},
  \]
  whereby \Cref{fact:bal}
  grants that every $w \in \Ker(\mu_{\Dcan})^\perp$ with $\cE(w, \mu_{\Dcan}) \leq 1$
  satisfies
  \[
    \|w\|_1
    \leq \frac {\cR(w;\mu_{\Dcan})}{\bars \Bal(\mu_{\Dcan})}
    \leq B.
  \]

  Now let $\eps > 0$ be given, and choose $\eps_0 \in (0, \min\{\eps^2,1\}]$ such that
  $\ell^{-1}(\sqrt{\eps_0}) \leq -1 - B$.
  Let $u\in \R^d$ be given
  with $\cE(u) \leq \eps_0$,
  whereby \Cref{cor:hc_split} grants that
  $\max\{
    \cE(u; \mu_{\Dcan})
    ,
    \cR(u; \mu_{\Dcan^c})
  \} \leq \eps_0$ as well.
  By \Cref{fact:Dc_controls} with $r:=\sqrt{\eps_0}$
  and the above definitions,
  \begin{align*}
    \eps
    \geq \sqrt{\eps_0}
    &\geq \mu(\{ z \in \Dcan^c : \ell((Au)(z)) \geq \sqrt{\eps_0}\})
    \\
    &= \mu(\{ z \in \Dcan^c : (Au)(z) \geq \ell^{-1}(\sqrt{\eps_0})\})
    \\
    &\geq  \mu(\{ z \in \Dcan^c : (Au)(z) \geq -1 -B\}).
  \end{align*}
  Now write $u$ as the direct sum $u = v \oplus u_\perp$, where $v \in \Ker(\mu_{\Dcan})$ and
  $u_\perp \in \Ker(\mu_{\Dcan})^\perp$.  By the earlier derivation, $\|u_\perp\|_1 \leq B$,
  and thus, for any $z\in\cZ$, we have $|(Au_\perp)(z)|\leq B$, and
  \[
    (Av)(z) \geq -1
    \quad\Longrightarrow\quad
    (Av)(z) \geq -1 -B - (Au_\perp)(z)
    \quad\iff\quad
    (Au)(z)
    %= (Hv)(z) + (Hu_\perp)
    \geq - 1 -B.
  \]
  This combines with the earlier derivation to yield
  \begin{align*}
    \eps
    &\geq \mu(\{ z \in \Dcan^c : (Au)(z) \geq -1 - B\})
    %\\
    %&
    \geq \mu(\{ z \in \Dcan^c : (Av)(z) \geq -1\})
  \end{align*}
  as desired.
\end{proof}

Thanks to the preceding \namecref{fact:findim:hc_split:helper},
splitting $\hcR_n$ into $\Dcan$ and $\Dcan^c$,
is straightforward (and similar to the proof of \Cref{cor:hc_split}).
%XXX with some further structural properties on \Dcan, could prove this by _invoking_ {cor:hc_split} ?

\begin{lemma}
  \label[lemma]{fact:findim:hc_split}
  Let probability measure $\mu$,
  hypotheses $\cH$ with $|\cH|=d$,
  and canonical difficult set $\Dcan$ be given.
  Then, for any $\delta>0$,
  with probability $1-\delta$ over a random draw of size $n$ from $\mu$,
  every loss $\ell\in\Ldiff$ satisfies
  \[
    \inf_{w\in \R^d} \cR(w;\hmu)
    =
    \inf_{w\in \R^d} \cR (w;\hmu_{\Dcan})
  \]
  and for every $w\in \R^d$
  \[
    \cE(w;\hmu_{\Dcan}) \leq \cE(w;\hmu)
\enspace,
\quad
    \cR(w;\hmu_{\Dcan^c}) \leq \cE(w;\hmu)
\enspace.
  \]
\end{lemma}
\begin{proof}
  Let $\cS$ to denote the sample,
  where $\cS_c := \cS \cap \Dcan^c$ with size $n_c:=|\cS_c|$ denotes the portion falling within $\Dcan^c$,
  and $\cS_\cD := \cS\cap\Dcan$ the portion falling within $\Dcan$.
  If $n_c = 0$, then all claims follow immediately (indeed, this implies $\hmu = \hmu_{\Dcan}$ and $\cR(w;\hmu_{\Dcan^c}) = 0$),
  thus suppose $n_c > 0$.

  By \Cref{fact:findim:hc_split:helper} with $\eps := \mu(\Dcan^c)\eps_0$
  and $\eps_0\coloneqq \min\{ 1/2, -\ln(1-\delta)/(2n_c)\}$,
  %and $S := \Dcan$ (cf. \Cref{fact:findim:D}), %XXX holdover from older version where I allowed that generality?  or thought I needed it?
  there exists $v \in \R^d$ satisfying $(Av)(z) = 0$ for $\mu$-a.e.\@ $z\in \Dcan$, and
  \[
    \mu(\{z\in \Dcan^c : (Av)(z) \geq -1\}) \leq \eps=\mu(\Dcan^c)\eps_0,
  \]
  or equivalently
  \[
    \mu_{|\Dcan^c}(\{z\in \Dcan^c : (Av)(z) \geq -1\}) \le \eps_0
  \enspace.
  \]
  Consequently, with probability at least $1-\delta$ over the draw of $\cS$, conditional on $n_c$, we obtain $(Av)(z_i) = 0$ for
  every $z_i \in \cS_{\cD}$ and $(Av)(z_i) \le -1$ for every $z_i\in \cS_c$, the latter statement since
  \begin{align*}
    \Pr[\forall i \in \cS_c,\, (Av)(z_i) \leq -1]
    &=
    \mu_{|\Dcan^c}(\{z\in \Dcan^c : (Av)(z) \geq -1\})^{n_c}
    \\
    &\geq (1-\eps_0)^{n_c}
    \\
    &\geq (1-(2\eps_0) + (2\eps_0)^2/2)^{n_c}
    \\
    &\geq \exp(-2n_c\eps_0)
    \\
    &\geq 1-\delta.
  \end{align*}
  Since $\lim_{z\to-\infty}\ell(z) = 0$,
  every $w\in \R^d$ and $z_i\in\cS_c$ satisfies
  $\inf_{r\geq 0} \ell((A(w+rv))(z_i)) = 0$, thus
  \begin{align*}
    %\inf_{w\in \R^d} \int \ell(Aw) d\hmu
    \inf_{w\in \R^d} \cR(w;\hmu)
    &=
    \inf_{\substack{w\in \R^d \\ r \geq 0}} \int \ell(A(w +rv)) d\hmu
    \\
    &=
    \inf_{\substack{w\in \R^d \\ r \geq 0}} \left(\int_{\Dcan} \ell(A(w +rv)) d\hmu +  \int_{\Dcan^c} \ell(A(w +rv)) d\hmu\right)
    \\
    &=
    \inf_{w\in \R^d} \left(\int_{\Dcan} \ell(Aw) d\hmu
      +
    \inf_{r \geq 0} \int_{\Dcan^c} \ell(A(w +rv)) d\hmu\right)
    \\
    &=
    %\inf_{w\in \R^d} \int_{\Dcan} \ell(Aw) d\hmu.
    \inf_{w\in \R^d} \cR(w;\hmu_{\Dcan}).
  \end{align*}

%%see comments above in lemma statement; this broken block was an early failed attempt at the statements above,
%%which can be proved as noted there but ended up not being needed anywhere.
% %XXX seemed like I thought I had this worked out but I didn't need it and never came back to it.. deadline hours away...
% \MJTDEBATE{
%   Now suppose contradictorily that that there is a dual optimum $\hat \q$ with $U := \{z\in \cS : \hat \q(z) > 0 \}$
%   satisfying $|U \setminus \cS_{\Dcan}| > 0$.
%   \MJT{blah blah should be easy, didn't do it yet.}
%
%   \MJT{if had $\ell\in\Ldiff$ and thus uniqueness, could use following:
%     Now let $\tilde \q$ denote the dual optimum for the restricted measure $\hmu_{\Dcan}$.
%     The choice $\hat \q$ with $\hat \q(z) := \tilde \q(z)\one[z\in {\Dcan}]$ is feasible
%     for the general problem, since for any $w\in \R^d$
%     \[
%       \int (Aw)\hat \q d\hmu
%       = \int (Aw)\tilde \q d\hmu_{\Dcan}
%       = 0,
%     \]
%     and it preserves dual objective value since $\ell^*(0) = 0$ (cf. \Cref{fact:loss_prop})
%     implies
%     \[
%       \int \ell^*(\hat \q)d\hmu
%       =
%       \int_{\Dcan} \ell^*(\tilde \q)d\hmu
%       +
%       \int_{{\Dcan}^c} \ell^*(0)d\hmu
%       =
%       \int_{\Dcan} \ell^*(\tilde \q)d\hmu.
%     \]
%     Since $\int \ell(\cdot) d\hmu$ is strictly convex, it follows that $\hat \q$ is the unique
%     dual optimum, and it satisfies the desired property
%     $\{ z \in \cS : \hat \q(z) > 0 \} \subseteq \cS_{\Dcan} \subset {\Dcan}$.
%   }
% }

  For the last part,
  proceeding similarly to the proof of \Cref{lemma:dual:D},
  the above derivation and $\ell \geq 0$ grant
  \begin{align*}
    \cE(w;\hmu_{\Dcan}) = \int_{\Dcan}\ell(Aw)d\hmu - \inf_{v\in \R^d} \int_{\Dcan} \ell(Av)d\hmu
    & \leq \int\ell(Aw)d\hmu - \inf_{v\in \R^d} \int \ell(Av)d\hmu
    = \cE(w;\hmu)
  \end{align*}
  directly, and
  \begin{align*}
    \cR(w;\hmu_{\Dcan^c})
    &= \int\ell(Aw)d\hmu - \int_{\Dcan}\ell(Aw)d\hmu
    \\
    &\leq \int\ell(Aw)d\hmu - \inf_{v\in \R^d} \int_{\Dcan}\ell(Av)d\hmu
    \\
    &= \int\ell(Aw)d\hmu - \inf_{v\in \R^d} \int\ell(Av)d\hmu
    = \cE(w;\hmu).
  \end{align*}
\end{proof}

\subsection{Controlling deviations over $\Dcan^c$}

This section will establish the deviation bound over $\Dcan^c$,
namely \Cref{fact:findim:gen:helper:Dc}.
Superficially, this is merely an application of the VC theorem,
however there are two issues under the surface.

First, note that this \namecref{fact:findim:gen:helper:Dc} does not attempt to control
$\cE(w;\mu_{\Dcan^c})$, which of course would allow a direct application of \Cref{fact:Dc_controls}
and ostensibly an easy analysis over $\Dcan^c$ within the proof of \Cref{fact:findim:gen:simplified}.
The reason is that there is evidence $\cE(w;\mu_{\Dcan^c})$ cannot be controlled without
placing strong restrictions on $w\in \R^d$ \citep{colt_logistic_hazan}.
On the other hand, the margin-like bound here is sufficient
to aid in the proof of \Cref{fact:findim:gen:simplified}.

The second issue is that $\Dcan^c$ is an object constructed over $\mu$ rather than $\hmu$,
which is circumvented via \Cref{fact:findim:hc_split}.

\begin{lemma}
  \label[lemma]{fact:findim:gen:helper:Dc}
  Let probability measure $\mu$,
  hypotheses $\cH$ with $|\cH|=d$,
  loss $\ell \in \Lclass$,
  and canonical difficult set $\Dcan$ with $\mu(\Dcan^c) > 0$ be given.
  Then with probability at least $1-2\delta$ over an i.i.d. draw of size $n$ from $\mu_{|\Dcan^c}$,
  every $w\in \R^d$ and $\eps > 0$ with $\eps \geq \sqrt{\hcE_n(w)}$ satisfies
  \[
    \mu_{|\Dcan^c}\left(
      \left\{
        z \in \Dcan^c : \ell(Aw) \geq \eps
      \right\}
    \right)
    \leq
    \frac{\sqrt{\hcE_n(w)}}{\hmu(\Dcan^c)}
    + \sqrt{\frac {32(1+d)\ln(1 + n) + 4\ln(1/\delta)}{n}}.
  \]
\end{lemma}

\begin{proof}
%  First note that every $\ell\in\Lclass$ is invertible over $(0,\infty)$,
%  thus $\ell^{-1}(s)$ is (unique and) well-defined whenever $s>0$.
  Since the set of linear threshold functions with weight vectors in $\R^d$ has VC dimension $1+d$,
  the nondecreasing property of $\ell$ combined with the VC theorem grants \citep{bbl_esaim}, with probability $1-\delta$,
  \begin{align*}
    &\sup_{\substack{w\in \R^d \\ s > 0}}
    \left|
    \mu_{|\Dcan^c}\left(
      \left\{
        z \in \Dcan^c : \ell(Aw) \geq s
      \right\}
    \right)
    -
    \hmu_{|\Dcan^c}\left(
      \left\{
        z \in \Dcan^c : \ell(Aw) \geq s
      \right\}
    \right)
    \right|
    \\
    &\qquad\qquad\qquad
    \leq
    \sup_{\substack{w\in \R^d \\ r\in \R}}
    \left|
    \mu_{|\Dcan^c}\left(
      \left\{
        z \in \Dcan^c : (Aw) \geq r
      \right\}
    \right)
    -
    \hmu_{|\Dcan^c}\left(
      \left\{
        z \in \Dcan^c : (Aw) \geq r
      \right\}
    \right)
    \right|
    \\
    &\qquad\qquad\qquad
    \leq \sqrt{\frac {16(1+d)\ln(1 + n) + 2\ln(1/\delta)}{n}}.
  \end{align*}
  %XXX Using {fact:loss_markov} and not {fact:Dc_controls} since \Dcan has not been proved to be a superset of difficult set over \hmu
  Now let $w\in \R^d$ be arbitrary.
  Instantiating the above display with $w$ and $s:=\eps > 0$,
  and then applying \Cref{fact:loss_markov} on measure $\hmu_{|\Dcan^c}$ and set $\Dcan^c$ with scalar $r:= \eps > 0$,
  it follows that
  \[
    \mu_{|\Dcan^c}\left(
      \left\{
        z \in \Dcan^c : \ell(Aw) \geq \eps
      \right\}
    \right)
    \leq
    \frac {\cR(w;\hmu_{|\Dcan^c})}{\eps}
    + \sqrt{\frac {16(1+d)\ln(1 + n) + 2\ln(1/\delta)}{n}}.
  \]
  To finish,
  \Cref{fact:findim:hc_split} grants
  $\cR(w;\hmu_{|\Dcan^c}) = \cR(w;\hmu_{\Dcan^c})/\hmu(\Dcan^c) \leq \hcE_n(w) / \hmu(\Dcan^c)$
  after discarding another $\delta$ failure probability,
  and the result follows by plugging in $\eps \geq \sqrt{\hcE_n(w)}$.
  %XXX wait uh was there a way to use margin bounds here?
\end{proof}

\subsection{Controlling deviations over $\Dcan$: Proof of \Cref{fact:findim:gen:helper:D}}
%\subsection{Proof of \Cref{fact:findim:gen:helper:D}}

In order to establish \Cref{fact:findim:gen:helper:D}, two lemmas are in order:
the first shows that $\Bal(\mu)$ is statistically stable,
and the second develops a refined deviation bound for $\Lbnd$ over $\Dcan$.

\begin{lemma}
  \label[lemma]{fact:findim:bal_stable}
  Let probability measure $\mu$,
  hypotheses $\cH$ with $|\cH|=d$,
  and canonical difficult set $\Dcan$ with $\mu(\Dcan)>0$ be given.
  Then with probability at least $1-\delta$ over a draw of size $n$ from $\mu_{|\Dcan}$,
  $\Ker(\mu_{|\Dcan}) \subseteq \Ker(\hmu_{|\Dcan})$, and
  \[
    \Bal(\hmu_{|\Dcan})
    \geq
    \Bal(\mu_{|\Dcan}) - 8\sqrt{\frac{\ln(2d) + \ln(4/\delta)}{n}}.
  \]
  Moreover, if $n \geq 256 (\ln(2d + \ln(4/\delta)) / \Bal(\mu_{|\Dcan})^2$,
  then
  $\Bal(\hmu_{|\Dcan}) \geq \Bal(\mu_{|\Dcan})/2$
  and
  $\Ker(\mu_{|\Dcan}) = \Ker(\hmu_{|\Dcan})$.
\end{lemma}
\begin{proof}
  Let $\cS := (z_i)_{i=1}^n$ denote the random draw from $\mu_{|\Dcan}$.

  First it will be shown that $\Ker(\mu_{|\Dcan}) \subseteq \Ker(\hmu_{|\Dcan})$ with probability 1.
  If $\Ker(\mu_{|\Dcan}) = \{0\}$, then the claim is immediate, thus suppose $\Ker(\mu_{|\Dcan})$
  is a nontrivial subspace of $\R^d$.
  Pick an orthonormal basis $(w_j)_{j=1}^k$ for $\Ker(\mu_{|\Dcan})$.
  For each $w_j$, define $N_j := \{z\in\cZ : (Aw_j)(z) \neq 0\}$,
  whereby $\mu_{|\Dcan}(N_j) = 0$ since $w_j \in \Ker(\mu_{|\Dcan})$.
  Since $\mu(\cup_{j=1}^kN_j) = 0$, then with probability 1 over the draw of sample $\cS$,
  every $z_i\in \cS$ and $w_j$ satisfy $(Aw_j)(z_i) = 0$.
  Consequently, given an arbitrary $w\in\Ker(\mu_{|\Dcan})$,
  there exist scalars $(\alpha_j)_{j=1}^k$ such that $w = \sum_{j=1}^k \alpha_j w_j$,
  and thus, for every $z_i \in \cS$, by linearity
  \[
    (Aw)(z_i)
    = \sum_{j=1}^k \alpha_j (Aw_j)(z_i)
    = 0,
  \]
  meaning $w \in \Ker(\hmu_{|\Dcan})$ as well.
  Hence,
  $\Ker(\mu_{|\Dcan}) \subseteq \Ker(\hmu_{|\Dcan})$.

  Throughout the remainder of this proof, discard the failure event for the above control on $\Ker(\hmu_{|\Dcan})$:
  in particular, suppose $\Ker(\mu_{|\Dcan}) \subseteq \Ker(\hmu_{|\Dcan})$,
  and equivalently $\Ker(\mu_{|\Dcan})^\perp \supseteq \Ker(\hmu_{|\Dcan})^\perp$.

  In order to produce the lower bound on $\Bal(\hmu_{|\Dcan})$,
  first consider the case $\Bal(\mu_{|\Dcan}) = \infty$.
  This means $\Ker(\mu_{|\Dcan})^\perp = \{0\}$,
  which combined with $\Ker(\mu_{|\Dcan})^\perp \supseteq \Ker(\hmu_{|\Dcan})^\perp$
  means $\Bal(\hmu_{|\Dcan}) = \infty$ as well,
  giving the desired bound.

  Now consider the case $\Bal(\mu_{|\Dcan}) < \infty$.
  First note that the map $z \mapsto \max\{(Aw)(z), 0\}$ is the composition
  of the 1-Lipschitz univariate map $\max\{\cdot,0\}$ together with a linear function,
  so by \Cref{fact:rad:prop}, it has Rademacher complexity
  $\|w\|_1\sqrt{2\ln(2d) / n}$ since $\abs{-yh(x)}\le 1$ for all $(x,y)$.
  Combining this with standard deviation bounds for Rademacher complexity
  (\Cref{fact:rad:dev}),
  with probability $1-\delta$,% \MJT{need to applied it with thingy flipped to get abs value} %no, signs agree...
  \begin{align}
    \sup_{\|w\|\leq 1} \left(
    %\int_{\{(x,y) : \ip{w}{xy} > 0\}} \ip{w}{xy}d\mu(x,y)
    \int\max\{ Aw, 0\}d\mu_{|{\Dcan}}
    -
    %\int_{\{(x,y) : \ip{w}{xy} > 0\}} \ip{w}{xy}d\hmu(x,y)
    \int\max\{ Aw, 0\}d\hmu_{|{\Dcan}}
    \right)
    &\leq 8 \sqrt{\frac{\ln(2d) + \ln(4/\delta)}{n}}.
    \label{eq:gen:H:1}
  \end{align}
  Combining \Eq{gen:H:1} with $\Ker(\mu_{|\Dcan})^\perp \supseteq \Ker(\hmu_{|\Dcan})^\perp$,
  \begin{align*}
    \Bal(\hmu_{|{\Dcan}})
    &= \inf\left\{\int \max\{Aw, 0\} d\hmu_{|{\Dcan}} : \|w\|_1 = 1, w \in\Ker(\hmu_{|{\Dcan}})^\perp\right\}
    \\
    &\geq \inf\left\{\int \max\{Aw, 0\} d\mu_{|{\Dcan}} : \|w\|_1 = 1, w \in\Ker(\hmu_{|{\Dcan}})^\perp\right\}
    - 8\sqrt{\frac{\ln(2d) + \ln(4/\delta)}{n}}
    \\
    &\geq \inf\left\{\int \max\{Aw, 0\} d\mu_{|{\Dcan}} : \|w\|_1 = 1, w \in\Ker(\mu_{|{\Dcan}})^\perp\right\}
    - 8\sqrt{\frac{\ln(2d) + \ln(4/\delta)}{n}}
    \\
    &= \Bal(\mu_{|{\Dcan}}) - 8\sqrt{\frac{\ln(2d) + \ln(4/\delta)}{n}}.
  \end{align*}

  For the last statements,
  suppose the provided lower bound on $n$;
  this immediately grants the bound $\Bal(\hmu_{|\Dcan})\geq \Bal(\mu_{|\Dcan})/2$ by the above derivation.
  To show $\Ker(\mu_{|\Dcan}) = \Ker(\hmu_{|\Dcan})$,
  it suffices, by the above, to show
  $\Ker(\mu_{|\Dcan}) \supseteq \Ker(\hmu_{|\Dcan})$.
  If $\Ker(\mu_{|\Dcan})^\perp = \{0\}$,
  then $\Ker(\mu_{|\Dcan})^\perp \subseteq \Ker(\hmu_{|\Dcan})^\perp$ immediately since the latter is a subspace,
  thus suppose $\Ker(\mu_{|\Dcan})^\perp$ is nontrivial.
  Let $w\in \Ker(\mu_{|\Dcan})^\perp$ with $\|w\|_1 > 0$ be arbitrary;
  by the definition of $\Bal(\mu_{|\Dcan})$,
  the lower bound on~$n$,
  the deviation bound from \Eq{gen:H:1},
  and since $\Bal(\mu_{|\Dcan}) > 0$ because $\mu(\Dcan)>0$ (see \Cref{fact:findim:D:0}), \begin{align*}
    \int \max\{ Aw, 0\}d\hmu_{|\Dcan}
    &=
    \|w\|_1\left(\int \max\{ Aw/ \|w\|_1, 0\}d\hmu_{|\Dcan}\right)
    \\
    &\geq
    \|w\|_1\left(\int \max\{ Aw/ \|w\|_1, 0\}d\mu_{|\Dcan}
    - 8 \sqrt{\frac{\ln(2d) + \ln(4/\delta)}{n}}\right)
    \\
    &\geq
    \|w\|_1\left(\Bal(\mu_{|\Dcan})
    - 8 \sqrt{\frac{\ln(2d) + \ln(4/\delta)}{n}}\right)
    \\
    &\geq \|w\|_1\Bal(\mu_{|\Dcan})/2 > 0.
  \end{align*}
  Since $\int\max\{Aw,0\} d\hmu_{|\Dcan} > 0$, there must exist $z_i\in\cS$ with $(Aw)(z_i)>0$, and in particular $w\not\in \Ker(\hmu_{|\Dcan})$.
  To see how this gives the result,
  suppose contradictorily that $\Ker(\mu_{|\Dcan}) \subsetneq \Ker(\hmu_{|\Dcan})$,
  whereby there must exist $w \in \Ker(\hmu_{|\Dcan}) \cap \Ker(\mu_{|\Dcan})^\perp$ with $w\neq 0$.
  But the above analysis showed that every $w\in \Ker(\mu_{|\Dcan})^\perp$ with $\|w\|_1 > 0$ has $w\not \in\Ker(\hmu_{|\Dcan})$,
  a contradiction.
  % this freaked me out so I double checked.
  % given subspaces L1 \subseteq L2 and any w \in L2 \setminus L1,
  % look at the orhogonal projections w = w1 + w2 where w1 is projection onto L1 and w2 is the remainder.
  % since L2 is a subpsace, then w - w1 \in L2, meaning w2 in L2.
  % But w2 \not \in L1 since then w1+w2=w \in L1, thus L2 contains vectors orthogonal to L1.
  % not sure what combination of tiredness and panic caused me to check this.
\end{proof}

Next lemma
is a refined analysis of deviations over $\Dcan$
under the assumption $\ell\in\Lbnd$.
In particular, $\Bal(\mu_{|\Dcan})$ will be used to establish
strong convexity of $\cR(\cdot;\mu_{|\Dcan})$ around $\bar w$.
The core of the convergence rate argument itself follows almost identically a proof by
\citet[Theorem 1]{karthik_sc_fastrates}, with two important differences that necessitated
a careful reproof.
\begin{itemize}
  \item
    Rather than controlling a function which is strongly convex everywhere thanks to a regularizer,
    it is instead only used that $\cR(\cdot;\mu_{|\Dcan})$ is inherently strongly convex around the optimum without any regularization.
    %XXX similar to least squares I guess, but too lazy to look up.

  \item
    This strong convexity around the optimum is only established
    over $\mu_{|\Dcan}$, and in particular not over $\hmu_{|\Dcan}$.
    Of course, since $\Bal(\mu_{|\Dcan})$ is statistically stable (see \Cref{fact:findim:bal_stable}),
    the same proof shows that $\cR(\cdot;\hmu_{|\Dcan})$ is also strongly convex along $\Dcan$,
    but it is interesting and pleasant that the proof works directly without establishing this.
\end{itemize}

\begin{lemma}
  \label[lemma]{fact:findim:sc}
  Let probability measure $\mu$ over $\cZ$,
  hypotheses $\cH$ with $|\cH| < \infty$,
  loss function $\ell \in \Lbnd$,
  and canonical difficult set $\Dcan$ with $\mu(\Dcan) > 0$ be given.
  Let a primal-dual optimal pair $(\bar w,\barq)$ for \Eq{duality}
  with measure $\mu_{|{\Dcan}}$ be given with $\bar w \in \Ker(\mu_{|{\Dcan}})^\perp$.
  Lastly, let $B\geq \|\bar w\|_1$ be given, and set
  $\cW := \{ w \in \Ker(\mu_{|{\Dcan}})^\perp : \|w\|_1 \leq B \}$.
  The following statements hold.
  \begin{enumerate}
   %\item
   %  Define
   %  \[
   %    \sigma_{\min}(\mu_{|{\Dcan}},\cH)
   %    := \inf\left\{ \int |Av|d\mu_{|{\Dcan}} : v\in \Ker(\mu_{|{\Dcan}})^\perp, \|v\|_1 = 1\right\}.
   %  \]
   %  By the assumptions above, $\sigma_{\min}(\mu_{|{\Dcan}},\cH) > 0$.
    \item
      Set $\tau := \inf_{|z|\leq B} \ell''(z)$
      and $\lambda := \tau \Bal(\mu_{|{\Dcan}})^2$,
      where $\tau > 0$ since $\ell \in \Lbnd$.
      Then, for every $w\in\cW$,
      \begin{equation}
        \cE(w;\mu_{|{\Dcan}})
        \geq \frac \lambda 2 \|w-\bar w\|_1^2.
        \label{eq:findim:sc:0}
      \end{equation}

    \item
      Let a draw $\cS$ from $\mu_{|{\Dcan}}$ of size $n$ be given.
      Then, with probability at least $1-\delta$, every $w\in\cW$ satisfies
      \[
        \cE(w;\mu_{|{\Dcan}})
        \leq
        2\cE(w;\hmu_{|{\Dcan}})
        +
        \frac {1024 \ell'(2B)^2 (\ln(2d) + \ln(4/\delta))}{\lambda n}.
      \]
  \end{enumerate}
\end{lemma}
\begin{proof}
  %The various statements are proved as follows; each statement will rely upon those coming earlier.
  \begin{enumerate}
   %\item
   %  Since $A$ is a (bounded) linear operator and $\mu_{|{\Dcan}}$ is a finite measure,
   %  the map $v\mapsto \int |Av|d\mu_{|{\Dcan}}$ is continuous
   %  (just like the infimand in $\Bal(\mu)$, cf. \Cref{fact:bal_lipschitz}).
   %  Moreover, the set $\left\{ v\in \Ker(\mu_{|{\Dcan}})^\perp : \|v\|_1 = 1\right\}$ is compact
   %  since $\R^d$ is finite dimensional, $\Ker(\mu_{|{\Dcan}})^\perp$ is a (closed) subspace,
   %  and the norm ball is compact.  Consequently, the infimum in the definition of
   %  $\sigma_{\min}(\mu_{|{\Dcan}},\cH)$ is of a continuous function over a compact set, and
   %  thus attained at some feasible $v'$.  But the infimand is positive at every feasible
   %  point, thus it is positive at the infimum $v'$.

    \item
      To start,
      applying Taylor's theorem pointwise,
      every $w\in\cW$ satisfies
      \begin{align*}
        \cE(w;\mu_{|{\Dcan}})
        &=
        \int \ell(Aw)d\mu_{|{\Dcan}}
        - \int \ell(A\bar w)d\mu_{|{\Dcan}}
        \\
        &\geq \underbrace{\int \ell'(A\bar w)(Aw - A\bar w)d\mu_{|{\Dcan}}}_{\heartsuit}
        + \underbrace{\frac 1 2 \int \tau (Aw - A\bar w)^2d\mu_{|{\Dcan}}}_{\triangle}.
      \end{align*}
      To manage $\heartsuit$, %TO MANAGE YOUR HEART
      since $\barq = \ell'(A\bar w)$ $\mu$-a.e.\@ (by \Cref{fact:duality}),
      and since $\barq$ is dual feasible (whereby $A^\top \barq = 0$),
      then
      \[
        \heartsuit = \int \ell'(A\bar w)(Aw - A\bar w)d\mu_{|{\Dcan}} = \int \barq (A(w - \bar w))d\mu_{|{\Dcan}} = 0.
      \]
      %XXX lulz I started trying to figure out how to estabilsh first order optimiality in the smallest number of steps
      %XXX and just did the above instead.. lots of power encoded in that pair..
      For the second term $\triangle$, by Jensen's inequality and the definition of $\Bal(\mu_{|{\Dcan}})$,
      \begin{align*}
        \triangle
        &\geq \frac \tau 2 \left( \int |A(w - \bar w)| d\mu_{|{\Dcan}} \right)^2
        \\
        &\geq \frac \tau 2 \left( \int \max\{ A(w - \bar w), 0\} d\mu_{|{\Dcan}} \right)^2
        \\
        &\geq \frac {\tau \Bal(\mu_{|{\Dcan}})^2} 2 \|w - \bar w\|_1^2
      \enspace,
      \end{align*}
      which gives the bound.

    \item
      As discussed above, this proof follows one due to \citet[Proof of Theorem 1]{karthik_sc_fastrates}.

      To start, when $\nu \in \{\mu_{|{\Dcan}}, \hmu_{|{\Dcan}}\}$ and $\ell\in\Lbnd$,
      then $\cR(\cdot;\nu)$ is $L := \ell'(2B)$ Lipschitz;
      additionally, it satisfies \Eq{findim:sc:0},
      meaning $\cR(\cdot;\nu)$ is $\lambda$-strongly-convex around $\bar w$ as above.

      Let $r > 0$ be a constant to be optimized at the end of the proof, and define
      \begin{align*}
        k_w &:= \min\left\{k \in \Z_+ : \cE(w;\mu_{|{\Dcan}}) \leq r 4^k\right\},
        \\
        f_w(z) &:= \ell\bigParens{(Aw)(z)} - \ell\bigParens{(A\bar w)(z)},
        \\
        g_w(z) &:= 4^{-k_w}f_w(z),
        \\
        \cG &:= \left\{
          g_w
          :
          w\in \cW
        \right\}.
      \end{align*}
      Applying \Cref{fact:rad:dev} to $\cG$,
      %XXX \MJT{can simplify via the new second form}
      %XXX ->actually I checked briefly and things looked messy and weird so I dropped it.
      then with probability at least $1-\delta$, each $w\in\cW$ satisfies
      \begin{equation}
        \int g_wd\mu_{|{\Dcan}}
        \leq
        \int g_wd\hmu_{|{\Dcan}}
        +
        2\underbrace{\fR(\cG)}_{\spadesuit} + 4\underbrace{\sup_{w\in \cW,z\in\cZ}|g_w(z)|}_{\diamondsuit} \sqrt{\frac{2\ln(4/\delta)}{n}}.
        \label{eq:findim:sc:1}
      \end{equation}
      Following the proof scheme of \citet[Proof of Theorem 1]{karthik_sc_fastrates}, the two critical terms are bounded as follows.
      \begin{itemize}
        \item
          First, $\diamondsuit = \sup_{w\in \cW, z\in\cZ}|g_w(z)| \leq L \sqrt{2r / \lambda}$
          as follows.
          For any $w\in\cW$ and any $z\in \cZ$,
          by the fact that $\cR(\cdot ; \mu_{|{\Dcan}})$ is $L$-Lipschitz and satisfies \Eq{findim:sc:0},
          since $\cE(w; \mu_{|{\Dcan}}) \leq  r4^{k_w}$ by definition of $k_w\geq 0$,
          \begin{align*}
            |g_w(z)|
            &= 4^{-k_w} \BigAbs{\ell\bigParens{(Aw)(z)} - \ell\bigParens{(A\bar w)(z)}}
            \\
            &\leq 4^{-k_w} L \|w - \bar w\|_1
            \\
            &\leq 4^{-k_w}L \sqrt{2\cE(w; \mu_{|{\Dcan}})/\lambda}
            \\
            &\leq 4^{-k_w}L \sqrt{2r4^{k_w}/\lambda}
            \\
            &= 4^{-k_w/2} L \sqrt{2r/\lambda}
            \\
            &\leq L \sqrt{2r/\lambda}
          \end{align*}
          as desired.

        \item
          Second, $\spadesuit=\fR(\cG) \leq 4 L\sqrt{r\ln(2d)/(\lambda n)}$.
          For this, first define two helper classes
          \begin{align*}
            \cF(a) &:= \left\{
              f_w
              :
              w \in \cW, \cE(w; \mu_{|{\Dcan}}) \leq a
            \right\},
            \\
            \tilde \cF(a) &:= \left\{
              f_w
              :
              w \in \cW, \|w- \bar w\|_1 \leq \sqrt{2a/\lambda}
            \right\}.
          \end{align*}
          By \Eq{findim:sc:0},
          $f_w \in \cF(a)$ implies $\|w-\bar w\|_1\leq \sqrt{2a/\lambda}$,
          thus $\cF(a) \subseteq \tilde \cF(a)$.
          By various properties of Rademacher complexity from \Cref{fact:rad:prop},
         %By definition of Rademacher complexity \citep[Chaper 26]{shai_shai_book},
         %and moreover of Rademacher complexity for Lipschitz losses
         %\citep[Lemma 26.9]{shai_shai_book}
         %and for linear classes
         %\citep[Theorem 1, specifically eq. (5)]{rad_paper},
          \begin{align*}
            \fR(\cF(a))
            &\leq \fR(\tilde \cF(a))
            \\
            &\leq L\fR(\{z \mapsto (Aw)(z) : w\in \Ker(\mu_{|{\Dcan}})^\perp, \|w-\bar w\|_1 \leq \sqrt{2a/\lambda}\})
            \\
            &\leq L\fR(\{z \mapsto (Aw)(z) : w\in \Ker(\mu_{|{\Dcan}})^\perp, \|w\|_1 \leq \sqrt{2a/\lambda}\})
            \\
            &\leq L \sqrt{\frac {4a\ln(2d)}{\lambda n}}.
          \end{align*}
          To control $\fR(\cG)$ first note that $0\in \cG$ and $0 \in \cF(a)$ for any $a\geq 0$,
          since these sets all consider the choice $\bar w\in\cW$.  Consequently, \Cref{fact:rad:prop}.ii may
          be applied, which together with \Cref{fact:rad:prop}.i yields
          \begin{align*}
            \fR(\cG)
            &\le \fR\left( \cup_{k=0}^\infty 4^{-k} \cF(r4^k) \right)
            \leq \sum_{k=0}^\infty 4^{-k}\fR\left( \cF(r4^k) \right).
          \end{align*}
          This completes the bound on $\fR(\cG)$, since the above estimates grant
          \begin{align*}
            \sum_{k=0}^\infty 4^{-k}\fR\left( \cF(r4^k) \right)
            \leq L\sqrt{\frac{4r\ln(2d)}{\lambda n}} \sum_{k=0}^\infty 4^{-k/2}
            = 4L\sqrt{\frac{r\ln(2d)}{\lambda n}}.
          \end{align*}
      \end{itemize}

      Continuing with the deviation bound in \Eq{findim:sc:1},
      set $r$ with foresight as
      %XXX I played with a few things, doesn't seem to matter
      \[
       %r := 64L^2\left(
       %    \sqrt{\frac {\ln(2d)}{\lambda n}} + \sqrt{\frac{\ln(4/\delta)}{\lambda n}}
       %\right)\right)^2.
       %\qquad
        r := 8192L^2\left(
            \frac {\ln(2d) + \ln(4/\delta)}{\lambda n}
        \right).
      \]
      Now combining 
      the preceding inequalities on $\diamondsuit$ and $\spadesuit$,
      the choice of $r$,
      and the general inequality $\sqrt a + \sqrt b \leq \sqrt{2(a+b)}$ over nonnegative reals,
      it follows for every $w\in\cW$ that
      \begin{align*}
        \cE(w;\mu_{|\Dcan})
        - \cE(w;\hmu_{|\Dcan})
        &=
        \cR(w;\mu_{|\Dcan})
        - \cR(\barw;\mu_{|\Dcan})
        -\left( \cR(w;\hmu_{|\Dcan})
        - \inf_{v\in \R^d} \cR(w;\hmu_{|\Dcan})\right)
          \\
        &\leq
        \cR(w;\mu_{|\Dcan})
        - \cR(\barw;\mu_{|\Dcan})
        -\left( \cR(w;\hmu_{|\Dcan})
          - \cR(\barw;\hmu_{|\Dcan})
      \right)
        \\
        &=
        \int f_wd\mu_{|{\Dcan}}
        -\int f_wd\hmu_{|{\Dcan}}
        \\
        &=
        4^{k_w}\left( \int g_wd\mu_{|{\Dcan}}
        -\int g_wd\hmu_{|{\Dcan}} \right)
        \\
        &\leq
        4^{k_w}\left(
        8L \sqrt{ \frac{r\ln(2d)}{\lambda n} }
        + 4L \sqrt{ \frac{4r\ln(4/\delta)}{\lambda n} }
        \right)
        \\
        &\leq
        4^{k_w}
        \cdot \sqrt{r}
        \cdot 8L\sqrt{2}
        \cdot \sqrt{ \frac{\ln(2d) + \ln(4/\delta)} {\lambda n} }
        \\
       %&=
       %4^{k_w}
       %\cdot
       %\sqrt{r}
       %\cdot 8L
       %\cdot \left(
       %\sqrt{ \frac{\ln(2d)}{\lambda n} }
       %+ \sqrt{ \frac{\ln(4/\delta)}{\lambda n} }
       %\right)
       %\\
        &=
        \frac {r4^{k_w}}{8}.
      \end{align*}
      To finish the proof, consider two cases for the value of $k_w$: either $k_w = 0$, or $k_w > 0$.
      When $k_w = 0$, then the choice of $r$
     %combined with the general inequality $(a+b)^2 \leq 2(a^2+b^2)$
      gives
      \[
        \cE(w;\mu_{|{\Dcan}})
        - \cE(w;\hmu_{|{\Dcan}})
        \leq \frac{r4^0}{8}
        = \frac {1024 L^2 (\ln(2d) + \ln(4/\delta))}{\lambda n},
      \]
      which yields the desired inequality since $L \leq \ell'(B)$ and by adding $\cE(w;\hmu_{|{\Dcan}}) \geq 0$ to the right hand side.
      On the other hand,
      when $k_w > 0$,
      then the definition of $k_w$ implies $\cE(w;\mu_{|{\Dcan}}) > 4^{k_w-1}r$,
      which plugged back into the above gives
      \[
        \cE(w;\mu_{|{\Dcan}})
        - \cE(w;\hmu_{|{\Dcan}})
        \leq
        \frac {r4^{k_w}}{8}
        \leq
        \frac {r4^{k_w-1}}{2}
        <
        \frac 1 2
        \cE(w;\mu_{|{\Dcan}}),
      \]
      meaning $\cE(w;\mu_{|{\Dcan}}) \leq 2\cE(w;\hmu_{|{\Dcan}})$, giving the desired bound.
  \end{enumerate}
\end{proof}

%Combining all these pieces, the proof of \Cref{fact:findim:gen:helper:D} follows.

%\begin{proof}[Proof of \Cref{fact:findim:gen:helper}]
\begin{proofof}{\Cref{fact:findim:gen:helper:D}}
  This proof will be focused on parts (ii) and (iii);
  to start, note the following two supporting results,
  the second of which will establish part (i) of the desired statement along the way.
  \begin{itemize}
    \item
      First note how $\inf_{w\in \R^d} \cR(w;\nu)$ can be related for $\nu\in \{\mu_{|\Dcan}, \hmu_{|\Dcan}\}$.
      By \Cref{fact:findim:D:0} and the assumption $\mu(\Dcan) > 0$,
      $\Bal(\mu_{|\Dcan}) > 0$,
      and thus \Cref{fact:findim:primal_optimum} gives a primal optimum $\barw \in \Ker(\mu_{|{\Dcan}})^\perp$
      with
      $\|\bar w\|_1 \leq \ell(0)/(\bars \Bal(\mu_{|{\Dcan}}))$.
      Consequently, by Hoeffding's inequality applied to a random variable with range
      $\ell(\|\bar w\|_1)$, with probability at least $1-\delta$,
      \begin{align*}
        \inf_{v \in \R^d} \cR(v;d\hmu_{|\Dcan})
        &\leq \cR(\bar w;d\hmu_{|\Dcan})
        \leq \cR(\bar w;d\mu_{|\Dcan}) + \ell(\|\bar w\|_1)\sqrt{\frac {2\ln(1/\delta)}{n}}.
        \\
        &=
        \inf_{v \in \R^d} \cR(v;d\mu_{|\Dcan}) + \ell(\|\bar w\|_1)\sqrt{\frac {2\ln(1/\delta)}{n}},
      \end{align*}
      which will be useful via the rearrangement
      \begin{equation}
        - \inf_{v \in \R^d} \cR(v;d\mu_{|\Dcan})
        \leq
        - \inf_{v \in \R^d} \cR(v;d\hmu_{|\Dcan})
        + \ell(\|\bar w\|_1)\sqrt{\frac {2\ln(1/\delta)}{n}}.
        \label{eq:sotired:23}
      \end{equation}
    \item
      Secondly, assume the final consequence of \Cref{fact:findim:bal_stable} holds,
      discarding along the way another failure event having probability at most $\delta$:
      by the lower bound on $n$,
      $\Ker(\mu_{|\Dcan}) = \Ker(\hmu_{|\Dcan})$ and
      \begin{align}
        2\Bal(\hmu_{|{\Dcan}}) \geq \Bal(\mu_{|{\Dcan}}).
        \label{eq:sotired:24}
      \end{align}

      To see the value of $\Ker(\mu_{|\Dcan}) = \Ker(\hmu_{|\Dcan})$,
      given any $w \in \R^d$,
      henceforth write $w = w_0 \oplus w_\perp$ with $w_0 \in \Ker(\mu_{|\Dcan})$ and $w_\perp \in \Ker(\mu_{|\Dcan})^\perp$,
      where additionally $w_0 \in \Ker(\hmu_{|\Dcan})$ and $w_\perp \in \Ker(\hmu_{|\Dcan})^\perp$.
      As a first consequence,
      \begin{equation}
        (Aw)(z) = (Aw_0)(z) + (Aw_\perp)(z) = (Aw_\perp)(z)
        \quad
        \textup{for $\mu$-a.e.\@ and $\hmu$-a.e.\@ $z\in\Dcan$}.
        \label{eq:sotired:200}
      \end{equation}
      which further implies
      \begin{equation}
        \cR(w_\perp;\nu) = \cR(w;\nu)
        \qquad\textup{and}\qquad
        \cE(w_\perp;\nu) = \cE(w;\nu)
        \qquad
        \textup{for $\nu \in \{\mu_{|\Dcan},\hmu_{|\Dcan}\}$.}
        \label{eq:sotired:20000}
      \end{equation}
      %and analogous statements for $\cE(\cdot;\mu_{\Dcan})$ and $\cE(\cdot;\hmu_{\Dcan})$.
      Secondly, by \Cref{fact:bal} applied to measure $\hmu_{|\Dcan}$ (where \Cref{fact:bal} requires $w_\perp \in \Ker(\hmu_{|\Dcan})^\perp$),
      and also using \Eq{sotired:24}, \Eq{sotired:20000}, and the form of $B_w$,
      \begin{equation}
        \|w_\perp\|_1
        \leq \frac {\cR(w_\perp;\hmu_{|\Dcan})}{\bars\Bal(\hmu_{|{\Dcan}})}
        \leq \frac {2 \cR(w;\hmu_{|\Dcan})}{\bars\Bal(\mu_{|{\Dcan}})}
        \leq B_w.
        \label{eq:sotired:2000}
      \end{equation}
      This last inequality is essential as it allows $\|w_\perp\|_1$ and $\cR(w;\hmu_{|\Dcan})$ to be
      related, the latter being a purely sample-dependent quantity.
      In particular, part (i) follows immediately by combining \Eq{sotired:200} and \Eq{sotired:2000};
      that is, for $\mu$-a.e.\@ $z\in \Dcan$ and $\hmu$-a.e.\@ $z\in\Dcan$,
      $|(Aw)(z)| = |(Aw_\perp)(z)| \leq \|w_\perp\|_1 \leq B_w$.
  \end{itemize}

  The remainder of the proof will establish parts (ii) and (iii)
  by organizing $\R^d$ into sets $(\cW_i)_{i\geq 1}$ with $\R^d = \cup_{i\geq 1} \cW_i$.
  For every integer $i\geq 1$ define
  \begin{align*}
    R_i &:= i + \|\bar w\|_1, %XXX I wanted i\|\bar w\|_1 because I arleady have lots of product terms, but consider \bar w = 0... argh
    \\
    \cW_i &:= \left\{w \in \R^d : \|w_\perp\|_1 \leq R_i \right\},
    \\
    \delta_i &:= \delta/(i+1)^2.
  \end{align*}
  By this choice, $\sum_{i\geq 1} \delta_i \leq \delta$,
  and thus proving both types of bound contributes to the final $2\delta$ in the full statement.
  Secondly, note how \Eq{sotired:2000} gives a way to use
  $\cR(w;\hmu_{|\Dcan})$ to choose $i$ with $w\in\cW_i$:
  the largest $i$ granting $w\in \cW_i$ satisfies
  $i \leq 1 - \|\bar w\|_1 + \|w_\perp\|_1
  \leq B_w - 1 - \|\bar w\|_1$.

  With this structure in place, parts (ii) and (iii) are established for each $i\in \Z_{++}$ separately as follows,
  and the general bounds follow by
  replacing the term $R_i$ via
  \[
    R_i \leq i + \|\barw\|_1 \leq \BigParens{B_w-1 - \|\bar w\|_1} + \|\barw\|_1 = B_w - 1.
  \]
  Note that, restricted to $\cW_i$, $\ell$ satisfies $\mu$-a.e.\@ boundededness in the
  sense that $\sup_{w\in\cW_i} \ell((Aw)(z)) \leq \ell(R_i)$ for $\mu$-a.e.\@ $z\in \Dcan$ and also $\hmu$-a.e.\@ $z\in\Dcan$
  (by part (i)),
  and $\ell$ is Lipschitz with constant $\ell'(R_i)$.

  \begin{enumerate}
    \item[(ii)]
      Using Rademacher complexity of Lipschitz functions (\Cref{fact:rad:dev,fact:rad:prop}),
      and \Eq{sotired:20000} to swap $w$ and $w_\perp$,
      and noting the general inequality $\sqrt{a} + \sqrt{b} \leq \sqrt{2a + 2b}$ for nonnegative reals,
      then for any fixed $i$, with probability at least $1-\delta_i$,
      every $w\in \cW_i$ satisfies
      \begin{align*}
        \cR(w;\mu_{|{\Dcan}})
        -
        \cR(w;\hmu_{|{\Dcan}})
        &=
        \cR(w_\perp;\mu_{|{\Dcan}})
        -
        \cR(w_\perp;\hmu_{|{\Dcan}})
        \\
        &\leq
        2\ell'(R_i)R_i\sqrt{2\ln(2d)/n} + 4\ell(R_i)\sqrt{2\ln(4/\delta_i)/n}
        \\
        &\leq
        8(\ell'(R_i)R_i + \ell(R_i))\sqrt{\frac {\ln(2d) + \ln(4/\delta_i)}{n}}
        \\
        &\leq
        8\ell(2R_i)\sqrt{\frac {\ln(2d) + \ln(4/\delta_i)}{n}},
      \end{align*}
      where the last simplification used $\ell(R_i) + \ell'(R_i)(2R_i - R_i) \leq \ell(2R_i)$ by convexity.
      To finish the proof, combining the above display with \Eq{sotired:23} gives
      \begin{align*}
        \cE(w;\mu_{|{\Dcan}})
        -
        \cE(w;\hmu_{|{\Dcan}})
        &=
        \cR(w;\mu_{|{\Dcan}})
        -
        \cR(w;\hmu_{|{\Dcan}})
        - \inf_{v\in \R^d}\cR(v;\mu_{|{\Dcan}})
        + \inf_{v\in \R^d}\cR(v;\hmu_{|{\Dcan}})
        \\
        &\leq
        \cR(w;\mu_{|{\Dcan}})
        -
        \cR(w;\hmu_{|{\Dcan}})
        + \ell(\|\bar w\|_1)\sqrt{\frac {2\ln(1/\delta)}{n}}.
        \\
        &\leq
        10\ell(2R_i)\sqrt{\frac {\ln(2d) + \ln(4/\delta_i)}{n}}.
      \end{align*}
      %I checked it over again and it seems legit.  Basically, the strong convexity bound is just more efficient...
      %\MJT{seems odd that lipschitz is $\ell'(R_i)R_i$ for one and $\ell'(R_i)$ for the other... double check.}

    \item[(iii)]
      Similarly to the purely Lipschitz case above, but now using \Cref{fact:findim:sc} to control deviations,
      with probability at least $1-\delta_i$, each $w\in \cW_i$ satisfies
      \begin{align*}
        \cE(w; \mu_{|{\Dcan}})
        &=
        \cE(w_\perp,\mu_{|{\Dcan}})
        \\
        &\leq
        2\cE(w_\perp; \hmu_{|{\Dcan}})
        +
        \frac {1024 \ell'(2R_i)^2 (\ln(2d) + \ln(4/\delta_i))}{n\tau(R_i) \Bal(\mu_{|{\Dcan}})^2}
        \\
        &=
        2\cE(w;\hmu_{|{\Dcan}})
        +
        \frac {1024 \ell'(2R_i)^2 (\ln(2d) + \ln(4/\delta_i))}{n\tau(R_i) \Bal(\mu_{|{\Dcan}})^2}.
      \end{align*}
  \end{enumerate}
\end{proofof}
%\end{proof}

\subsection{Proof of \Cref{fact:findim:gen:simplified}}

Before proving \Cref{fact:findim:gen:simplified},
note briefly how samples drawn from $\mu$
can be treated as a draw from $\mu_{|\Dcan}$ and $\mu_{|\Dcan^c}$.

\begin{lemma}
  \label[lemma]{fact:findim:sample_split}
  Let probability measure $\mu$
  and a canonical difficult set $\Dcan$ be given.
  Let $\cS$ denote a draw from $\mu$ of size $n \geq 8\ln(1/\delta)$,
  and define $\cS_\cD := \cS \cap \Dcan$ and $\cS_c := \cS \cap \Dcan^c$
  with sizes $n_\cD := |\cS_\cD|$ and $n_c := |\cS_c|$.
  Then with probability at least $1-\delta$ over the draw of $\cS$,
  \[
    n_\cD \geq n\mu(\Dcan)/2,
    \qquad\qquad
    n_c \geq n\mu(\Dcan^c)/2,
  \]
  and $\cS_{\cD}$ and $\cS_c$ can be treated as draws of size $n_{\Dcan}$ and $n_c$ from
  $\mu_{|\Dcan}$ and $\mu_{|\Dcan^c}$, respectively.
\end{lemma}
\begin{proof}
  Treating the partitioned sample as two independent draws is the usual rejection sampling.
  Moreover,
  by multiplicative Chernoff bounds \citep[Theorem 9.2]{kearns_vazirani} and the lower bound on $n$,
  \begin{align*}
    n_{\Dcan}/n = \hmu(\Dcan) &\geq \mu(\Dcan)\left(1 - \sqrt{2\ln(1/\delta)/n}\right) \geq \mu(\Dcan) / 2,
    \\
    n_c / n = \hmu(\Dcan^c) &\geq \mu(\Dcan^c)\left(1 - \sqrt{2\ln(1/\delta)/n}\right) \geq \mu(\Dcan^c) / 2.
  \end{align*}
\end{proof}

All the pieces are in place to prove \Cref{fact:findim:gen:simplified}.

\begin{proofof}{\Cref{fact:findim:gen:simplified}}
  To prove the bound,
  set $\delta' := \delta/7$,
  and let a sample $\cS$ be given
  with size
  \[
    n \geq 8\ln(1/\delta')
    + \one[\mu(\Dcan) > 0] \left(\frac {512(\ln(2d) + \ln(4/\delta'))}{\mu(\Dcan)\Bal(\mu_{|\Dcan})^2}\right)
    = \Omega(\ln(1/\delta')).
  \]
  By \Cref{fact:findim:sample_split}, conditioning away a first failure probability of $\delta'$,
  $\mu(\Dcan) > 0$ implies the set $\cS_\cD = \cS\cap \Dcan$ is an i.i.d. draw from $\mu_{|\Dcan}$ of size $n_\cD = \Omega(n)$
  satisfying moreover
  \begin{equation}
    n_\cD \geq \frac {256(\ln(2d) + \ln(4/\delta'))}{\Bal(\mu_{|\Dcan})^2},
    \label{eq:sotired:27}
  \end{equation}
  whereas
  $\mu(\Dcan^c) > 0$ implies the set $\cS_c = \cS\cap \Dcan^c$ is an i.i.d. draw from $\mu_{|\Dcan^c}$ of size $n_c = \Omega(n)$.

  Let $w\in \R^d$ be arbitrary,
  and note
  \begin{equation}
    \int |\eta_w - \bar\eta|d\mu
    =
    \underbrace{
      \int_{\Dcan} |\eta_w - \bar\eta|d\mu
    }_{\heartsuit}
    +
    \underbrace{
      \int_{\Dcan^c} |\eta_w - \bar\eta|d\mu
    }_\triangle;
    \label{eq:sotired:26}
  \end{equation}
  the proof will proceed by controlling $\heartsuit$ and $\triangle$ separately,
  where either term is 0 automatically if either $\mu(\Dcan) = 0$ or $\mu(\Dcan^c) = 0$, respectively.
  Note, throughout this proof, that $\cD = \Dcan$ $\mu$-a.e.\@ since $\ell\in\Lbnd$ thanks to \Cref{fact:findim:Dcan}.

  First consider the term $\heartsuit$ (when $\mu(\Dcan) > 0$);
  the goal will be to invoke \Cref{fact:D_controls},
  however many of the messy terms therein will be controlled via \Cref{fact:findim:gen:helper:D}.
  In particular, assume the various parts of \Cref{fact:findim:gen:helper:D},
  and condition away an additional $4\delta'$ failure probability,
  noting that $n_\cD$ is sufficiently large by \Eq{sotired:27}.
  Let $B_w$ be as in the statement of \Cref{fact:findim:gen:helper:D},
  which crucially only depends on $w$ only through $\hcE_n(w)$,
  and satisfies $|(Aw)(z)| \leq B_w$ for $\mu$-a.e.\@ $z\in \Dcan$.
  Furthermore, since $\Bal(\mu_{|\Dcan}) > 0$ by $\mu(\Dcan)>0$ and \Cref{fact:findim:bal_positive},
  by \Cref{fact:findim:primal_optimum}
  there exists a primal optimum $\bar w$ with $\barq_{\Dcan} \in\partial \ell(A\bar w)$ $\mu$-a.e.\@ over $\Dcan$,
  and $\mu$-a.e.\@ $z\in\Dcan$ satisfies
  \[
    |(A\bar w)(z)| \leq \|\barw\|_1 \leq \frac {\ell(0)\mu_{|\Dcan}(\cZ)}{\ell'(0)\Bal(\mu_{|\Dcan})} \leq B_w.
  \]
  By \Cref{lemma:dual:D}, $\tilde\q(z) := \barq_{\Dcan}(z)\one[z\in\Dcan]$ is also optimal for the full problem
  over $\mu$; but \Cref{fact:duality} provided that the full dual optimum is $\mu$-a.e.\@ unique,
  meaning the general $\barq$ and this specialized $\tilde\q$ agree $\mu$-a.e.\@ over $\Dcan$,
  and in particular $\mu$-a.e.\@ $z\in \Dcan$ satisfies
  \[
    \ell'(-B_w) \leq \ell'((A\barw)(z)) \quad = \quad \barq(z) \quad = \quad \ell'((A\barw)(z)) \leq \ell'(B_w).
  \]
  Consequently, applying \Cref{fact:D_controls}
  with constants $c_1 := B_w$, $c_2 := \inf_{|r| \leq B_w} \ell'(r) = \ell'(-B_w)$,
  and $c_3 := \sup_{|r| \leq B_w} \ell'(r) = \ell'(B_w)$,
  we have
  \[
    \mu(S_-)=\mu(S_+)=\mu(V)=0
  \enspace,
  \]
  so it suffices to include the term for $U$. Now using \Cref{fact:findim:gen:helper:D}.iii
  to relate $\cE(w;\mu_{|\Dcan})$ and $\cE_n(w;\hmu_{|\Dcan})$,
  additionally the general inequality $\sqrt{a +b} \leq \sqrt{a} + \sqrt{b}$ for nonnegative reals,
  and lastly recalling the notation $\tau(B_w) := \inf_{|q|\leq B_w} \ell''(q)$ from \Cref{fact:findim:gen:helper:D},
  \begin{align*}
    \heartsuit
    &\leq L_\phi \sqrt{\frac {2\cE(w;\mu_{\Dcan})}{\tau(B_w)}}
    \\
    &\leq L_\phi \sqrt{\frac{2\mu(\Dcan)}{\tau(B_w)}}\left(\sqrt{2\hcE_n(w;\mu_{|\Dcan})}
    + \sqrt{
        \frac {1024 \ell'(2B_w)^2 (\ln(2d) + \ln(4B_w^2/\delta'))}{n_\cD \Bal(\mu_{|\Dcan})^2 \tau(B_w)}
      }
    \right)
    \\
    &\leq
    \cO\left(
      f_2\left(\hcE_n(w)\right)
      \left(
        \sqrt{\hcE_n(w)}
        + \sqrt{\frac {\ln(1/\delta)}{n}}
      \right)
    \right),
  \end{align*}
  where the term $f_2(\cE_n(w))$ collects all terms depending on $B_w$, which itself depends on $w$ only through $\hcE_n(w)$
  as per \Cref{fact:findim:gen:helper:D}.

  Now consider the term $\triangle$ (when $\mu(\Dcan^c) > 0$);
  the goal will be to invoke \Cref{fact:Dc_controls},
  however once again some terms in the bound will be handled manually via \Cref{fact:findim:gen:helper:Dc}.
  Set $\eps := r:= \sqrt{\hcE_n(w)}$,
  and define
  $S_r := \{z\in\Dcan^c : \ell\bigParens{(Aw)(z)} \geq r\}$ exactly as in \Cref{fact:Dc_controls},
  and which also appears in \Cref{fact:findim:gen:helper:Dc};
  applying \Cref{fact:findim:gen:helper:Dc} with this $\eps$ to $w$ (where $\eps >0 $ since $\cE(w;d\hmu_{\Dcan}) > 0$
  by $\mu(\Dcan)>0$ and the assumed lower bound on $n_c$ and since $\ell\in\Lbnd$),
  and discarding an additional $2\delta'$ failure probability along the way,
  $\mu(S_r) \leq \cO\BigParens{\hcE_n(w)^{1/2} + \sqrt{(\ln(n_c) + \ln(1/\delta'))/n_c}}$.
  Combining this bound on $\mu(S_r)$
  with the bound on  $|\eta_w - \bar\eta|$ from \Cref{fact:Dc_controls}
  (which uses the fact that $\cD = \Dcan$ $\mu$-a.e.\@ since $\ell\in\Lbnd$)
  gives
  \begin{align*}
    \triangle
    &\leq \mu(S_r)
    + r \mu(\Dcan^c\setminus S_r)\max\left\{ \frac {1}{\ell(0)}, \frac {c_\ell}{\ell'(0)}\right\}
   %\\
   %&
    =
    \cO\left(
      \sqrt{\hcE_n(w)} + \sqrt{\frac{\ln(n) + \ln(1/\delta)}{n}}
    \right).
  \end{align*}
  Plugging these bounds on $\triangle$ and $\heartsuit$ back into \Eq{sotired:26} gives the desired inequality.

  Lastly, the convergence statement is, as usual, a consequence of the Borel-Cantelli lemma.
  In particular, let $\eps > 0$ be arbitrary, set $\delta_n := 1 / {n^2}$,
  and define the event
  \[
    E_{n,\eps} := \left[
      %WTF IS THIS
      %\int \ell(Aw_n) d\mu > \eps + \inf_{v\in \R^d} \int \ell(Av)d\mu
      \int |\eta_{w_n} - \bar\eta|d\mu > \eps
    \right].
  \]
  Applying the bound above for each $w_n$, as $n\to\infty$ and $\hcE_n(w_n) \to 0$,
  we obtain that there exists some $N$ so that every $n > N$ has $\Pr( E_{n,\eps} ) \leq \delta_n$.
  Consequently,
  \[
    \sum_{n\geq 1} \Pr( E_{n,\eps} )
    \leq \sum_{n=1}^N 1 + \sum_{n > N} \delta_n
    < \infty,
  \]
  and the result follows by applying the Borel-Cantelli lemma.
\end{proofof}

\end{document}